\documentclass[journal]{IEEEtran}

\ifCLASSINFOpdf
\else
\fi
\usepackage{amsmath}
\usepackage{enumitem}
\usepackage{amsthm}
\usepackage{amssymb}
\usepackage{graphicx}
\usepackage{latexsym}
\usepackage{subfigure}
\usepackage{color}
\usepackage{epsfig}
\usepackage{dsfont}
\usepackage{fancyhdr}
\usepackage{pst-all}
\usepackage[scanall]{psfrag}
\usepackage{multirow}
\usepackage{booktabs}
\usepackage{textcomp}
\usepackage{siunitx}
\usepackage{textcomp}
\usepackage{courier}

\usepackage{amsfonts,amssymb}
\usepackage[german,english]{babel}
\usepackage{url}
\usepackage[colorlinks=true,
            linkcolor=black,
            citecolor=black,
             urlcolor=blue,
            ]{hyperref}

\renewenvironment{proof}{\vspace{.1cm}\noindent{\sc Proof.}\hspace{0.10cm}\,\,}{ \hfill }
\newtheorem{theorem}            {Theorem}[section]

\newtheorem{proposition}		[theorem]{Proposition}

\newcommand{\dd}   {{\rm d}\hbox{\hskip 0.5pt}}

\newcommand{\rline}{{\mathbb R}}

\newcommand{\bbm}[1]{\left[\begin{matrix} #1 \end{matrix}\right]}
\newcommand{\sbm}[1]{\left[\begin{smallmatrix} #1
	\end{smallmatrix}\right]}

\newcommand{\rfb}[1]{\mbox{\rm
		(\ref{#1})}\ifx\undefined\stillediting\else:\fbox{$#1$}\fi}

\newcommand{\bluff}{{\hbox{\raise 15pt \hbox{\hskip 0.5pt}}}}
\newcommand{\etal}{\textit{et al }}
\newfont{\roma}{cmr10 scaled 1200}

\begin{document}
\title{Source Seeking Control of Unicycle Robots with 3D-printed Flexible Piezoresistive Sensors}

\author{Tinghua Li\textsuperscript{1} , Bayu Jayawardhana\textsuperscript{1} ,  Amar Kamat\textsuperscript{2} ,  and Ajay Giri Prakash Kottapalli\textsuperscript{2} 
    \thanks{This work was supported in part by China Scholarship Council, in part by the SNN programme on CoE Smart Sustainable Manufacturing, and in part was labelled by ITEA and funded by local authorities under the grant agreement ITEA-2018-17030-Daytime.}
	\thanks{\textsuperscript{1}Tinghua Li and Bayu Jayawardhana are with DTPA, ENTEG, Faculty  of  Science  and  Engineering,  University  of  Groningen, The Netherlands
		\{\tt\small t.li, b.jayawardhana\}@rug.nl.}%
	\thanks{\textsuperscript{2}Amar Kamat and Ajay Giri Prakash Kottapalli are with APE, ENTEG, Faculty  of  Science  and  Engineering,  University  of  Groningen, The Netherlands
		\{\tt\small a.m.kamat, a.g.p.kottapalli\}@rug.nl.}%
}

\maketitle

\begin{abstract}
We present the design and experimental validation of source seeking control algorithms for a unicycle mobile robot that is equipped with novel 3D-printed flexible graphene-based piezoresistive airflow sensors. Based solely on a local gradient measurement from the airflow sensors, we propose and analyze a projected gradient ascent algorithm to solve the source seeking problem. In the case of partial sensor failure, we propose a combination of Extremum-Seeking Control with our projected gradient ascent algorithm. For both control laws, we prove the asymptotic convergence of the robot to the source. Numerical simulations were performed to validate the algorithms and experimental validations are presented to demonstrate the efficacy of the proposed methods.

\end{abstract}


\IEEEpeerreviewmaketitle

\section{Introduction}

In many biological organisms, source seeking is an innate ability that is crucial for their survival. For instance, bacterial chemotaxis is used to reach the much-needed chemical nutrients where the gradient of chemical concentration directs the bacterial movement \cite{Wadhams}. Another example is the navigation of blind cave fish in dark underwater caves based on its ability to measure the local fluid flow sensitively \cite{Kottapalli}, \cite{Bora1}, \cite{Bora}. The last example is the ability of seals to locate their preys using sensitive flow measurement via their whiskers \cite{Dehnhardt}. Inspired by these examples from nature, such source seeking capability will enable autonomous robotic systems to navigate and localize sources that are crucial for completing some desired tasks. For example, search-and-rescue robotic missions in a hazardous environment may require the ability to locate heat sources, in order to search for escape exits in a collapsed site or subterranean area or to navigate through the air flow to find dangerous substances. In an environmental disaster recovery mission, the robots must be able to seek hazardous chemicals and to contain them as quickly as possible. In particular, as a carrier of chemical substances, airflow plays a key role in the advection and diffusion transportation mechanisms. Thus, estimating the local flow field using airflow sensors can yield useful information about the environment. The monitoring of volumetric airflow can effectively estimate the strength and help search for the source direction in a localization task. For instance, in extreme rescue situations, an array of sensitive elastic airflow sensors helps sense the ambient airflow in multiple directions. The airflow information can provide effective assistance for aerial robots in exploration trajectories \cite{Al-Sabban}, \cite{Neumann}, minimize the onboard energy usage and improve concentration leak localization \cite{Kowadlo} and mapping \cite{Reggente}. Further, by tracking the lowest airflow path, the aerodynamic drag and fuel consumption in a platoon of autonomous trucks can be reduced \cite{ Mahrle}.

Depending on the tasks and sources at hand, the autonomous systems can be equipped with a multitude of local and global sensor systems, based on which, the control systems steer them towards the right trajectory. The sources are typically assumed to be fields of physical variables (e.g., the field of electromagnetic, temperature, pressure or chemical concentration) or advection-induced flow (such as, fluid/air flows and heat transfer). The corresponding sensor systems can then provide the field or flow information by its magnitude or as a gradient vector.

Based on the available information, a source-seeking control algorithm, which is designed based on the specific systems' dynamics, computes the required control signal to steer the systems towards the sources. When the gradient information is lacking, bio-inspired source seeking control algorithms have been proposed in literature \cite{Nurzaman, Grasso} where the robots perform chemotaxis-like exploration to locate the sources. Another approach to tackle the lack of gradient information is to deploy Bayesian inference method as pursued in \cite{Porat} or to deploy a mobile sensor network \cite{Ogren, Pang}. The aforementioned approaches however rely on the availability of global position information to track the robot's location in real time. Yet this assumption does not typically hold true in difficult environments, such as, in the deep sea, under the ground, indoors and in fire grounds, where the robots must rely on local sensor measurements. In this paper, we focus on the deployment of local sensor systems and source seeking control for unicycle-like mobile robots that operate in the latter environment, e.g., without position measurement.

A large class of source seeking algorithms developed in the literature use gradient information. In this case, the gradient-descent or gradient-ascent algorithms have been developed and deployed both for single or a group of vehicles, see for examples \cite{Baronov2008, Burian1996, Madgwick, Huynh, Soltero, Cortes}. The Artificial Potential Field (APF) method was proposed by Khatib for  robot motion planning \cite{Khatib}.
In general, the APF-based control approach relies on the use of an artificial potential, which is known apriori, in order to design an admissible trajectory of the robot for reaching the target while avoiding obstacles. Such admissible trajectory will then be used in the low-level control as a reference trajectory that also takes into account the kinematics of the underlying robot \cite{Rostami,Vadakkepat}. When the gradient information is not available, an approximation to the gradient using multiple sensors in a unicycle agent has been presented \cite{Fabbiano}.  Recently, Bachmayer and Leonard \cite{Bachmayer} proposed the use of a coordinated control strategy for a group of autonomous vehicles to descend or climb an environmental gradient, depending on the measurements of the environment together with relative position measurements of nearest neighbors. The approximation of the local gradient in \cite{Bachmayer} is based on the use of a single sensor per vehicle, where each vehicle is assumed to be able to measure the gradient only in the direction of motion. Subsequently the authors presented a distributed controller with inter-vehicle communication in order to steer the group to the global minimum (or maximum) of the sampled environmental gradient field. Another related work \cite{Moore} presents a distributed control method for moving agents in a specific shape formation that solves the source seeking problem collaboratively. Without the gradient estimation, Matveev \etal \cite{Matveev} proposed a sliding mode navigation strategy to control the agent by a limited time-varying angular velocity control.

Another family of popular source seeking methods are the Extremum Seeking Control (ESC) based algorithms. Roughly speaking, it is based on the use of averaging technique via dither signals in order to extract the gradient information \cite{Ghadiri}. The dither signals are typically combined with the sensor signals and the control signals, and they can also be regarded as singular perturbation as studied in \cite{Durr,DurrHB}. Some recent works on ESC-based source seeking for unicycle-like mobile robots are presented in \cite{Zhang, Ghods, Lin, Liu, LinJ, Azuma, Fu}. By setting the angular velocity constant, Zhang \etal in \cite{Zhang} presented a control law for the forward velocity to seek the source. The use of ESC for the control of angular velocity component has also been explored in  \cite{Cochran}. In \cite{Ghods, Lin}, the authors presented the ESC method for both the angular and forward velocity of unicycle-like vehicles. Fu and Ozguner \cite{Fu} studied ESC method for unicycle-like agents when there are constraints on the accessible area of the agent.
While sinusoidal dither signals are used in the aforementioned works, filtered white noise dither signals have been used in \cite{Liu,LinJ,Azuma} that are commonly known as the stochastic source seeking methods. In \cite{Raisch2016}, Raisch and Krstic combined two kinds of perturbation signals (periodic-based and constant-based) for an efficient convergence toward the source and maintaining tight hovering near the source.  By bounding the update rate in the optimum seeking and stabilization control law, Scheinker and Krstic proposed a new constrained ESC scheme \cite{Scheinker2014a}, which is advantageous to the hardware control implementation \cite{Scheinker2014b} and is generalized in \cite{Scheinker2014c}.

In our first main contributions, we present the design of two source seeking control methods for a unicycle agent based solely on the use of local field gradient or flow information and local coordinate frame:
\begin{enumerate}
\item In our first control method, we propose a projected gradient-ascent control law where we control both the longitudinal and angular velocities in order to reach the position of a local maxima using instantaneous gradient or flow information provided by an on-board sensor system. In most of the gradient-based source seeking algorithm for unicycle agent, it is common to consider the angular velocity as the input variable while keeping longitudinal velocity constant, see for instance \cite{Fabbiano, Cochran}. In \cite{Baronov2008}, the authors present a gradient-based control law for both the longitudinal and angular velocities where the longitudinal velocity is proportional to the magnitude of the gradient. Due to the particular structure of the control law, the controller in \cite{Baronov2008} is restricted by the bounds on the gradient's curvature and it results in a semi-global asymptotic stability of the source location. We relax this limitation by introducing a projected gradient-ascent control law to both velocities and we prove the asymptotic stability of the source (which is a global one when there is only one extremum).
\item In our second control approach, we propose a combination of extremum seeking control and our projected gradient-ascent control law to solve the source seeking problem when only the magnitude of the gradient or flow is available due to limitation introduced to the sensor systems, e.g., due to a fault in the on-board sensor systems. For both approaches, we prove the asymptotic convergence of the unicycle agent to the source and show the efficacy of the controllers numerically and experimentally via hardware-in-the-loop and lab experiments.
\end{enumerate}

The second main contribution of this paper pertains to the deployment of novel 3D-printed flexible graphene-based piezoresistive flow sensors as a proxy to the gradient information. The sensor design is based on our previous work in \cite{Kamat2020A,Kamat2020B}. In particular, we fabricate and characterize four flexible all-polymer flow sensors which are mounted on our mobile robot platform wherein each sensor provides information of bidirectional flows. Subsequently, we incorporate the sensor readout as a proxy of the gradient and experimentally demonstrate the performance of both proposed controllers in lab experiments.

The rest of the paper is organized as follows. In Section II, we discuss the source seeking control problem formulation. In Section III, we propose two control laws and we subsequently present asymptotic stability analysis of the closed-loop systems. The design and characterization of a 3D-printed flexible piezoresistive flow sensors are presented in Section IV. In Section V, we present numerical simulation results and in Section VI, we present experimental setup and results. Finally, the conclusions and future work are discussed in Section VII.

\section{Problem formulation }
\begin{figure}[htbp]
	\centering
	\includegraphics[width=8cm]{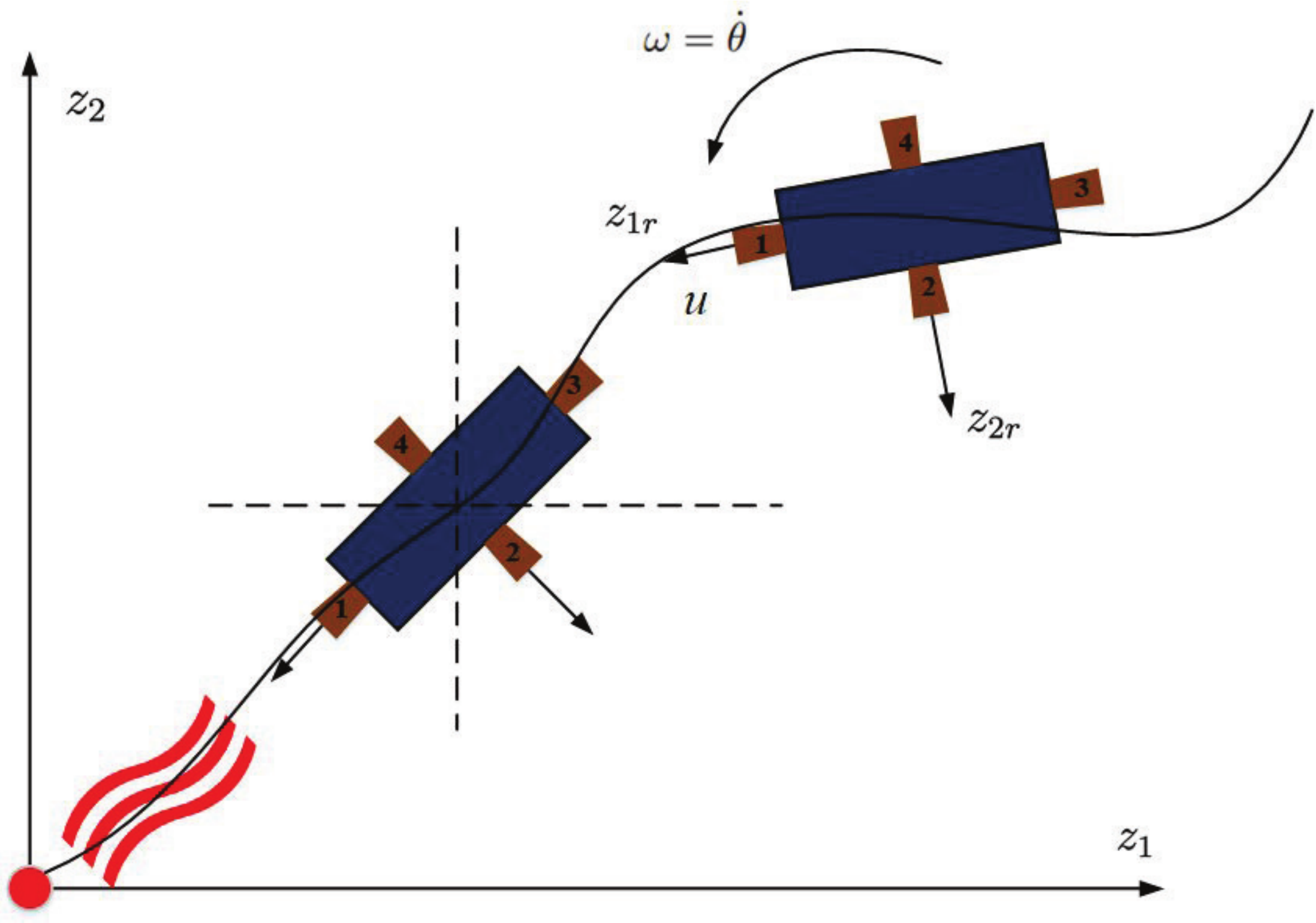}
	\caption{An illustration of a unicycle mobile robot navigating in a 2D plane. The red wavy lines illustrates the wind flow that emanates from a source (red dot). At any given position $(z_1,z_2)$ and heading angle $\theta$, the robot controls the longitudinal $u$ and angular $\omega$ velocity in order to seek the source based on the use of on-board flow sensors (indicated by brown squares on its side) and without global positioning systems.}
	\label{Robot Model with four airflow sensors}
\end{figure}

As described in the Introduction, we consider a unicycle robot which is equipped with four air flow sensors mounted at its front, end, left and right side. Figure \ref{Robot Model with four airflow sensors} shows an illustration of the unicycle robot which traverses the 2D plane. For a given 2D position $(z_1,z_2)$ and heading angle $\theta$, the robot can drive forward with a longitudinal velocity $u$ and rotate with an angular velocity $\omega$. 

Correspondingly, we consider the following dynamic model of unicycle robot
\begin{equation}\label{eq:unicycle_model}
\begin{bmatrix}
\dot{z_1} \\
\dot{z_2}  \\
\dot{\theta}
\end{bmatrix}=\begin{bmatrix}
u\cos(\theta) \\
u\sin(\theta)  \\
\omega
\end{bmatrix}
\end{equation}
where $\sbm{z_1(t) \\ z_2(t)}$ is the 2D planar robot's position with respect to a global frame of reference, $\theta(t)$ is the heading angle, $u(t)$ is the longitudinal velocity input variable and $\omega(t)$ is the angular velocity input variable.

For our source-seeking control problem, we consider a source that emits (laminar) air flow on the 2D plane whose strength decays with distance to the source. Let the potential function $J(z_1,z_2)$ define the magnitude of the air flow which has a global maximum at the source location $(z_1^*,z_2^*)$. Furthermore, we consider the robot setup where it is able to measure the local gradient of $J$, denoted by $\nabla J(z_1,z_2)$, using our 3D-printed graphene-based piezoresistive sensors \cite{Kamat2020A,Kamat2020B} which will be described further in Section IV. Based on this setup, we can define our control design problem as follows.

{\bf Gradient-based source-seeking control problem:} For the unicycle systems as in \eqref{eq:unicycle_model} and the available measurement of gradient $\nabla J(z_1,z_2)$ and orientation $\theta$, design feedback control laws $u=F(\nabla J(z_1,z_2),\theta)$ and $\omega=G(\nabla J(z_1,z_2),\theta)$ such that
\begin{equation}\label{eq:convergence}
\lim_{t\to\infty} \left\|\bbm{z_1(t)-z_1^*\\z_2(t)-z_2^*}\right\| = 0
\end{equation}
holds for all initial conditions $(z_1(0), z_2(0)) \in \mathcal Z\subset \rline^2$ (globally when $\mathcal Z=\rline^2$).

We note that in the above control problem formulation, the control laws do not depend on the availability of global position $\sbm{z_1(t)\\z_2(t)}$. It relies only on the local gradient measurement as well as its local orientation. As discussed in the Introduction, many results in the literature assume that $F$ is a constant \cite{Fabbiano, Cochran} or $F(a,b)=|a|$ \cite{Baronov2008}.

\section{Control design and analysis}
\subsection{Projected gradient-ascent control law}\label{sec:gradient_ascent}

Let us now present a projected gradient-ascent control law that is based on the measured
local gradient information from airflow sensors in an unknown nonlinear map $J(z_1,z_2)$. Firstly, let us denote the unit vector orientation of the mobile robot by
\begin{equation}\label{eq:v_theta}
\vec{v}(\theta) =\begin{bmatrix}
\cos(\theta) \\
\sin(\theta)
\end{bmatrix}.
\end{equation}
 As usual, for a given scalar function $J(z_1,z_2)$, we define the local gradient $\nabla J(z_1,z_2)$ by 
\begin{equation}
\nabla J(z_1,z_2)=\begin{bmatrix}
\frac{\partial J}{\partial x}(z_1,z_2) & \frac{\partial J}{\partial y}(z_1,z_2)
\end{bmatrix}
\end{equation}
and correspondingly, we define
$\nabla J^{\perp }(z_1,z_2)$ as an orthogonal vector to $\nabla J$ satisfying 
\begin{equation*}
\left\langle\nabla J^{\perp }(z_1,z_2),\nabla J(z_1,z_2)\right\rangle=0
\end{equation*}
\begin{equation}\label{eq:crossproduct}
\nabla J^{\perp }(z_1,z_2)\times \nabla J(z_1,z_2)> 0
\end{equation}
\begin{equation}\label{eq:norm_gradient}
\left \| \nabla J^{\perp }(z_1,z_2) \right \|=\left \| \nabla J(z_1,z_2) \right \|
\end{equation}
where $\langle \cdot,\cdot \rangle$ denotes inner product. We note that the cross product in \eqref{eq:crossproduct} means that  $\nabla J^{\perp }(z_1,z_2)$ is the direction of  $\nabla^\perp J(z_1,z_2)$ is $-90^0$ of that of $\nabla J(z_1,z_2)$. The last equality \eqref{eq:norm_gradient} implies also that $\nabla J(z_1,z_2) = 0  \Leftrightarrow \nabla^\perp J(z_1,z_2) = 0$.

Using $\vec{v}(\theta)$, $\nabla J$ and $\nabla^\perp J$, the proposed projected gradient-ascent control laws for $u$ and $\omega$ are given by
\begin{equation}\label{eq:control_laws}
 \begin{array}{rl}
u & = k_1 \left\langle \vec{v}(\theta),\nabla J(z_1,z_2) \right\rangle \\
\omega & = -k_2 \left\langle \vec{v}(\theta),\nabla^\perp J(z_1,z_2) \right\rangle,
\end{array}
\end{equation}
where $k_1>0$ and $k_2>0$ are the longitudinal velocity gain and angular velocity gain, respectively.  We remark that the control law as given in \eqref{eq:control_laws} requires only local measurement of the gradient $\nabla J$ as well as local measurement of the orientation $\theta$ (such as, wheels' encoder) based on its local coordinate frame. In this regard, the framework is suitable for deployment in a GPS-free environment.
The closed-loop systems dynamics is given by

\begin{equation}\label{eq:closed_loop}
    \bbm{\dot z_1 \\ \dot z_2 \\ \dot \theta} = \bbm{k_1 \left\langle \vec{v}(\theta),\nabla J(z_1,z_2) \right\rangle \vec{v}(\theta) \\ -k_2 \left\langle \vec{v}(\theta),\nabla^\perp J(z_1,z_2) \right\rangle}.
\end{equation}

\begin{proposition}\label{prop_1}
	Consider the unicycle system in \eqref{eq:unicycle_model}. Assume that the potential function $J$ is twice-differentiable, radially unbounded and is a strictly concave function with a maximum at $\sbm{z_1^*\\z_2^*}$. Then for any positive gains $k_1, k_2>0$, the control law \eqref{eq:control_laws} solves the gradient-based source-seeking control problem globally.
\end{proposition}

\begin{proof}
Firstly, we introduce  auxiliary state variables
	\begin{equation}	
	\bbm{z_3 \\ z_4} = \bbm{\cos(\theta) \\ \sin(\theta)}
	\end{equation}
	so that we can rewrite the closed-loop systems equation \eqref{eq:closed_loop} into the following autonomous system
	\begin{align}\label{system_eq_new}
	\dot z = \bbm{\dot z_1\\ \dot z_2 \\ \dot z_3 \\ \dot z_4} & = \bbm{k_1\langle\nabla J(z_1,z_2),\sbm{z_3\\z_4}\rangle z_3 \\ k_1\langle \nabla J(z_1,z_2),\sbm{z_3\\z_4}\rangle z_4 \\ z_4 k_2\langle\nabla^{\perp} J(z_1,z_2),\sbm{z_3\\z_4}\rangle \\ -z_3 k_2\langle\nabla^{\perp} J(z_1,z_2),\sbm{z_3\\z_4}\rangle},
	\end{align}
	where $z = \bbm{z_1 & z_2 & z_3 & z_4}^T$ are the new state variables. We consider now the following twice-differentiable function
    \begin{equation}
	V(z_1,z_2,z_3,z_4) = -J(z_1,z_2) + \frac{1}{2}z_3^2 + \frac{1}{2}z_4^2+J(z_1^*,z_2^*),
    \end{equation}
	which is positive definite and radially unbounded by the strictly concave and radial unboundedness property of $J$. Its time derivative is given by
	\begin{equation}\label{eq:V_dot}
	\begin{aligned}
    \dot V(z_1,z_2,z_3,z_4) & = -\nabla J(z_1,z_2)\bbm{k_1\langle\nabla J(z_1,z_2),\sbm{z_3\\z_4}\rangle z_3 \\ k_1\langle\nabla J(z_1,z_2),\sbm{z_3\\z_4}\rangle z_4} \\
	& \qquad + z_3\left(z_4 k_2\langle\nabla^{\perp} J(z_1,z_2),\sbm{z_3\\z_4}\rangle\right) \\& \qquad+z_4\left(-z_3 k_2\langle\nabla^{\perp} J(z_1,z_2),\sbm{z_3\\z_4}\rangle\right)\\
	& = -\nabla J(z_1,z_2) \underbrace{k_1\langle\nabla J(z_1,z_2),\sbm{z_3\\z_4}\rangle}_{u} \sbm{z_3\\z_4} \\
	& = -k_1 \left\langle\nabla J(z_1,z_2),\bbm{z_3\\z_4}\right\rangle^2 \leq 0.
	\end{aligned}
	\end{equation}

It follows from \eqref{eq:V_dot} that for all initial conditions $z_1(0), z_2(0), z_3(0), z_4(0)$, we have that
\begin{align*}
V(z_1(t),z_2(t),z_3(t),z_4(t)) & \leq V(z_1(0),z_2(0),z_3(0),z_4(0))
\end{align*}
holds for all $t\geq 0$.
By the radial unboundedness of $V$, the sub-level set of $V(z(t))$
is compact and thus
$\sbm{z_1(t) & z_2(t) & z_3(t) & z_4(t)}^T$ is bounded for all time $t\geq 0$.
	Accordingly, by using the standard La-Salle invariance arguments (see, for instance, \cite[Theorem 3.4]{Khalil}), $z$ converges to the largest invariant set of the $\Omega$-limit set, in which $\langle \nabla J(z_1,z_2),\sbm{z_3\\z_4}\rangle = 0$ holds for all time.

	We will now characterize the invariant set $\Omega$.
     In this invariant set $\Omega$, any trajectory $z(t)\in \Omega$ must satisfy that $\nabla J(z_1(t),z_2(t))\sbm{z_3(t)\\z_4(t)}=0$ for all $t\geq 0$. Note that $\nabla J(z_1,z_2)\sbm{z_3\\z_4}=0$ implies that the vectors $\nabla J(z_1,z_2)^T$ and $\sbm{z_3\\z_4}$ are perpendicular with each other.

     Let us prove by contradiction that in the invariant set $\Omega$, $\sbm{z_1\\z_2}=\sbm{z_1^*\\z_2^*}$. Suppose that in $\Omega$, $\sbm{z_1 \\ z_2} \neq \sbm{z_1^* \\ z_2^*}$, in which case, according to \eqref{system_eq_new}, we have that
	\begin{equation}
	\bbm{\dot z_3 \\ \dot z_4} = \bbm{z_4 k_2\nabla^{\perp} J(z_1,z_2)\sbm{z_3\\z_4} \\ -z_3 k_2\nabla^{\perp} J(z_1,z_2)\sbm{z_3\\z_4}} \neq 0.
	\end{equation}
	The above relation follows from the fact that $\nabla J(z_1,z_2) \sbm{z_3\\z_4}=0$ with $\nabla J(z_1,z_2)\neq 0 \Rightarrow \nabla^{\perp} J(z_1,z_2)\neq 0$ and $\sbm{z_3\\z_4}\neq 0$ (by the definition of $\sbm{z_3\\z_4}$ that lives in a circle) so that the vectors $\nabla^{\perp} J(z_1,z_2)$ and $\sbm{z_3\\z_4}$ are co-linear. This implies that when $\sbm{z_1(t) \\ z_2(t)} \neq \sbm{z_1^* \\ z_2^*}$
	\begin{equation}
	\begin{aligned}
	& \frac{\dd}{\dd t}\langle \nabla J(z_1,z_2),\sbm{z_3\\z_4}\rangle
\\ & = \bbm{z_3&z_4}k_1\underbrace{\langle \nabla J(z_1,z_2),\sbm{z_3\\z_4} \rangle}_{=0} \nabla^2 J(z_1,z_2) \sbm{z_3\\z_4}
	\\ & \qquad + \nabla J(z_1,z_2)\bbm{z_4 \\ -z_3} k_2\left\langle\nabla^{\perp} J(z_1,z_2),\bbm{z_3\\z_4} \right\rangle \\
	& = \nabla J(z_1,z_2)\bbm{z_4\\-z_3} k_2 \underbrace{\left\langle\nabla^{\perp} J(z_1(t),z_2(t)),\bbm{z_3(t)\\z_4(t)}\right\rangle}_{\neq 0} \\
	& \neq 0,
	\end{aligned}
	\end{equation}
	where the last relation is also due to the fact that $\nabla J(z_1,z_2) \sbm{z_4\\-z_3}\neq 0$ since $\sbm{z_4\\-z_3}$ is perpendicular to $\sbm{z_3\\z_4}$ so that the vectors $\nabla J(z_1,z_2)$ and $\sbm{z_4\\-z_3}$ are co-linear. This is a contradiction, as $\frac{\dd}{\dd t}\langle \nabla J(z_1,z_2),\sbm{z_3\\z_4}\rangle = 0$ for all $z\in\Omega$. Indeed, when $\sbm{z_1(t) \\ z_2(t)} = \sbm{z_1^* \\ z_2^*}$, it follows that
 	\begin{equation*}
 	\frac{\dd}{\dd t}\langle \nabla J(z_1,z_2),\sbm{z_3\\z_4}\rangle = 0
	\end{equation*}
	i.e., $z(t)$ remains always in the invariant set $\Omega$.

	Thus the invariant set $\Omega$, where $\nabla J(z_1(t),z_2(t))\sbm{z_3(t)\\z_4(t)}=0$ holds for all $t\geq 0$, satisfies
	\[
	\Omega \subset \{z\in \rline^2\times S^1 \ | \ z = \bbm{z_1^* & z_2^* & \cos(\theta) & \sin(\theta)}^T, \theta\in \rline \},
	\]
	where $S^1$ defines the unit circle. Therefore, by the La-Salle invariance principle, we have that all bounded solutions $z(t)\to \Omega$ as $t\to \infty$, in particular, $\sbm{z_1(t)\\z_2(t)}$ converges to $\sbm{z_1^*\\z_2^*}$ as claimed.
	
	The global attractivity of $\sbm{z_1^*\\z_2^*}$ follows directly from the radial unboundedness of $V$ and the previous analysis applies vis-\`a-vis.
\end{proof}\vspace{0.2cm}

As can be seen in the proof of Proposition \ref{prop_1}, the radial unboundedness of $J$ is required to guarantee the forward completeness and boundedness of the closed-loop systems trajectories for any initial conditions. The assumption of radial unboundedness of $J$ in Proposition \ref{prop_1} can be relaxed to a locally strictly concave function provided that we can guarantee these boundedness properties of the trajectories for some initial conditions in the neighborhood of the maxima. In this case, the asymptotic convergence follows the same arguments of La-Salle invariance principle.

\subsection{Extremum Seeking Control based approach}\label{sec:extremum_seeking}

The source-seeking controller that we designed in Subsection \ref{sec:gradient_ascent} relies upon the availability of the airflow vector measurement, in which the wind direction is in line with the gradient of airflow strength $\nabla J$. In the case of a partial sensor failure, the real-time measurement of the airflow vector may no longer be available. In this situation, we need a fault tolerant mechanism for the source-seeking controller based only on the remaining working sensor that can still provide the information on the potential function $J$, instead of the gradient $\nabla J$. In this sub-section, we design a complementary controller to the one developed in Subsection \ref{sec:gradient_ascent} where we combine the projected gradient-ascent control law with the Extremum Seeking Control (ESC) approach.

\begin{figure}[htbp]
	\centering
	\includegraphics[width=8cm]{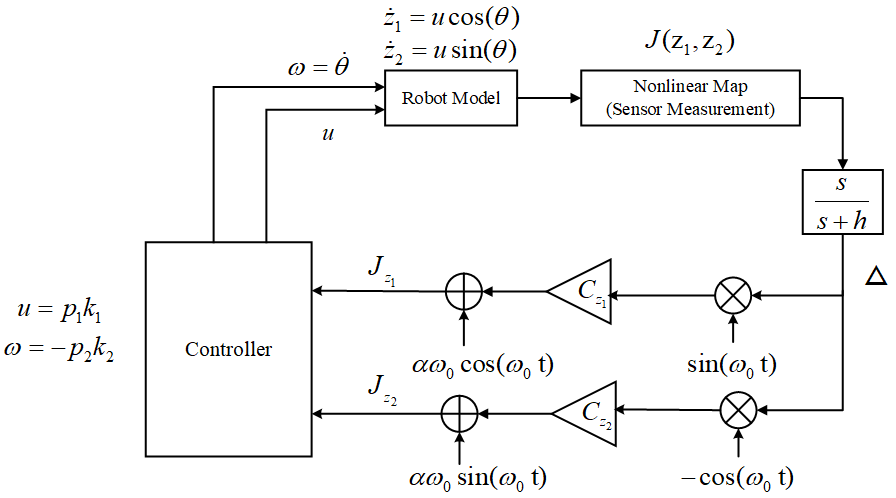}
	\caption{Block diagram of the proposed Extremum Seeking Control-based projected gradient-ascent control law for solving the source-seeking problem based only on the measurement of the potential function $J$. In this diagram, $p_1=  \left\langle \protect\sbm{J_{z_1} \protect\\ J_{z_2}},  \protect\sbm{\cos(\theta)\protect\\ \sin(\theta)}\right\rangle$ and $p_2=  \left\langle  \protect\sbm{-J_{z_2}\protect \\ J_{z_1}},  \protect\sbm{\cos(\theta)\protect \\ \sin(\theta)} \right\rangle$ where $ \protect\sbm{J_{z_1}\protect\\J_{z_2}}$ is the approximation of $\nabla J$ and $ \protect\sbm{-J_{z_2}\protect \\ J_{z_1}}$ is the approximation of $\nabla J^\perp$.}
	\label{Model of Extremum Seeking Control}
\end{figure}

ESC is an averaging control method for steering the systems towards the extremum point of a scalar function that represents a cost or potential function and is available through a sensor system. Roughly speaking, the ESC is done by introducing dither signals (typically, sinusoidal signals) into the measurement signal so that it is able to evaluate the potential function in the neighborhood region and to subsequently provide an approximation of the potential gradient. The approximate gradient information can then be used to steer the systems towards the extremum of the potential function.

Figure \ref{Model of Extremum Seeking Control} shows a block diagram of our proposed ESC-based projected gradient-ascent control law. As shown in this figure, perturbation signals $a\omega_0 \cos(\omega_0 t)$ and $a\omega_0 \sin(\omega_0 t)$ are firstly introduced to the gradient-ascent control law to excite the system. Subsequently, the measurement signals from sensor systems are filtered through a first-order high-pass filter with a cut-off frequency $h$ which filters out the DC-component.
The filtered signal $\Delta(t)$ is then modulated by a dither signal $\sbm{C_{z_1}\sin(\omega_0t)\\ C_{z_2}\cos(\omega_0t)}$,
where the constants $C_{z_1}$ and $C_{z_2}$ are design parameters. The approximated gradient of the potential function $J$ is then given by
\begin{align}
\label{eq:hat_J_z1}\frac{\partial \hat J}{\partial z_1} & := J_{z1}=C_{z_1}\Delta\sin(\omega_0 t)+a\omega_0 \cos(\omega_0  t),      \\ \label{eq:hat_J_z2} \frac{\partial \hat J}{\partial z_2} & :=J_{z2}=-C_{z_2}\Delta \cos(\omega_0 t)+a\omega_0 \sin(\omega_0  t).
\end{align}
Using $J_{z1}$ and $J_{z2}$ in \eqref{eq:hat_J_z1} and \eqref{eq:hat_J_z2}, respectively, the approximated gradient vector $\widehat{\nabla J}$ and its orthogonal vector $\widehat{\nabla J}^\perp$ is given by
\begin{equation}\label{eq:widehat_J}
\left.\begin{matrix}
\widehat{\nabla J}=\begin{bmatrix}
J_{z_1} & J_{z_2}
\end{bmatrix}^T \\
\widehat{\nabla J}^{\perp }=\begin{bmatrix}
-J_{z_2} & J_{z_1}
\end{bmatrix}^T
\end{matrix}\right.
\end{equation}

Using the approximated gradient  above and the projected gradient-ascent control law in \eqref{eq:control_laws}, the ESC-based projected gradient-ascent control law can be given as follows
\begin{equation}\label{eq:ESC_control_laws}
\left. \begin{array}{rl}
u & = k_1 \left\langle \vec{v}(\theta),\widehat{\nabla J}(z_1,z_2) \right\rangle \\
\omega & = -k_2 \left\langle \vec{v}(\theta),\widehat{\nabla J}(z_1,z_2)^\perp \right\rangle,
\end{array} \right.
\end{equation}
where $\vec{v}(\theta)$ is as in \eqref{eq:v_theta}.

In the following proposition, we show that for a quadratic local potential function, the average trajectory of $z_1$ and $z_2$ will converge to the local extremum point. For convenience, we define the average trajectories of $z_i$, $i=1,2,3,4$, by
\begin{align*}
z_{i,\text{avg}}(t) & = \frac{\omega_0}{2k\pi}\int_t^{t+\frac{2k\pi}{\omega_0}}z_i(\sigma) \dd\sigma,
\end{align*}
where $\omega_0$ is the frequency of the dither signal that is typically a high-frequency that allows for a time-scale separation with $k>0$ be a number of periods of dither signals that can be taken into account without affecting the slow time-scale dynamics.

\begin{proposition}\label{prop_2}
Consider the unicycle system in \eqref{eq:unicycle_model} and assume that the potential function $J$ is quadratic function given by
\begin{equation}\label{eq:quadratic_J}
J(z_1,z_2)=J^{*}-c_{1}(z_1-z_1^*)^2-c_{2}(z_2-z_2^*)^2,
\end{equation}
where $J^*$ is the local maximum, $c_1, c_2$ are unknown positive constants and $\sbm{z_1^*\\z_2^*}$ is the global maximizer.
Then, for any positive gains $k_1, k_2, C_{z1}, C_{z2}, a>0$ and sufficiently large $\omega_0$ the ESC-based projected gradient-ascent control law in  \eqref{eq:ESC_control_laws} with $\widehat{\nabla J}$ and $\widehat{\nabla J}^\perp$ be as in \eqref{eq:widehat_J} guarantees that  the average trajectories $\sbm{ z_{1,\text{avg}}\\z_{2,\text{avg}}}$ are bounded and
\begin{equation}\label{eq:ESC_convergence}
\lim_{t\to\infty} \left\|\bbm{z_{1,\text{avg}}(t)-z_1^*\\z_{2,\text{avg}}(t)-z_2^*}\right\| = 0
\end{equation}
holds for all initial conditions in the neighborhood of $\sbm{z_1^*\\z_2^*}$.
\end{proposition}

The proof of Proposition \ref{prop_2} can be found in \hyperref[Appendix]{Appendix}.

\section{3D-Printed Flexible Piezoresistive Flow Sensors}\label{sec:3D-sensor}

\subsection{Sensors design and fabrication}
The flow sensors were designed to be in the form of a soft polymeric cantilever (aspect ratio $= 40$) with graphene-based piezoresistors near its fixed end, as described in our previous work \cite{Kamat2020A}, \cite{Kamat2020B}. Airflow causes the high-aspect ratio cantilever to bend due to the flow-induced drag force, generating mechanical strain near its fixed end and consequently changing the electrical resistance of the serpentine graphene piezoresistors (Figure \ref{sensor_Amar}). This change in resistance, which can be either positive (for tensile strains) or negative (for compressive strains), is then calibrated against the airflow velocity to realize a bidirectional airflow sensor. The cantilever sensor used in this work had a length of $20 mm$, a thickness of $0.5 mm$, and a width that varied from $8 mm$ at its fixed end to $20 mm$ at its free end. This `inverted triangle’ geometry was chosen to enhance the drag force-induced bending strains (and consequently the sensitivity) of the cantilever flow sensor. Further, the design featured serpentine microchannels ($0.3 mm$ width $\times$  $0.15 mm$ depth $\times$ $ 15 mm$ total length) near the fixed end of the cantilever. The analytical and numerical models developed in \cite{Kamat2020B} were used as design guidelines for the above parameters.

\begin{figure}[htbp]
	\centering
	\includegraphics[width=8cm]{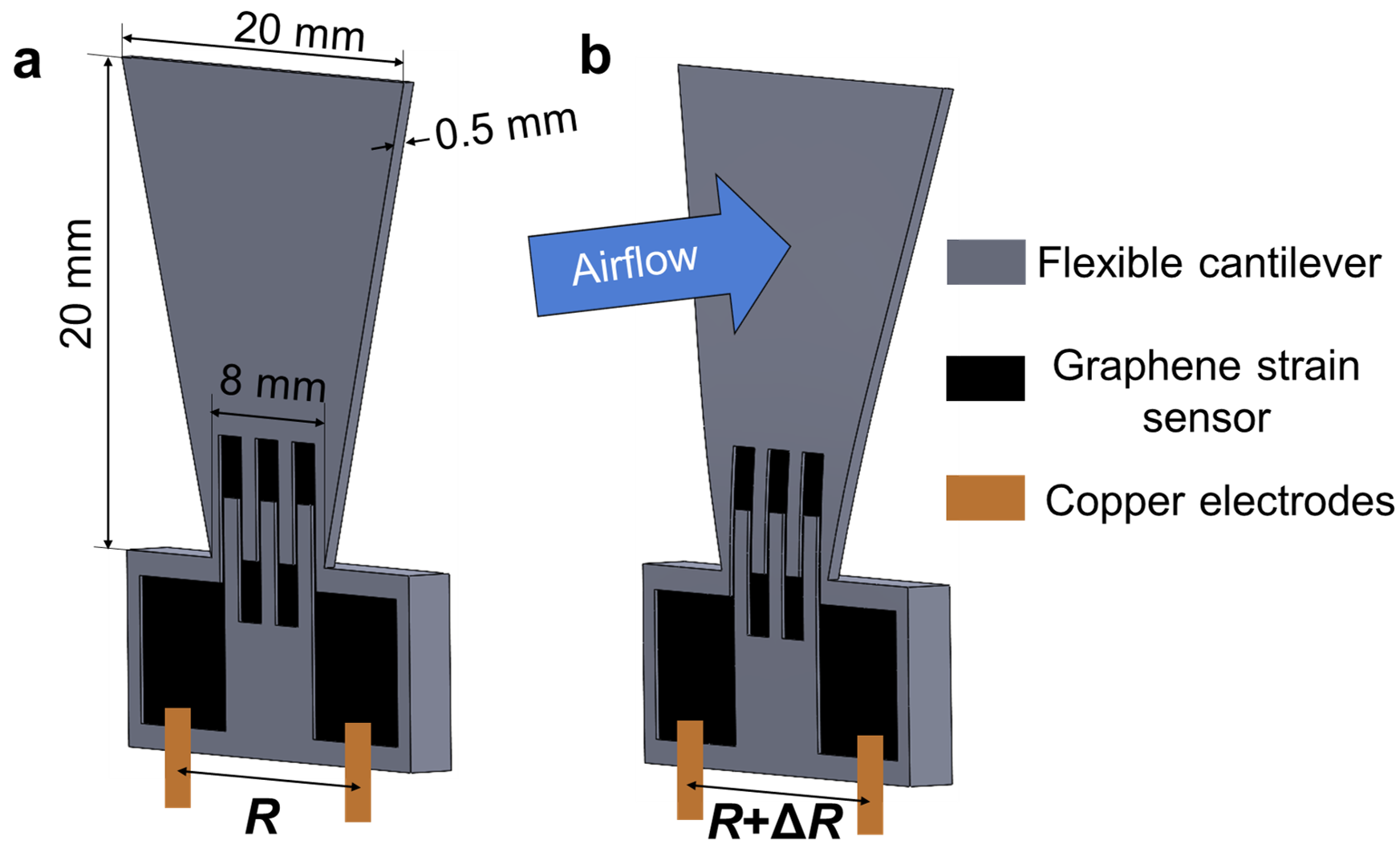}
	\caption{Schematic of the flexible cantilever airflow sensor: a) dimensions of cantilever structure at rest, and b) increase in sensor resistance caused by cantilever bending due to airflow (shown here for the case of tensile strain inducing a positive resistance change).}
	\label{sensor_Amar}
\end{figure}

The cantilever structure was 3D-printed in a commercial low-cost stereolithography (SLA) 3D printer (Form3, Formlabs) using the rubber-like `flexible resin’ (shore hardness  $\sim$ $80A$), a proprietary soft polymer resin offered by Formlabs. This novel approach leveraged recent developments in the 3D printing of soft materials and simplified our fabrication process compared to earlier work \cite{Kamat2020A}, \cite{Kamat2020B}, where 3D printing of a mold and casting of a soft polymer cantilever were performed as two separate steps. The cantilever was printed such that the build direction was parallel to its length to ensure minimal support removal after printing. Post-printing, the flexible cantilever was washed in isopropyl alcohol for $10$ minutes (Formwash, Formlabs) followed by UV-curing at $ \SI{60}{\celsius} $ for $15$ minutes (Formcure, Formlabs) to improve its mechanical properties. A diluted solution of conductive graphene nanoplatelets dispersion (Graphene Supermarket) was then drop casted into the serpentine microchannels and gently annealed ($\SI{100}{\celsius}$ for $1$ hour) to realize the flexible piezoresistive airflow sensor. Finally, the cantilever was mounted at the end of a glass slide and the graphene strain sensor was connected via conductive silver epoxy to copper tape electrodes, which served as connection points of the airflow sensor to the Wheatstone bridge circuit. The resulting sensors (nominal resistance $\sim$  $40-70 k\Omega$) were tested in a custom-built benchtop wind tunnel ($40mm$ $\times$ $40mm$ test section). We tested their piezoresistive response (e.g., the resistance change) to the airflow velocities in the range ($0-5 m/s$) that is of relevance to our source seeking application with the mobile robot. Several batches of sensors were produced and showed consistent performance and repeatability. In Figure \ref{sensor calibration}, a representative calibration result is presented where the sensor shows good sensitivity for low wind velocities on the order of $1m/s$. The Figure also shows the calibration curve that was calculated based on the average of the resulting hysteresis loop from the piezoresistive sensors. We will use this calibration curve later in the experimental setup to fit a linear curve for estimating the wind velocity in the direction of compression or tension. Despite the presence of strong nonlinearity hysteresis behaviour, we will show later in our experiments that our proposed control laws are still able to seek the source succesfully and show that they are robust against such nonlinearities. This is due to the monotonicity property of the sensors (as shown in the figure) that do not alter the extremum point of the potential function. 

We note that the advantages of using such a sensor design and fabrication workflow include simplicity of operation, facile and `cleanroom-free’ fabrication, rapid prototyping (typical fabrication times $\sim$ $2-3$ hours for a batch of multiple sensors) during iterative design, and high sensitivity to airflow ($\sim$  $5$ $k\Omega/ms^{-1}$) due to the combination of a flexible substrate and high-gauge factor graphene sensing elements.

\begin{figure}[htbp]
	\centering
			\centering		
			\includegraphics[width=0.4\textwidth]{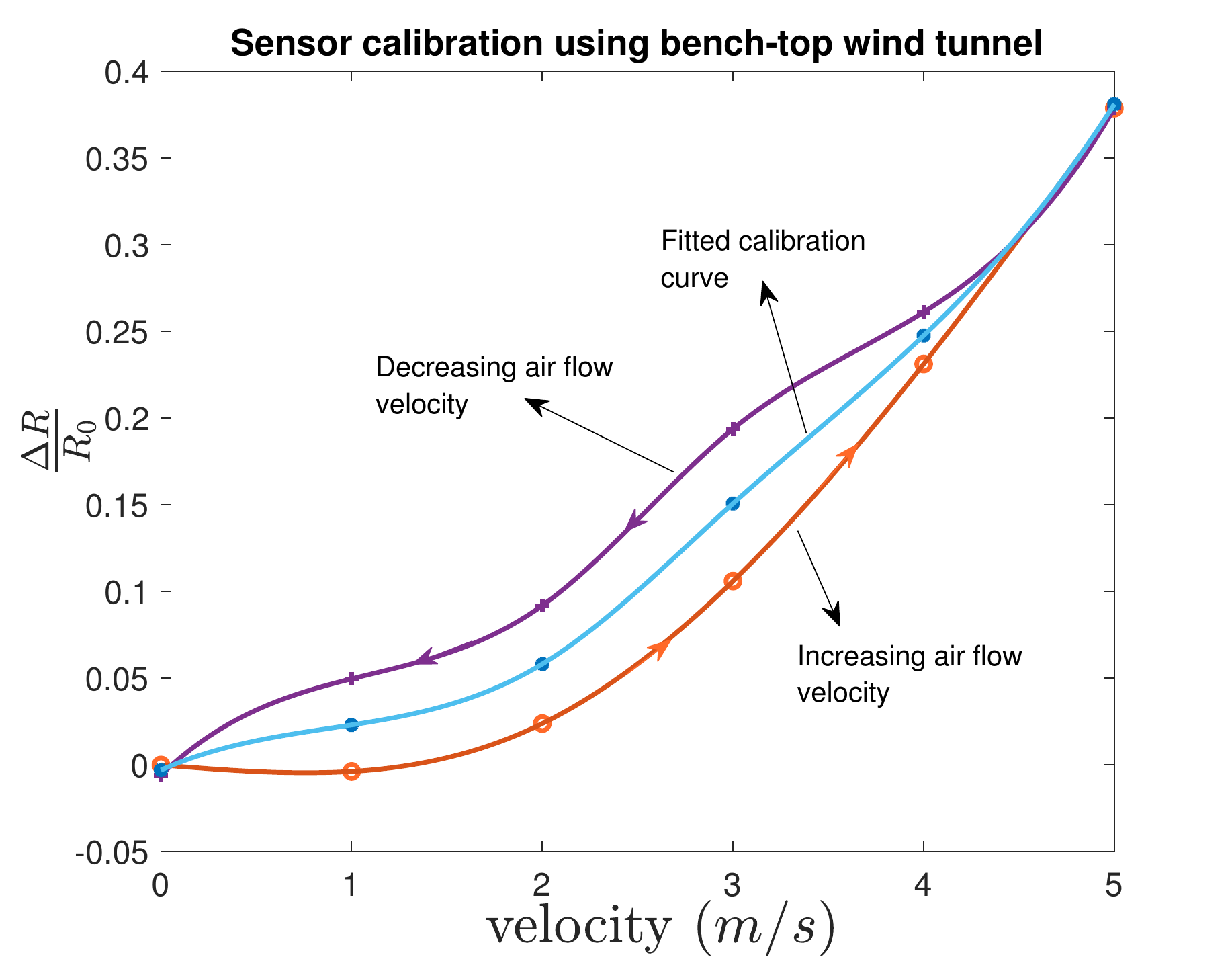}

	\centering
	\caption{Calibration result of a 3D-printed flexible piezoresistive flow sensor using a benchtop wind tunnel. The orange and purple lines show the sensor's response to increasing and decreasing air flow, respectively, where the wind velocity is plotted against the resistance change $\left(\frac{\Delta R}{R_0}\right)$ in the sensors with a cyclic wind load and with nominal resistance $R_0$.
	The calibration curve is shown in blue which is computed based on the average of the cyclic load and gives an approximation of the anhysteresis curve. }
	\label{sensor calibration}
\end{figure}

\subsection{Measurement of local potential gradient}\label{sec:sensor_reading}
In this subsection, we describe the approximation of local potential gradient based on the directional air flow measurement from four sensors placed along the longitudinal and lateral axis (see also Figure \ref{Robot Model with four airflow sensors}).
We assume that for any given airflow direction, it will interact dominantly with, at least, a pair of sensors which are dependent on the wind direction. For instance, in Figure \ref{Robot Model with four airflow sensors}, the sensor pair $1$ and $2$ provide a reliable measurement on the wind flow. Based on this assumption and using the robot-base frame as shown in Figure \ref{Robot Model with four airflow sensors}, we need only to determine the maximum values of the sensor measurements in the longitudinal $z_{1r}$ and lateral $z_{2r}$ axes for obtaining the vector of airflow velocity, e.g.
\begin{align}\label{eq:sensor1}
  S_{z_{1r}} = S_{\mathrm{argmax}_{i \in \{1,3\}}{|S_{i}|}} \\
\label{eq:sensor2}
S_{z_{2r}} = S_{\mathrm{argmax}_{i \in \{2,4\}}{|S_{i}|}}
\end{align}
where $S_{z_{1r}}$ and $S_{z_{2r}}$ are projection of the airflow velocity in the longitudinal and in the lateral direction, respectively, and $S_1, S_2, S_3$ and $S_4$ are the approximated wind flow strength  normal to the flexible cantilever in each sensors. The airflow velocity $v$ in the local coordinate frame is thus given by \[
v=\bbm{S_{z_{1r}}& S_{z_{2r}}}^T.
\]
Correspondingly, if the maximum value comes from sensor 3 or 4, the airflow velocities are considered to be a negative value, as they point into the negative direction of the $z_{1r}$ and $z_{2r}$ axes, respectively. In this paper, we assume that the total airflow strength $S_r$ is related to the projected airflow strength onto a pair of sensors by a trigonometric ratio
\begin{equation}\label{eq:Sr}
S^2_{z_{1r}}+S^2_{z_{2r}}=S^2_r.
\end{equation}

Based on the airflow velocity $v$, the local potential gradient is approximated by
\begin{equation}
\triangledown J(z_1,z_2)=\begin{bmatrix}
J_{z_1}\\
J_{z_2}
\end{bmatrix}\\=\frac{\begin{bmatrix}
S_{z_{1r}} & S_{z_{2r}}
\end{bmatrix}^T}{\begin{Vmatrix}
S_{z_{1r}}&S_{z_{2r}}
\end{Vmatrix}^T}\Delta J
\end{equation}
\begin{equation}
\triangledown ^{\perp} J(z_1,z_2)=\begin{bmatrix}
J_{z_2}\\
-J_{z_1}
\end{bmatrix},
\end{equation}
where $\Delta J$ is the magnitude of the gradient, approximated based on the use of dirty derivative using the current and past magnitude of $S_r$ as in \eqref{eq:Sr} as follows
\[
\Delta J(t_k) = \frac{S_r(t_k)-S_r(t_k-1)}{\left\|\bbm{z_1(t_k)-z_1(t_{k-1})\\z_2(t_k)-z_2(t_{k-1})}\right\|}
\]
where $t_k$ denotes the current discrete-time.

\section{Simulation Results}

In this section, we provide simulation results to validate our proposed control laws to solve the source seeking problem. Throughout the section, we consider
a stationary source located at $(0,0)$ where the wind field is given either by a quadratic function $J(z_1,z_2)=-z^2_1-z^2_2$ or by a non-quadratic one $J(z_1,z_2)=-z^2_1-(z^2_2-z^3_1)^2$,or $J(z_1,z_2)=-z^2_1-(z_2-z^2_1)^2$.

\subsection{Projected gradient-ascent control law}
For numerical validation of the projected gradient-ascent control law in \eqref{eq:control_laws}, we evaluate a number of different values of parameters $k_1$ and $k_2$.
 For the first case, where we consider the  quadratic function $J(z_1,z_2)=-z^2_1-z^2_2$ as the wind field, we consider four randomly chosen initial positions $\sbm{z_{10}\\z_{20} \\ \theta_0}$ of the mobile robot:
\[
\bbm{z_{10}\\z_{20} \\ \theta_0} \in \left\{\bbm{4\\3\\30^o}, \bbm{-3\\3\\45^o}, \bbm{-2\\-4\\60^o}, \bbm{3\\-2\\90^o}\right\}.
\]
and the resulting trajectories of robots using the control law \eqref{eq:control_laws} are shown in Figure \ref{Trajectories of robot from four different initial positions} which is plotted using the global frame, for clarity. The simulation results confirm the theoretical results in Proposition \ref{prop_1}, where the extremum point of $J$, which is $(0,0)$, is (globally) attractive.

\begin{figure}[htbp]
	\centering
	
	\subfigure[$ t=1000$]{
		\begin{minipage}[t]{0.3\linewidth}
			\centering			\includegraphics[width=1\textwidth]{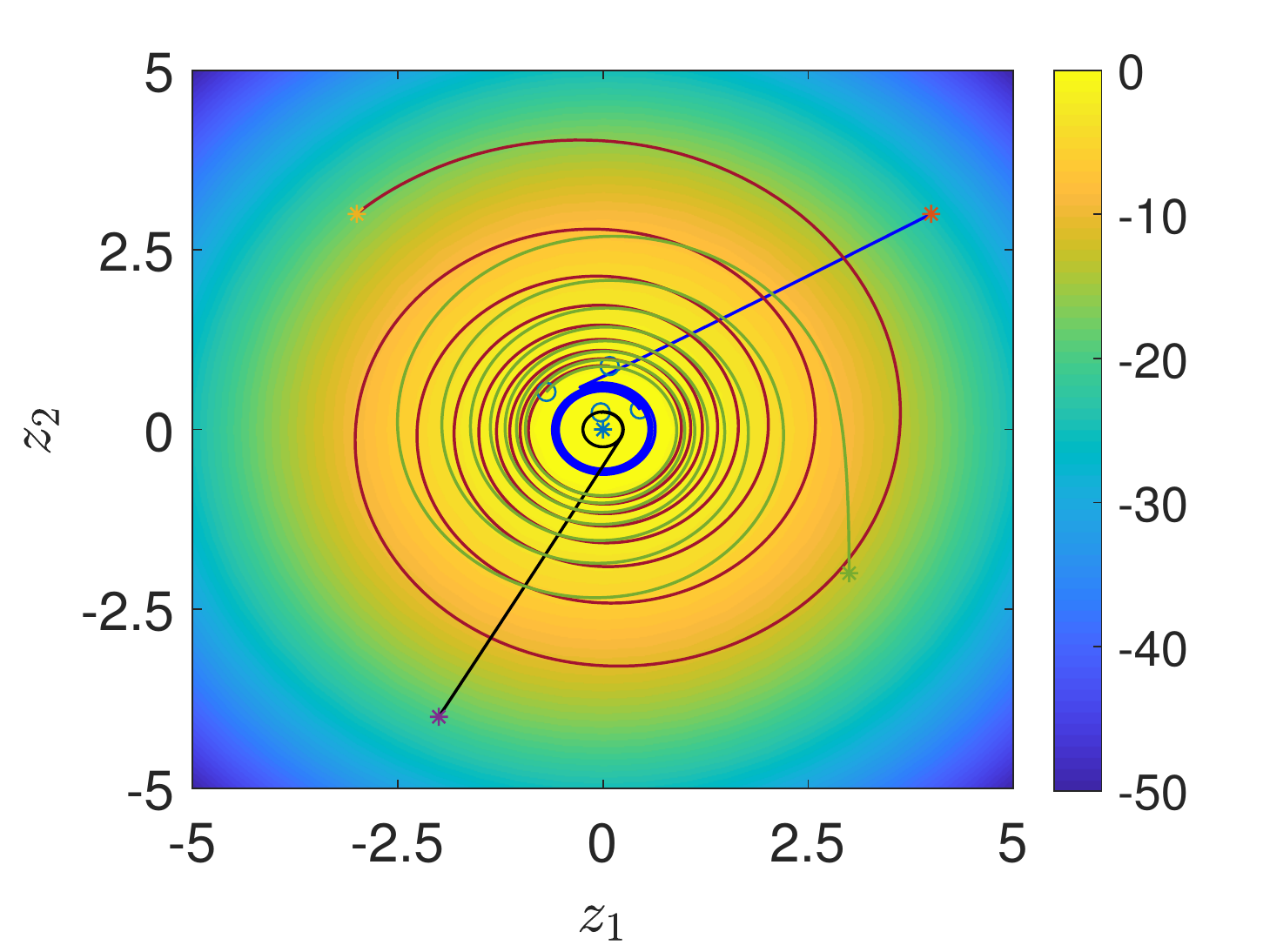}
		\end{minipage}%
	}%
	\subfigure[ $t=100$]{
		\begin{minipage}[t]{0.3\linewidth}
			\centering			\includegraphics[width=1\textwidth]{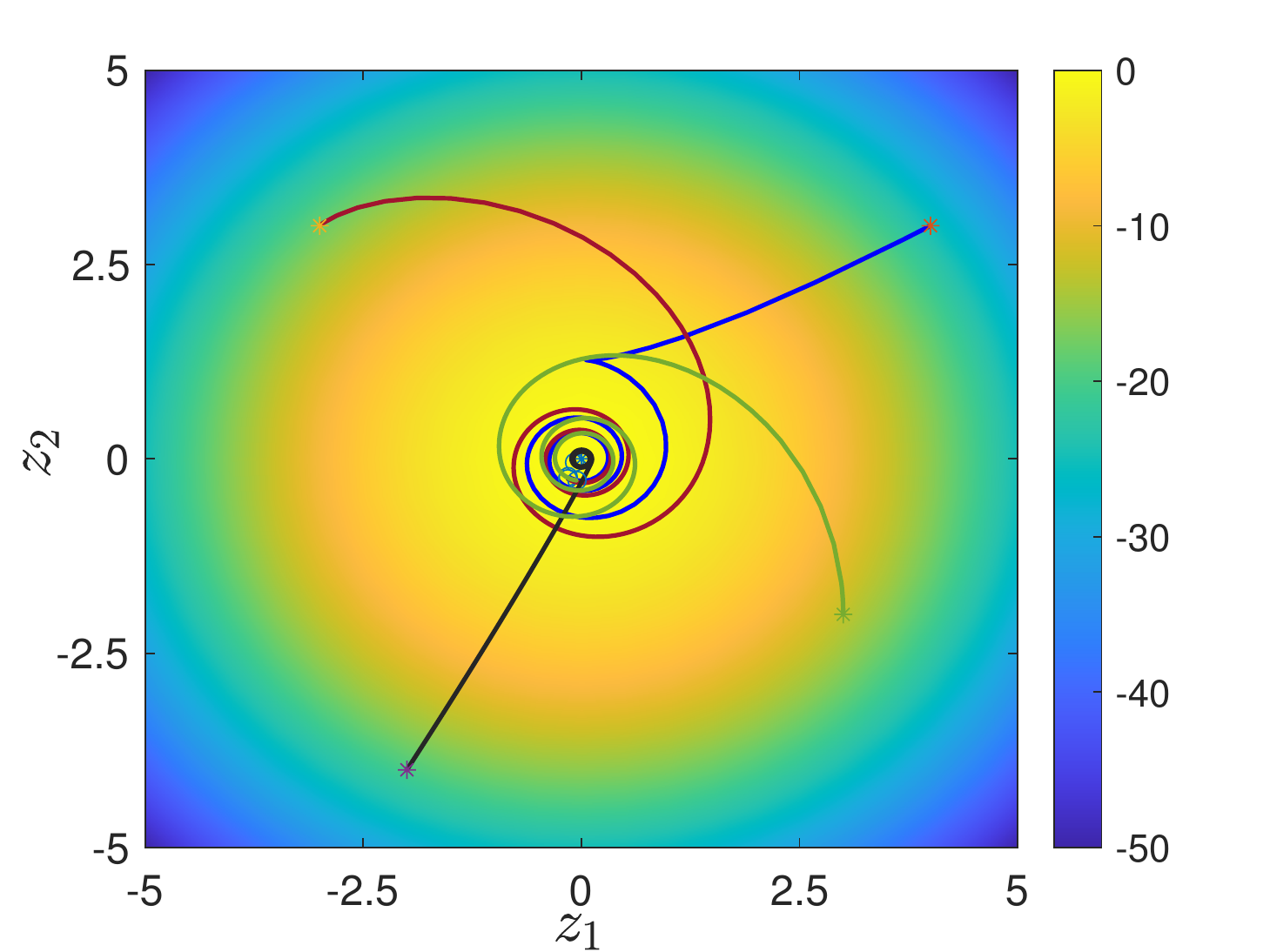}
		\end{minipage}%
	}%
	\subfigure[ $t=50$]{
		\begin{minipage}[t]{0.3\linewidth}
			\centering			\includegraphics[width=1\textwidth]{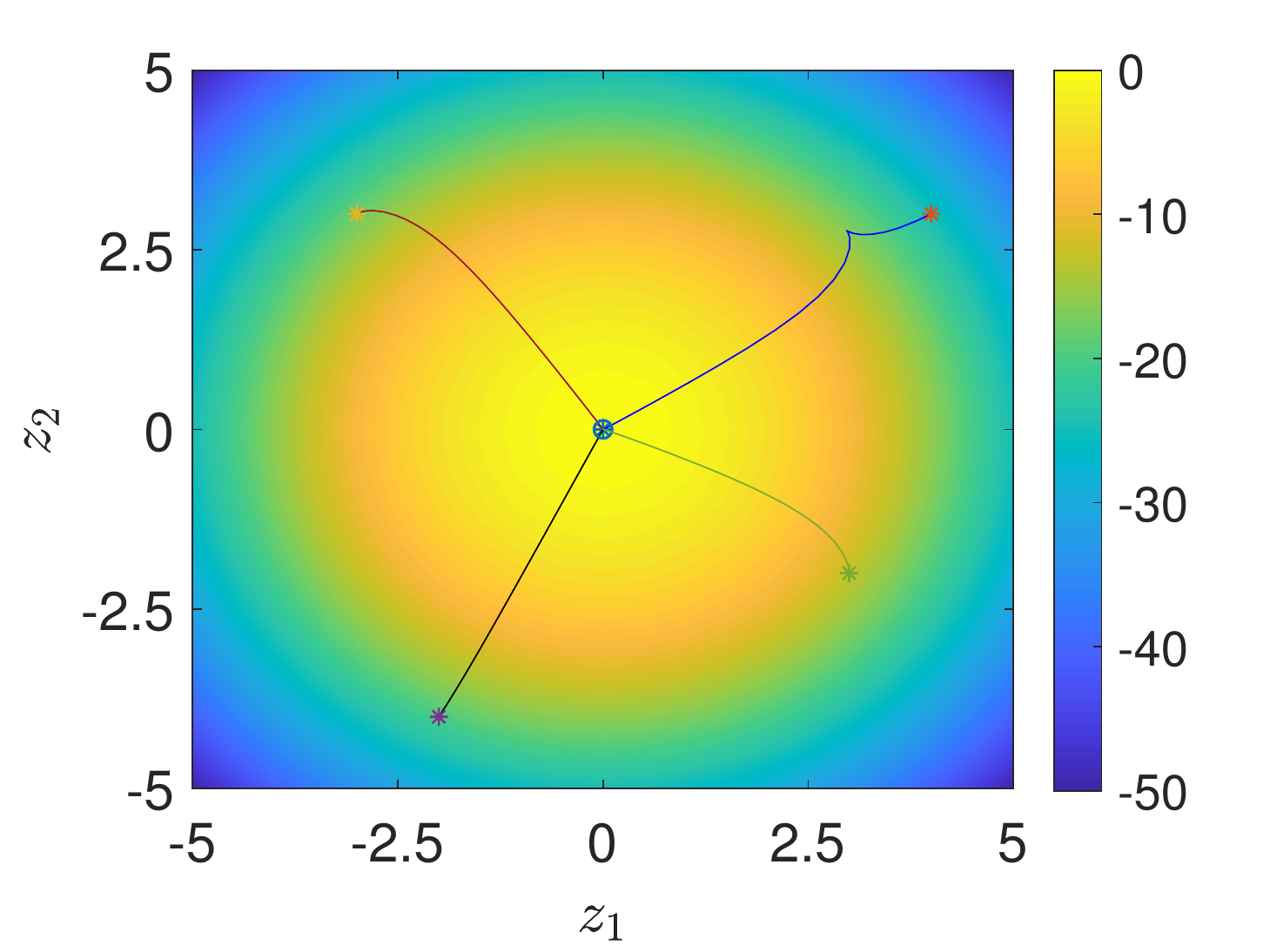}
		\end{minipage}%
	}%
	\centering
	\caption{Simulation results of the closed-loop system using the projected gradient-ascent control law \eqref{eq:control_laws} from four randomly chosen initial conditions, where the potential function $J$ is given by a quadratic function $J(z_1,z_2)=-z^2_1-z^2_2$, and the longitudinal velocity gain $k_1$ and angular velocity gain $k_2$ are set to be: (a). $k_1=1, k_2=1, t=1000$; (b). $k_1=1, k_2=10, t=100$; (c). $k_1=0.1, k_2=10, t=50$. The indicated time $t$ beneath each figure gives the total running time of the robot in approaching the source.}
	\label{Trajectories of robot from four different initial positions}
\end{figure}



We denote $T_s$ as the time to settle from a given initial condition to within $20\%$ of the source location.  The robot is initialized at 100 randomly chosen initial positions, the simulation results in Figure \ref{boxplot of k} shows the effect of increasing longitudinal velocity gain $k_1$ or angular velocity gain $k_2$ on the $T_s$ of robot as another parameter is fixed in quadratic map. As shown in this figure, the robot achieves a faster source seeking motion when small $k_1$ and large $k_2$ are used.

\begin{figure}[htbp]
	\centering

	\subfigure[Fixed $k_2=10 $]{
		\begin{minipage}[t]{0.5\linewidth}
			\centering			\includegraphics[width=1\textwidth]{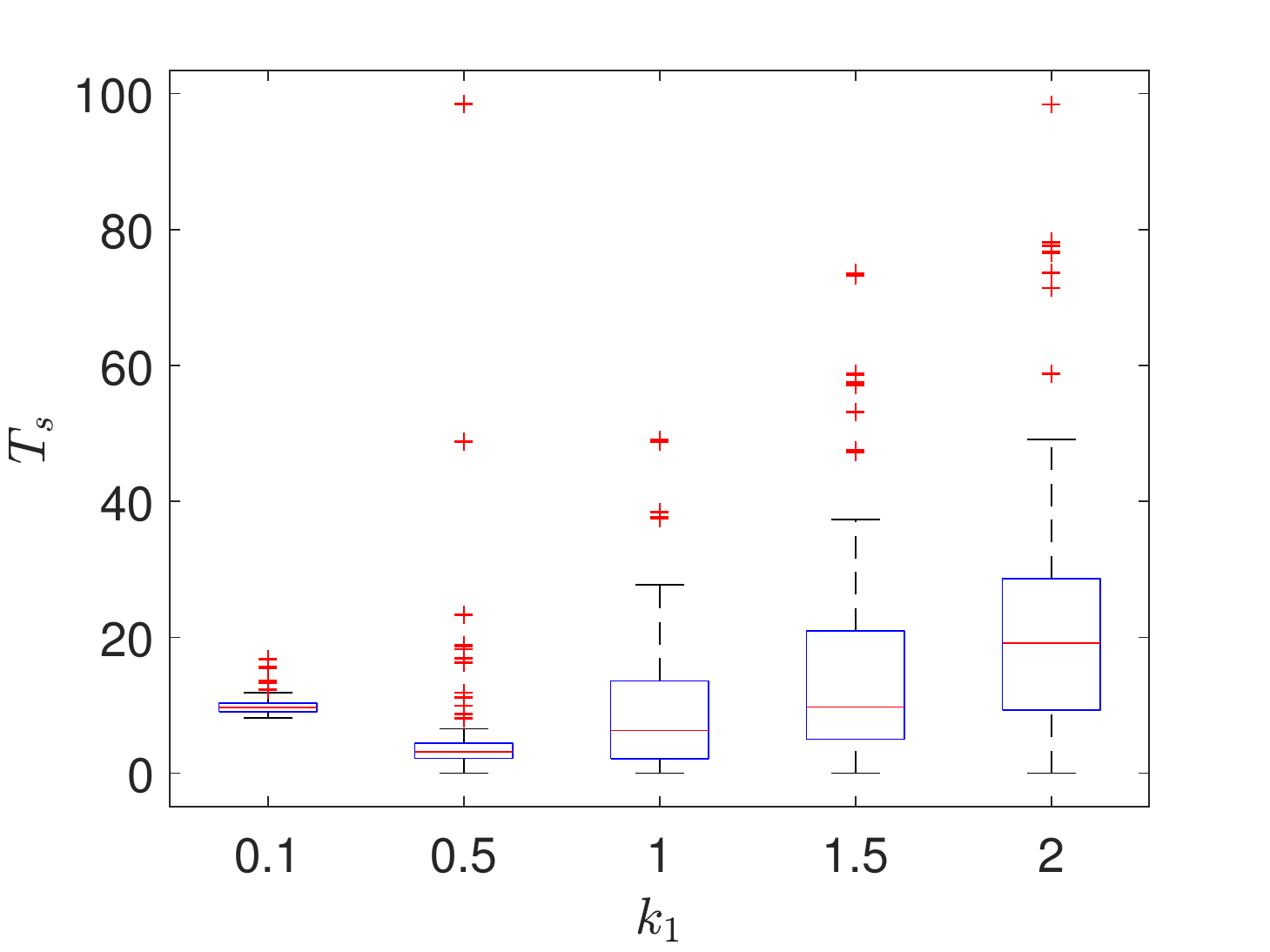}
		\end{minipage}%
	}%
	\subfigure[Fixed $k_1=0.5$]{
		\begin{minipage}[t]{0.5\linewidth}
			\centering			\includegraphics[width=1\textwidth]{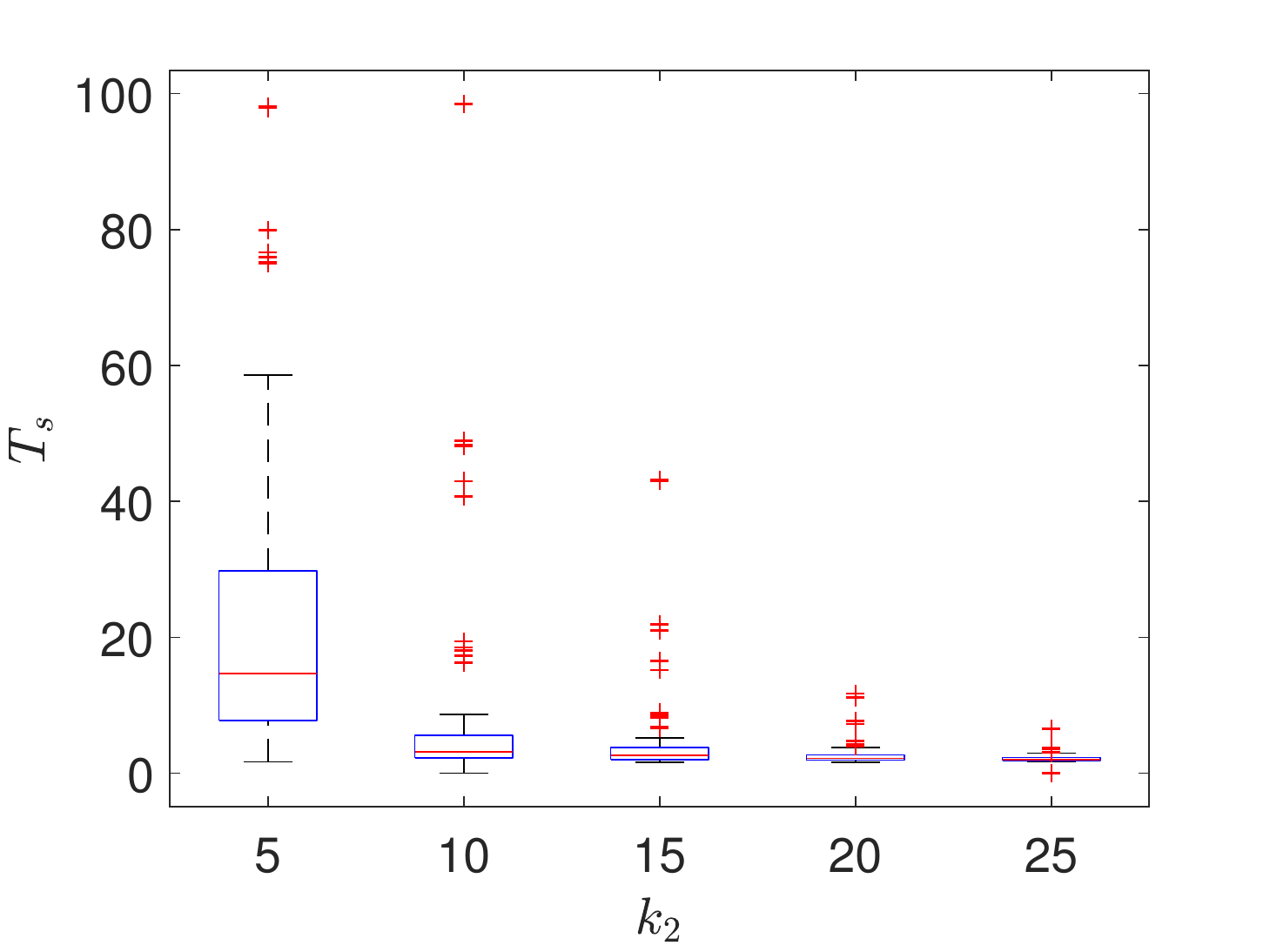}
		\end{minipage}%
	}%
	\centering
	\caption{The boxplot of $T_s$ w.r.t the increasing (a). $k_1$, and (b). $k_2$ in the quadratic map, using the projected gradient-ascent control law. $T_s$ denotes the time for robot approaching the final $20\%$ distance from the source, each group of robot is initialized at 100 randomly chosen initial positions.}
	\label{boxplot of k}
\end{figure}

Similar results are also obtained for the non-quadratic potential function  $J(z_1,z_2)=-z^2_1-(z^2_2-z^3_1)^2$. Figure \ref{four trajectories in non-quadratic } shows the trajectories of the closed-loop systems when they are randomly initialized at
\[
\bbm{z_{10}\\z_{20} \\ \theta_0} \in \left\{\bbm{0.9\\0.6\\30^o}, \bbm{-0.7\\0.6\\45^o}, \bbm{-0.5\\-0.9\\60^o}, \bbm{0.8\\-0.5\\90^o}\right\}.
\]
As the given non-quadratic function $J$ satisfies the hypotheses in Proposition \ref{prop_1}, namely, strictly concave and twice-differentiable with a global maximum at $(0,0)$, all trajectories converge to the origin.


\begin{figure}[htbp]
	\centering
	
	\subfigure[$t=100$]{
		\begin{minipage}[t]{0.3\linewidth}
			\centering			\includegraphics[width=1\textwidth]{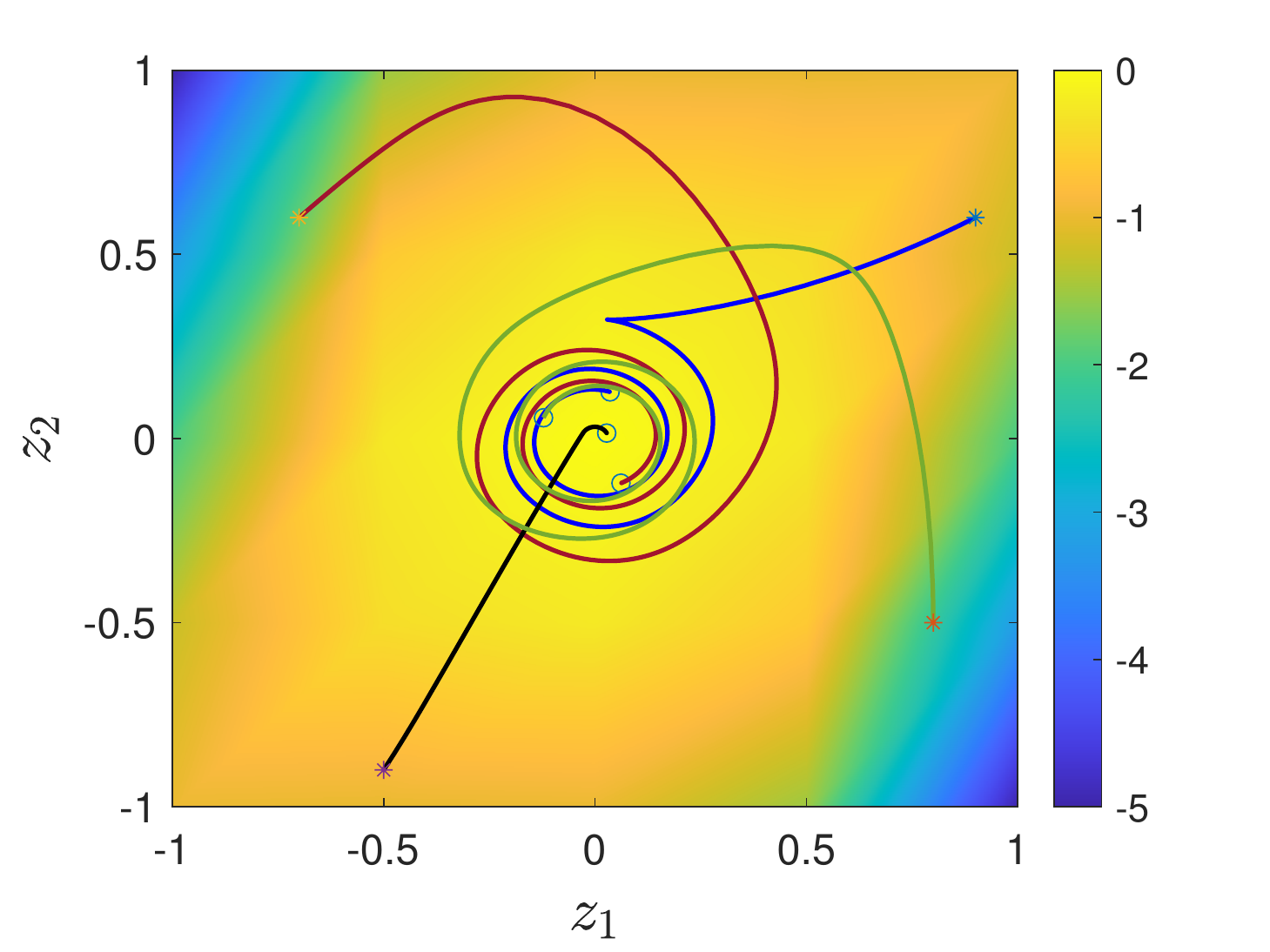}
		\end{minipage}%
	}%
	\subfigure[ $ t=50$]{
		\begin{minipage}[t]{0.3\linewidth}
			\centering			\includegraphics[width=1\textwidth]{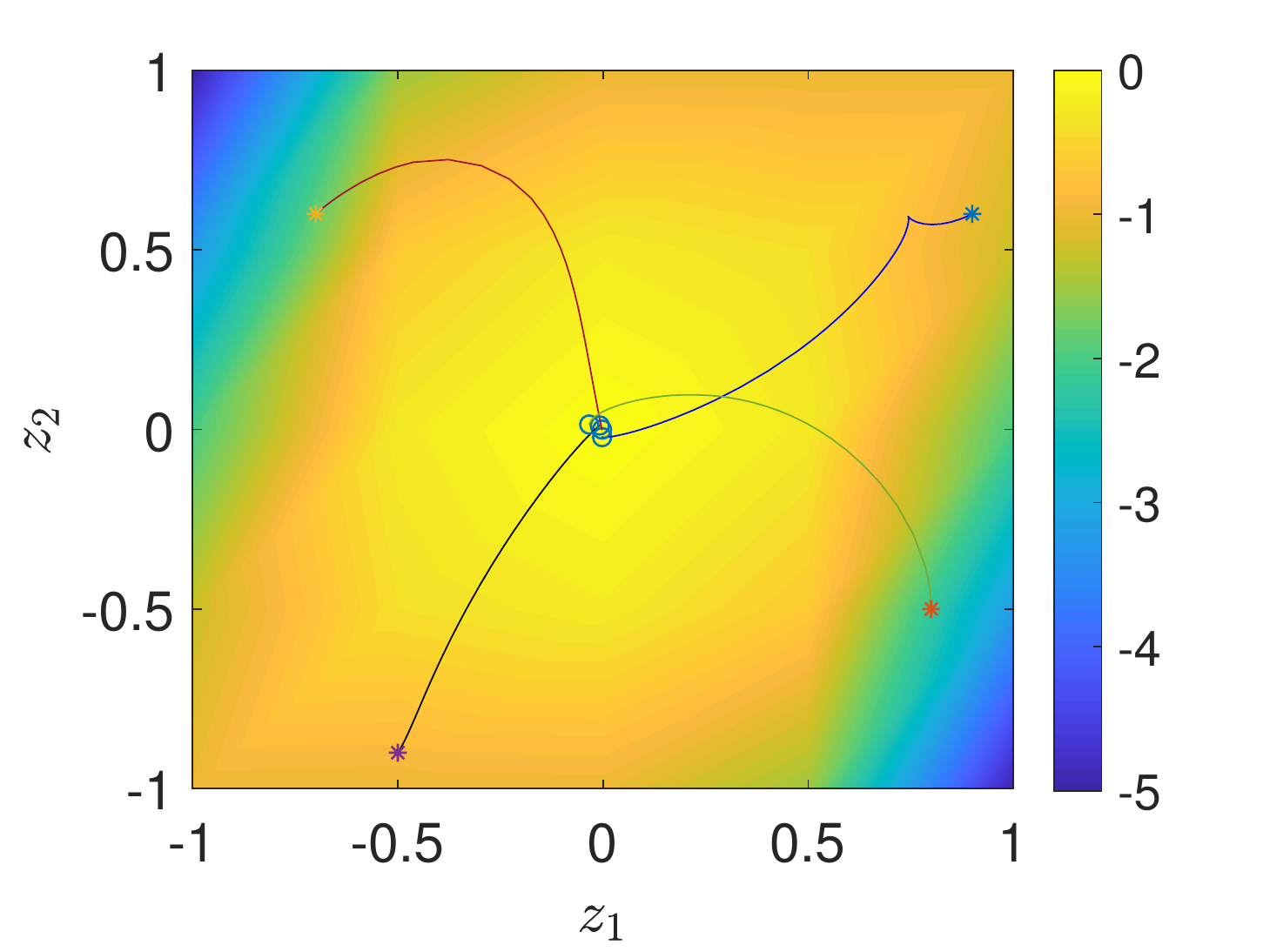}
		\end{minipage}%
	}%
	\subfigure[$t=500$]{
		\begin{minipage}[t]{0.3\linewidth}
			\centering			\includegraphics[width=1\textwidth]{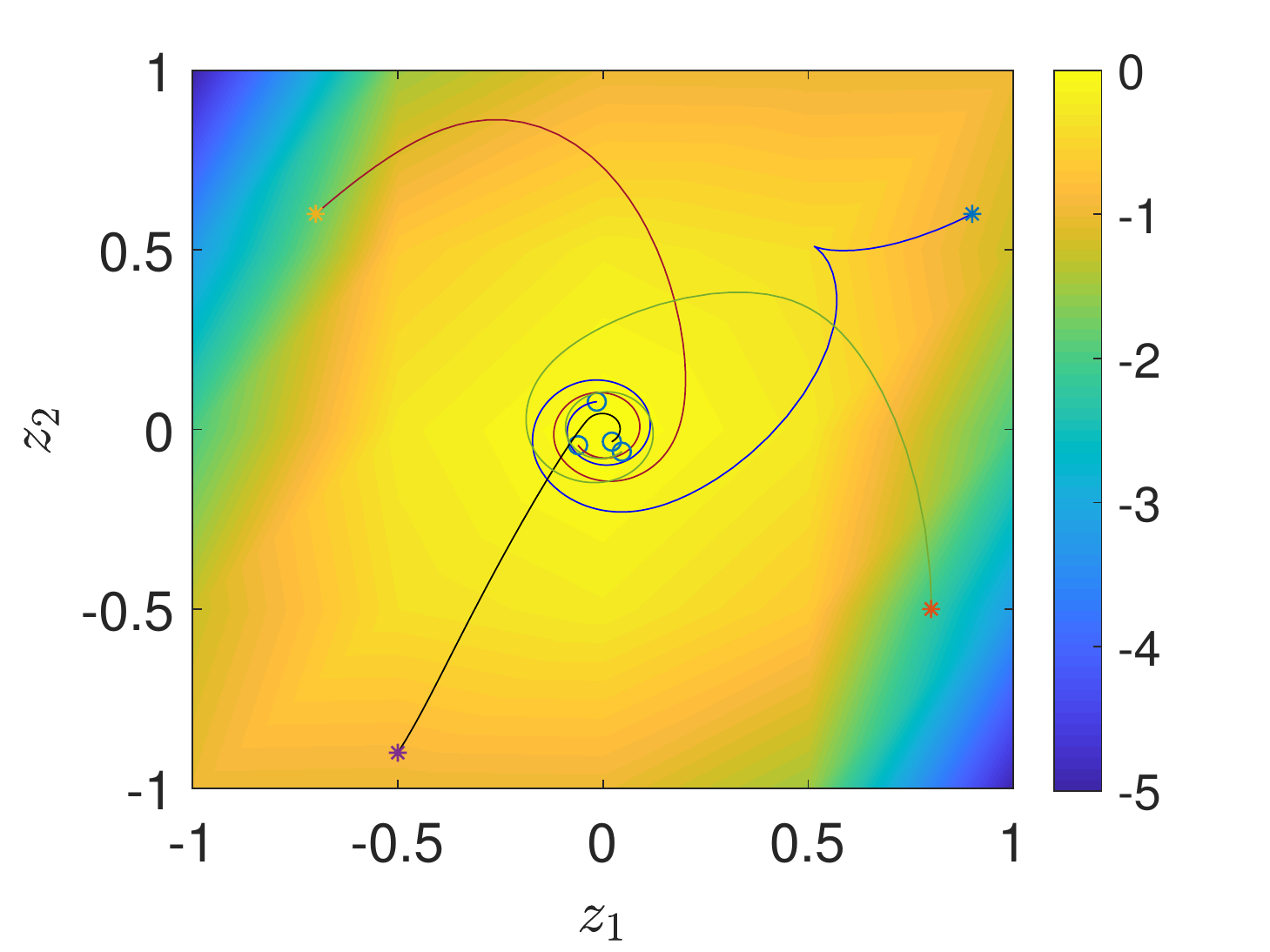}
		\end{minipage}%
	}%
	\centering
	\caption{ Simulation results of the closed-loop system using the projected gradient-ascent control law \eqref{eq:control_laws}, where the potential function $J$ is given by a non-quadratic function$J(z_1,z_2)=-z^2_1-(z^2_2-z^3_1)^2$, and the longitudinal velocity gain $k_1$ and angular velocity gain $k_2$ are set to be: (a). $k_1=1, k_2=20$; (b). $k_1=0.1, k_2=20$; (c). $k_1=0.1, k_2=5 $. The indicated time $t$ beneath each figure gives the total running time of the robot in approaching the source.}
	\label{four trajectories in non-quadratic }
\end{figure}

As Figure \ref{four trajectories in non-quadratic } indicated, the effect of gain $k_1$ and $k_2$ is in accordance with the previous observation from Figure \ref{boxplot of k}.

\subsection{ESC-based projected gradient-ascent control law}
In this subsection, we will numerically validate the ESC-based controller that is studied in Proposition \ref{prop_2}. Similar as before, we evaluate its efficacy in dealing with both quadratic as well as non-quadratic potential function. The latter is relevant since we have proven in Proposition \ref{prop_2} only for the quadratic case.

On the other hand, when we consider the quadratic potential function as before, Figure \ref{Trajectories of robot in quadratic map}(a) and (b) show the trajectories of the closed-loop mobile robot with the ESC-based control law in \eqref{eq:ESC_control_laws} where the initial conditions are randomly set at
\[
\bbm{z_{10}\\z_{20} \\ \theta_0} \in \left\{\bbm{-7\\6\\90^o}, \bbm{6\\8\\30^o}\right\},
\]
and the control parameters are given by $a=0.2, h=3,C_{z1}= 0.5, C_{z2}=0.5,\omega_0=10 rad/s, k_1=1, k_2=20$. The sub-figures show that the ESC-based control law solves the source-seeking problem as expected from Proposition \ref{prop_2} for quadratic potential function. On the other hand, when we take the non-quadratic potential function  $J(z_1,z_2)=-z^2_1-(z_2-z^2_1)^2$, Figure \ref{Trajectories of robot in quadratic map}(c) and (d) show that the ESC-based control law with initial conditions set randomly at
\[
\bbm{z_{10}\\z_{20} \\ \theta_0} \in \left\{\bbm{-1\\2\\-60^o}, \bbm{1.5\\2.5\\-90^o}\right\},
\]
and with the control parameters $a=0.2,h=2,C_{z1}= 0.1, C_{z2}=0.1,\omega_0=10 rad/s, k_1=1, k_2=20$, is still able to seek the source reliably.

\begin{figure}[htbp]
	\centering
	
	\subfigure[Initial position $\protect\sbm{-7 \protect\\6 \protect\\90^o}$]{
		\begin{minipage}[t]{0.5\linewidth}
			\centering			\includegraphics[width=1\textwidth]{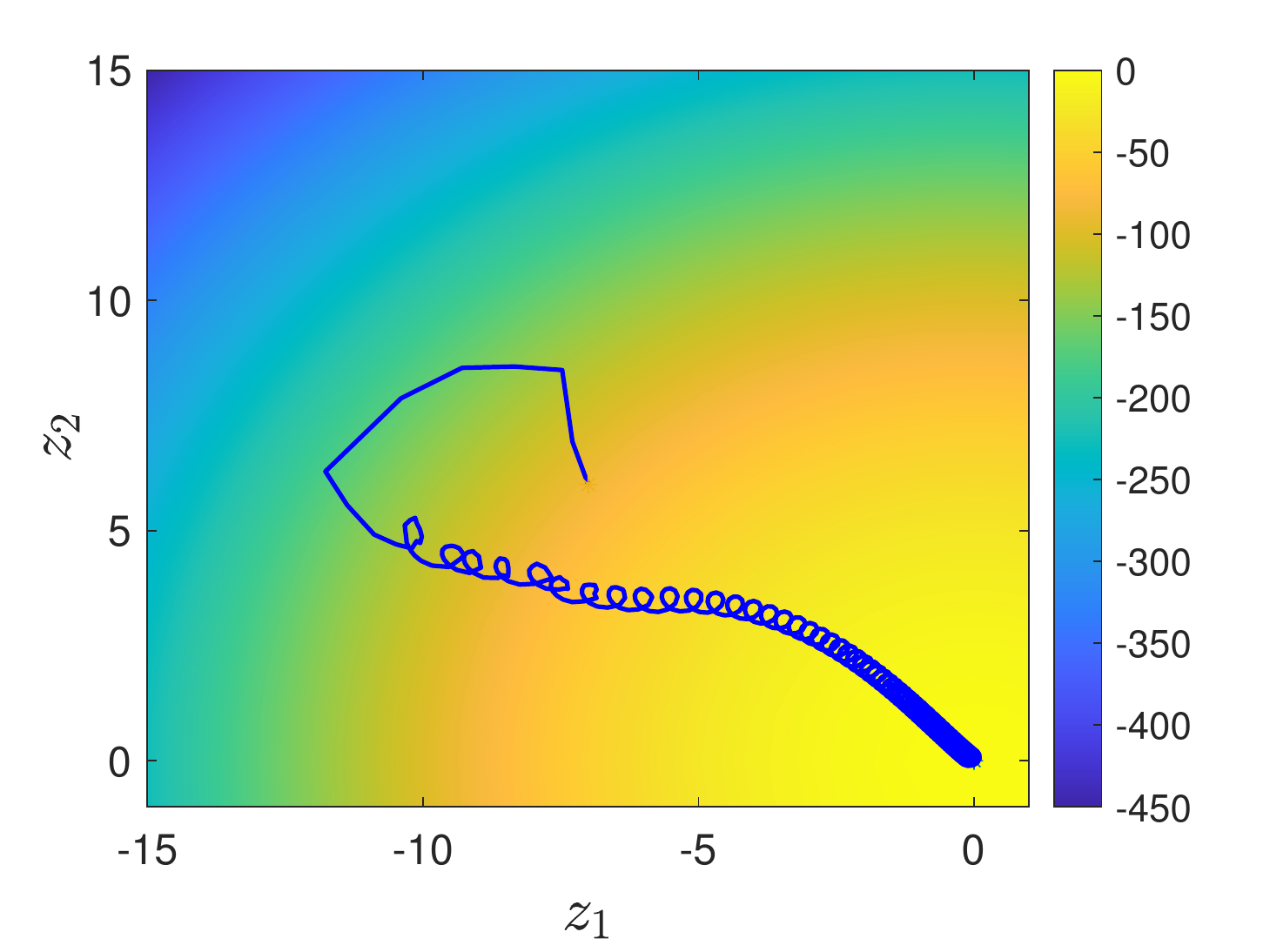}
		\end{minipage}%
	}%
	\subfigure[Initial position $\protect\sbm{6 \protect\\8 \protect\\30^o}$]{
		\begin{minipage}[t]{0.5\linewidth}
			\centering
		\includegraphics[width=1\textwidth]{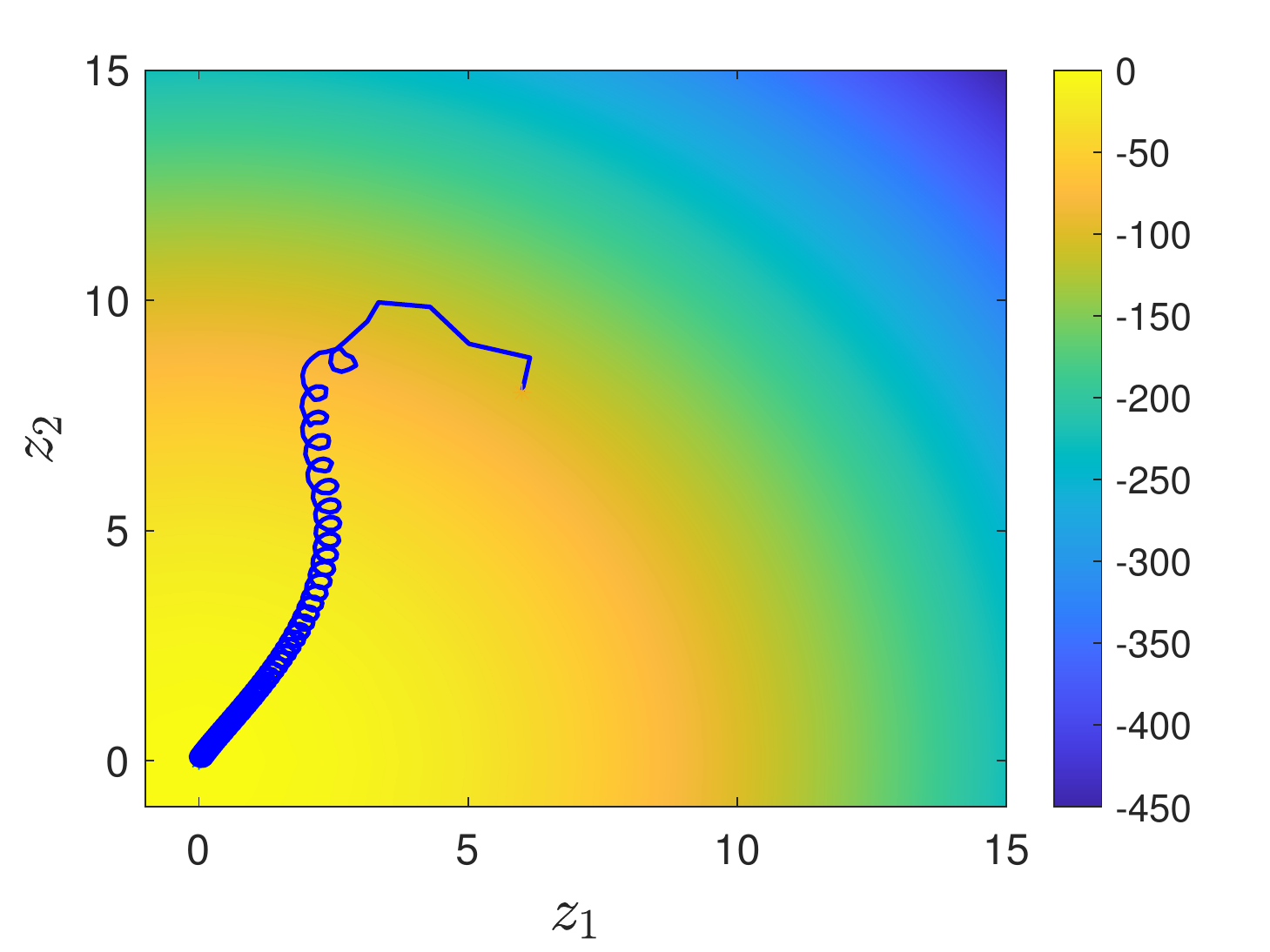}
		\end{minipage}%
	}%
	
	\subfigure[Initial position $\protect\sbm{-1 \protect\\2 \protect\\-60^o}$]{
		\begin{minipage}[t]{0.5\linewidth}
			\centering
			\includegraphics[width=1\textwidth]{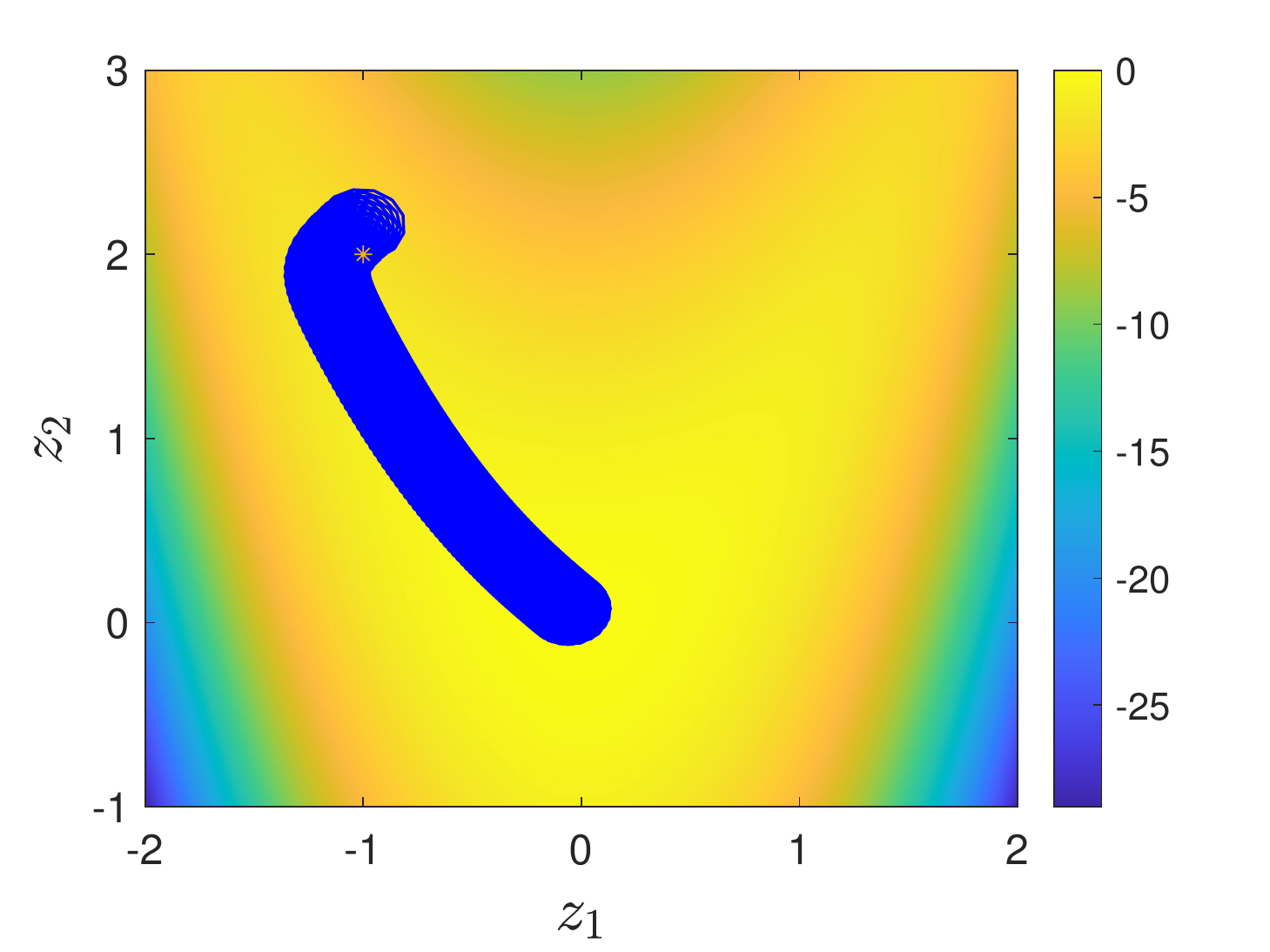}
		\end{minipage}%
	}%
	\subfigure[Initial position $\protect\sbm{1.5 \protect\\2.5\protect\\-90^o}$]{
		\begin{minipage}[t]{0.5\linewidth}
			\centering
			\includegraphics[width=1\textwidth]{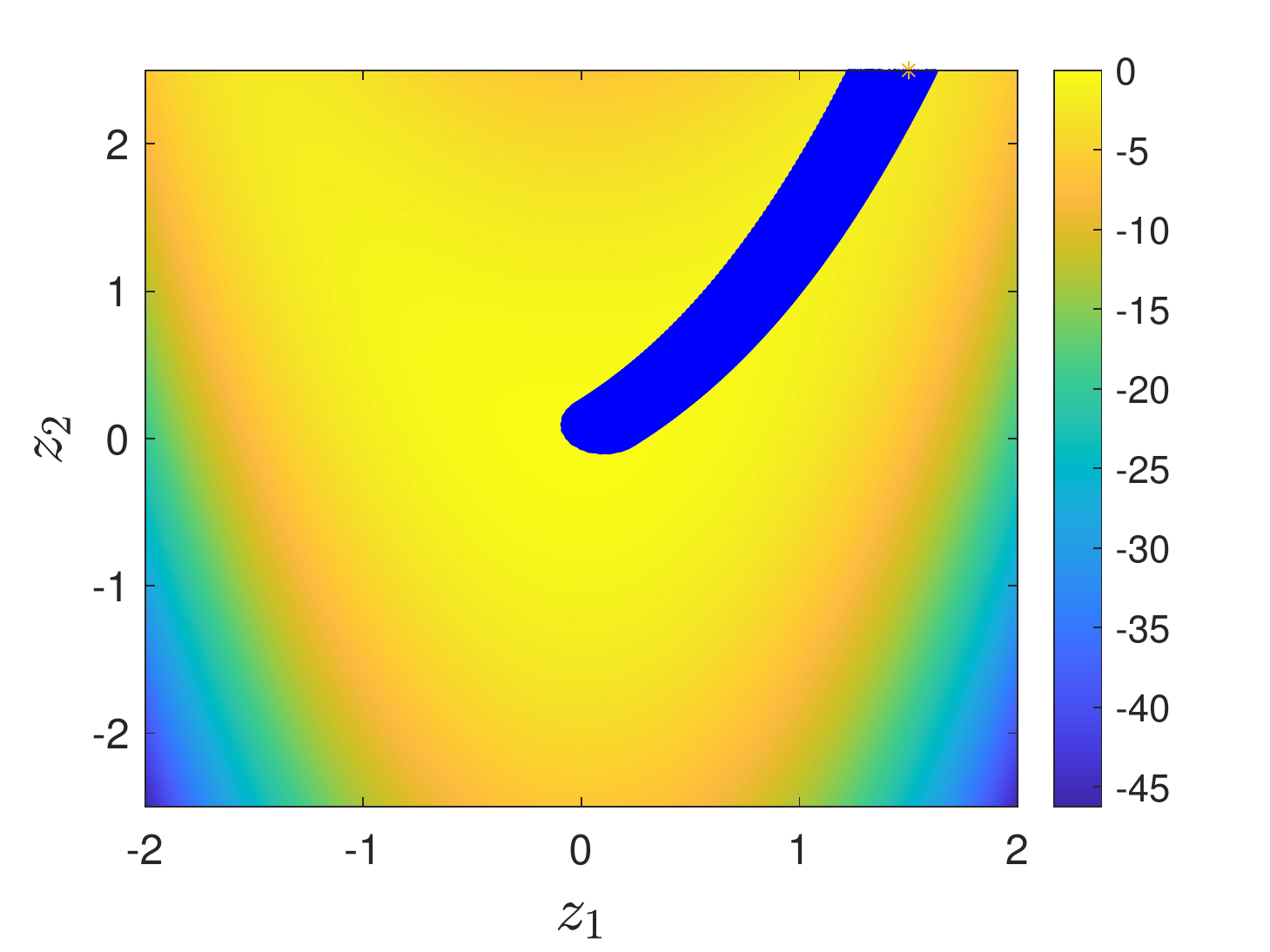}
		\end{minipage}%
	}%

	\centering
	\caption{Simulation results of the closed-loop system using the ESC-based projected gradient-ascent control law as in Proposition \ref{prop_2}. The case of quadratic potential function is shown in (a) and (b) and that of non-quadratic one ( $J(z_1,z_2)=-z^2_1-(z_2-z^2_1)^2$) is shown in (c) and (d) with the associated random initial conditions. }
	\label{Trajectories of robot in quadratic map}
\end{figure}

Using the same set of initial state, Figure \ref{four w} and \ref{p_t_ESC_w_z} show the robot trajectories based on various dither signal $\omega_0$ in a quadratic map. Subsequently, we performed a Monte Carlo simulation to validate the performance of the controller using 
100 random initial positions. The resulting analysis is presented in the boxplot in Figure \ref{box_w}. In this figure, $T_s$ represents the time for the robot to reach the final $20\%$ of the distance between the source and its initial position. We can conclude from this analysis that  the larger the frequency $\omega_0$, the more time the robot takes to reach the source.
A similar observation on the effect of dither frequency $\omega_0$ to the settling time in a non-quadratic map is shown in Figure \ref{four-w-non} and \ref{p_t_ESC_w_z_n}.

\begin{figure}[htbp]
	\centering
	
	\subfigure[ $\omega_0=3 rad/s$, $t=20$]{
		\begin{minipage}[t]{0.5\linewidth}
			\centering			\includegraphics[width=1\textwidth]{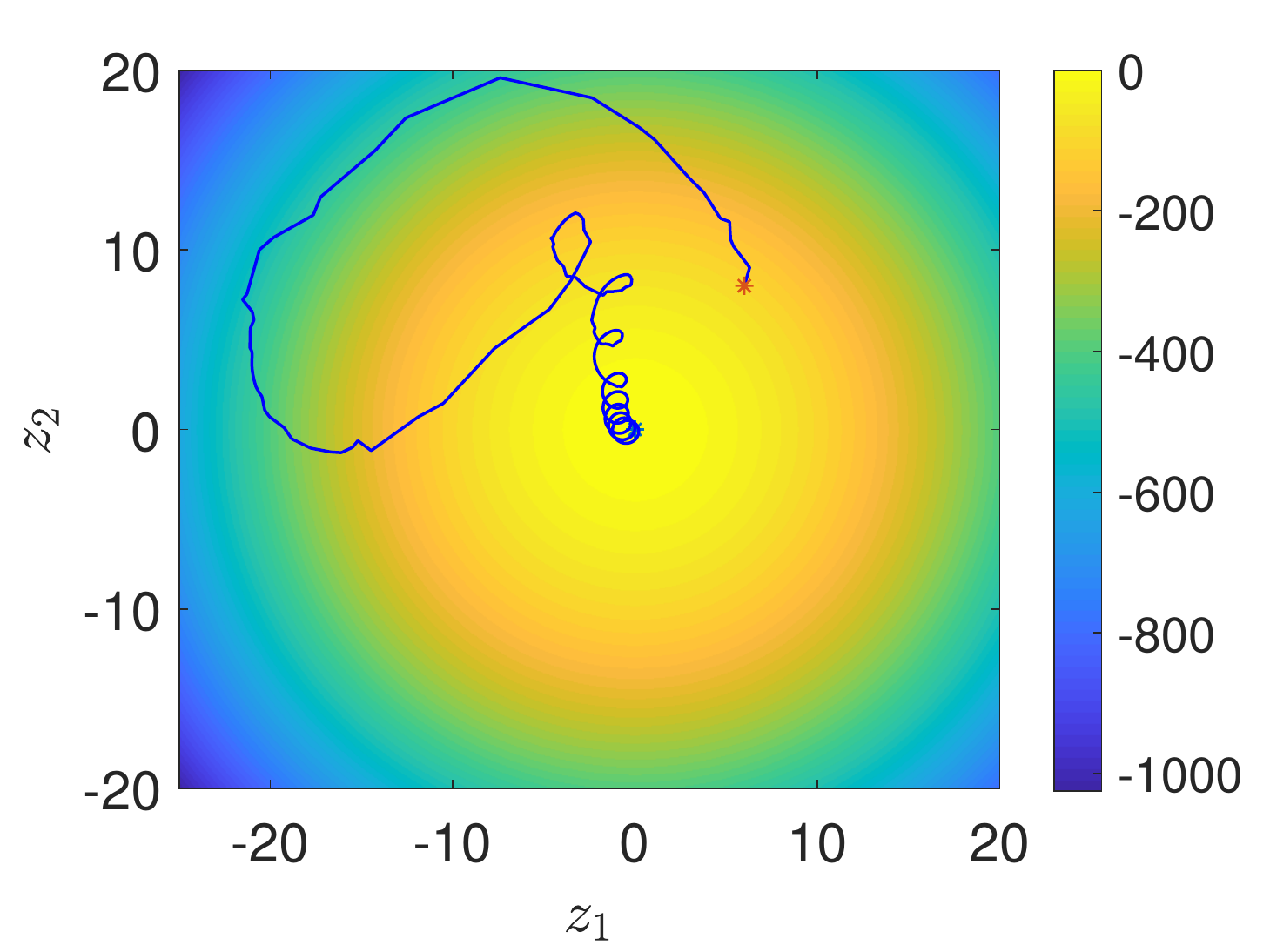}
		\end{minipage}%
	}%
	\subfigure[ $\omega_0=10$ rad/s, $t=50$]{
		\begin{minipage}[t]{0.5\linewidth}
			\centering			\includegraphics[width=1\textwidth]{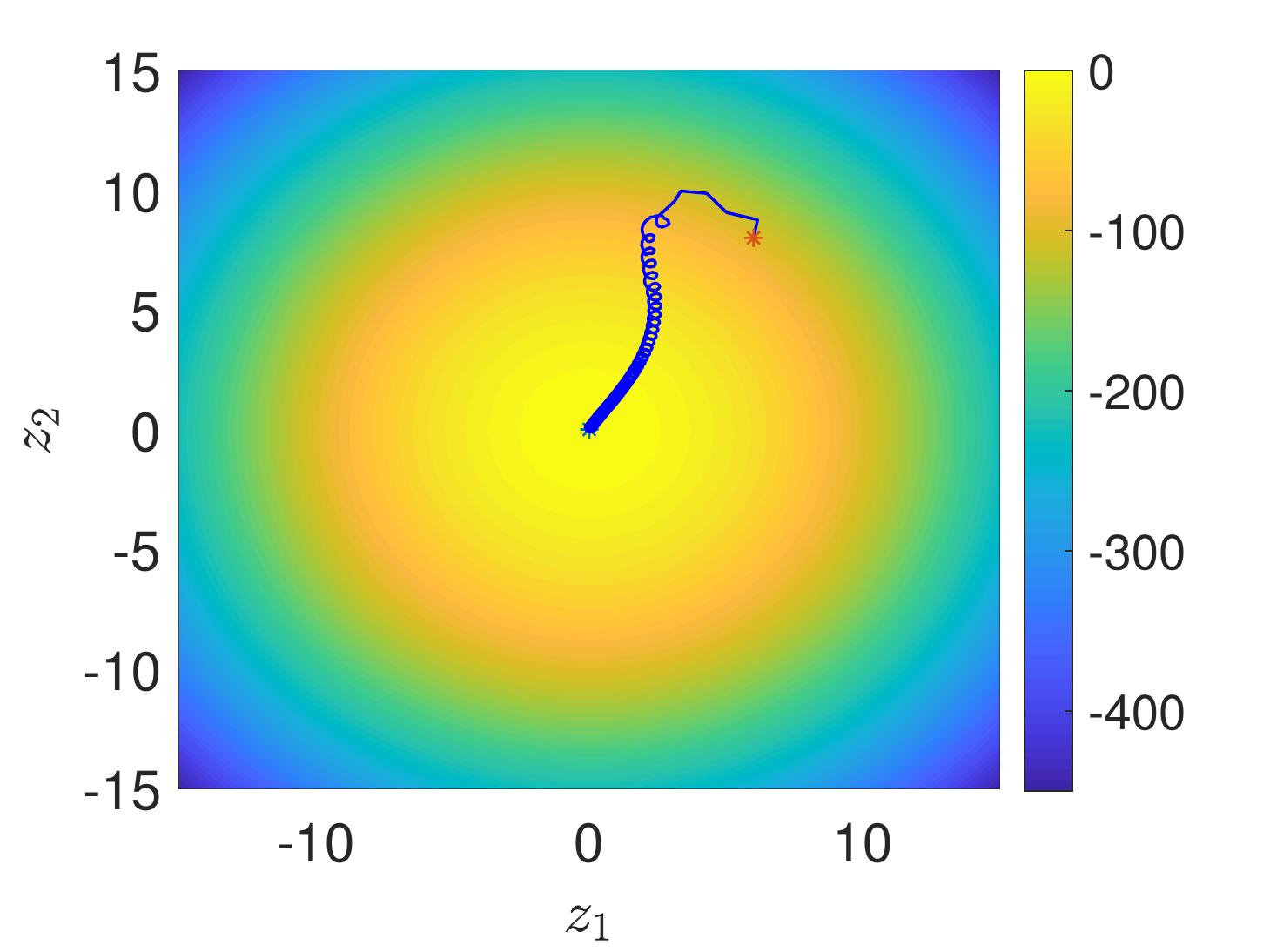}
		\end{minipage}%
	}%
	
	\subfigure[ $\omega_0=60 rad/s$, $t=350$]{
		\begin{minipage}[t]{0.5\linewidth}
			\centering			\includegraphics[width=1\textwidth]{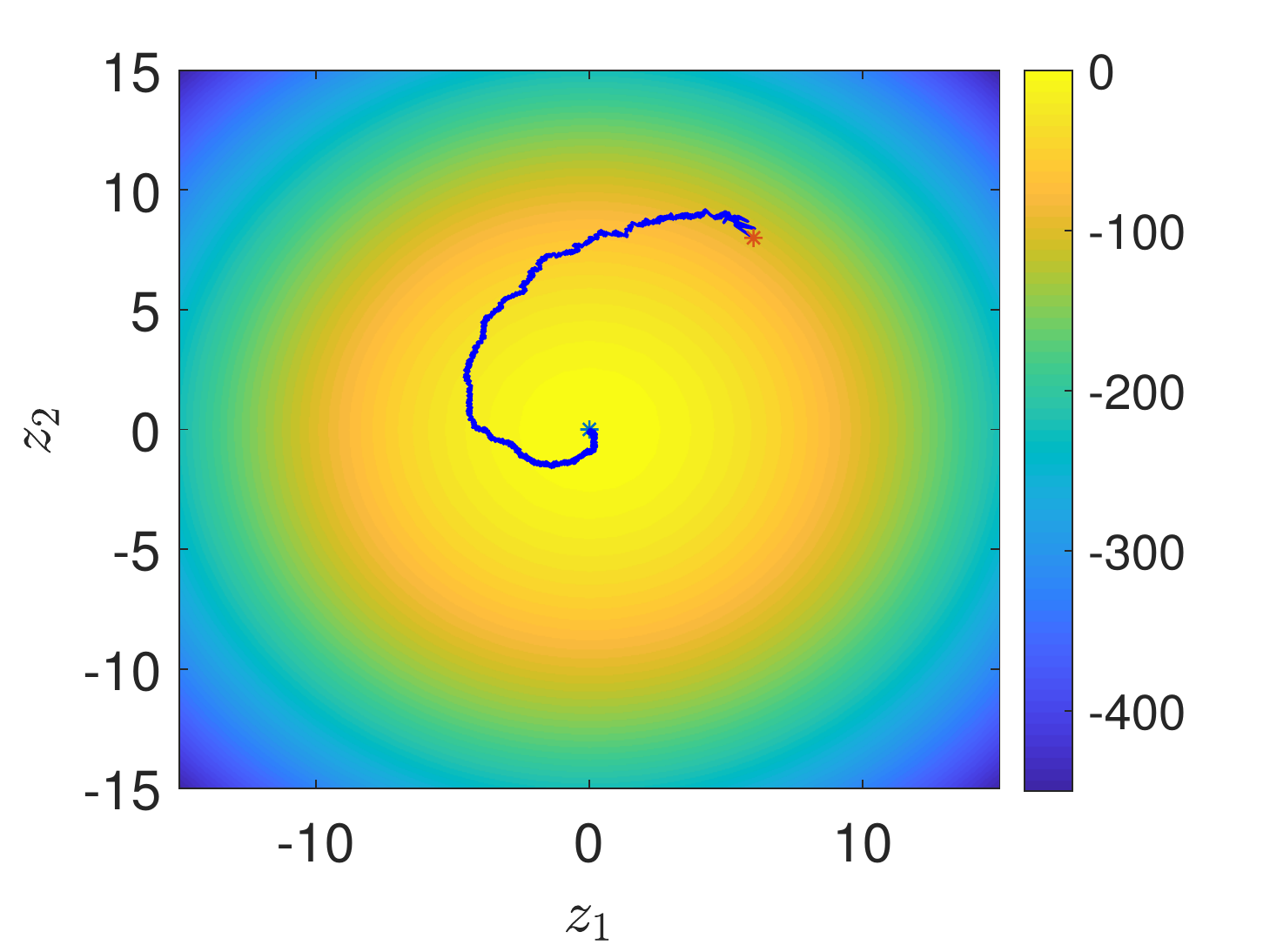}
		\end{minipage}%
	}%
	\subfigure[ $\omega_0=100 rad/s$, $t=3500$ ]{
		\begin{minipage}[t]{0.5\linewidth}
			\centering
		\includegraphics[width=1\textwidth]{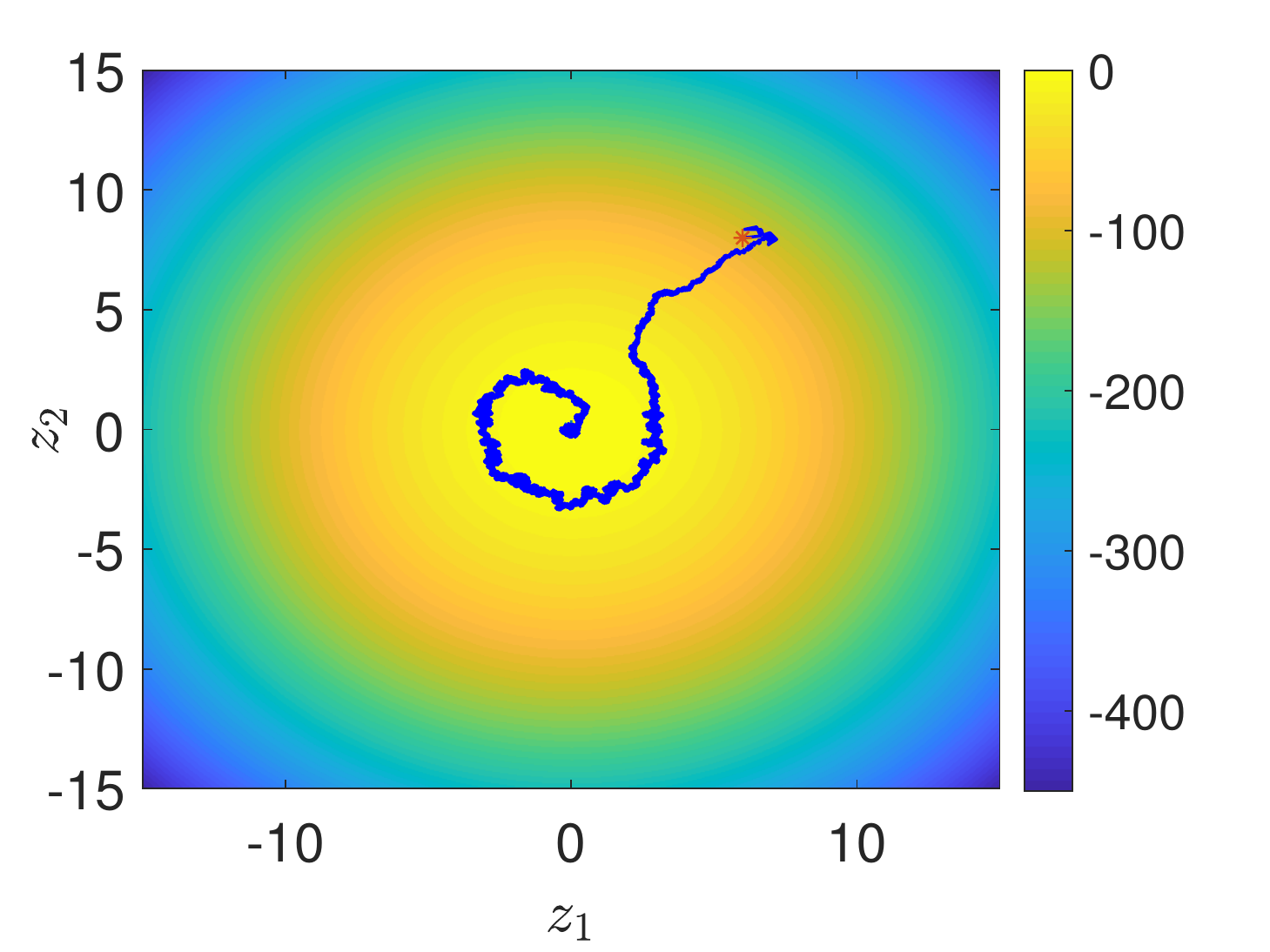}
		\end{minipage}%
	}%

	\centering
	\caption{Simulation results of the closed-loop system using ESC-based projected gradient-ascent control law and various dither signals in the quadratic map ( $J(z_1,z_2)=-z^2_1-z^2_2$, with fixed $k_1=1$, $k_2=20$). The higher $\omega_0$ results in the slower overall robot motion and lower amplitude oscillations.  }
	\label{four w}
\end{figure}

\begin{figure}[htbp]
	\centering
	\subfigure[ Robot State $z_1$]{
		\begin{minipage}[t]{0.5\linewidth}
			\centering			\includegraphics[width=1\textwidth]{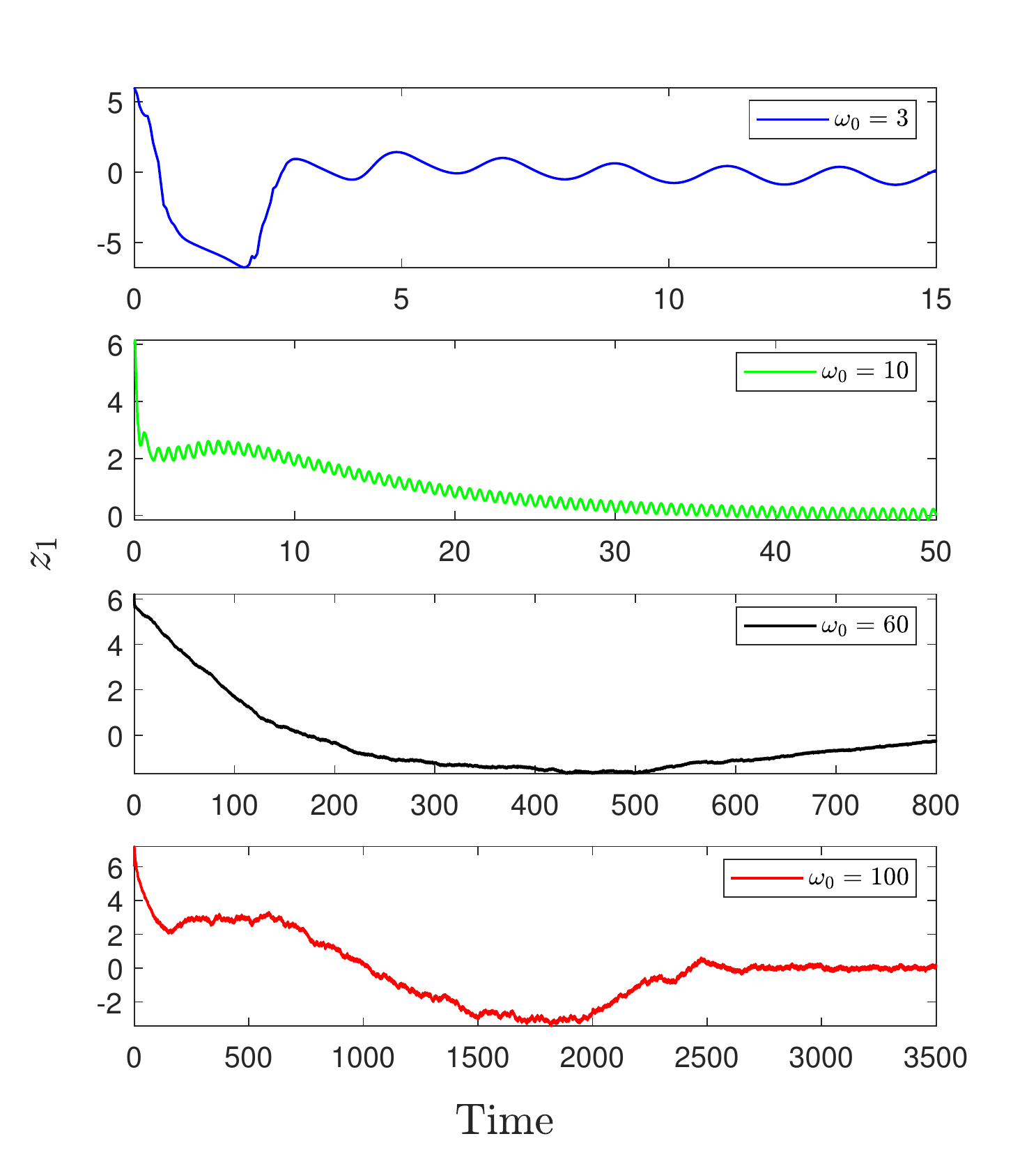}
		\end{minipage}%
	}%
	\subfigure[ Robot State $z_2$ ]{
		\begin{minipage}[t]{0.5\linewidth}
			\centering
		\includegraphics[width=1\textwidth]{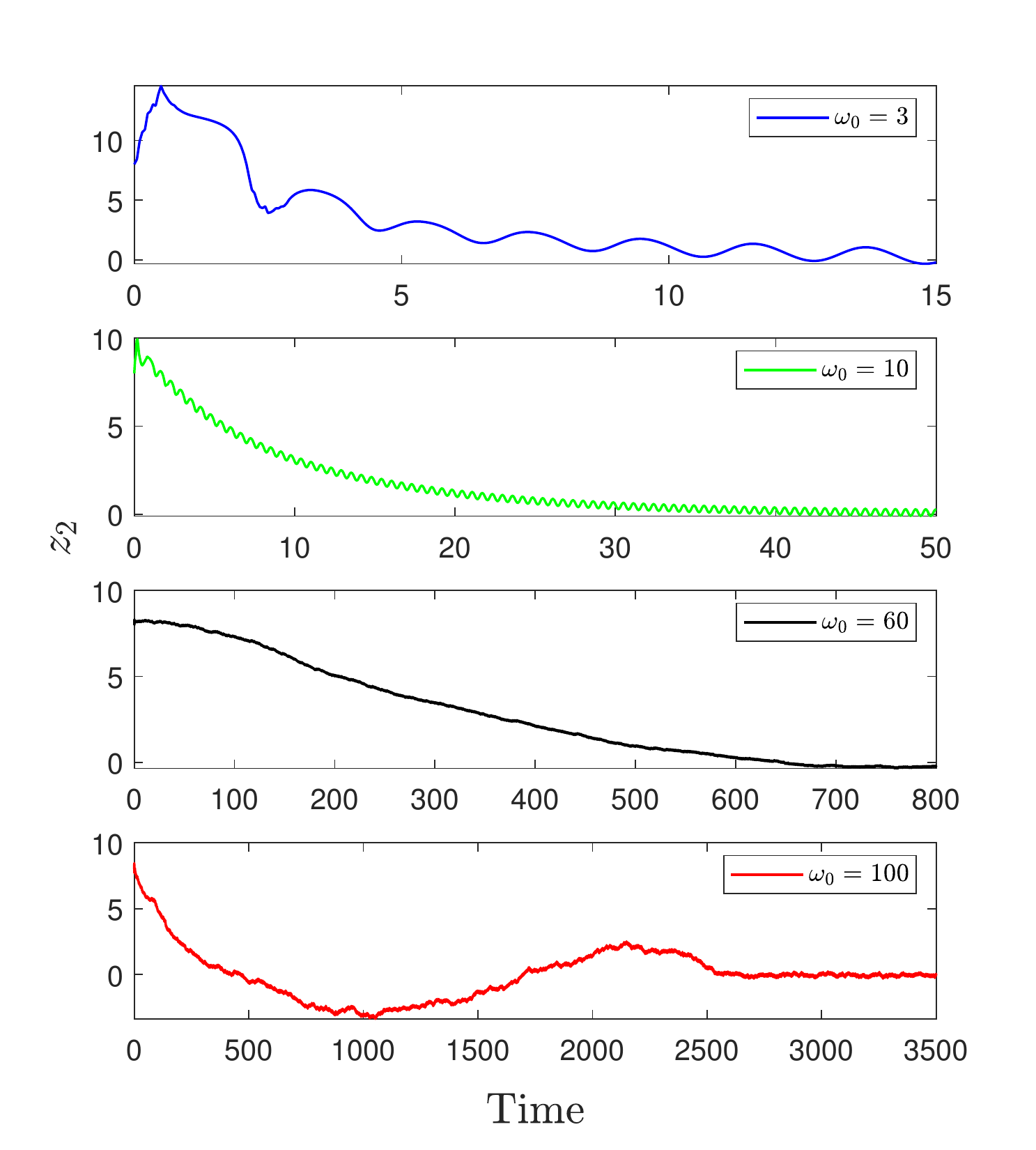}
		\end{minipage}%
	}%
	\centering
	\caption{The plot of robot position as a function of time using ESC-based projected gradient-ascent control law with different values of $\omega_0$ in quadratic map. Different time scale is used for different $\omega_0$ which is due to the different settling time. 
	}
	\label{p_t_ESC_w_z}
\end{figure}

\begin{figure}[htbp]
	\centering
    \includegraphics[width=0.4\textwidth]{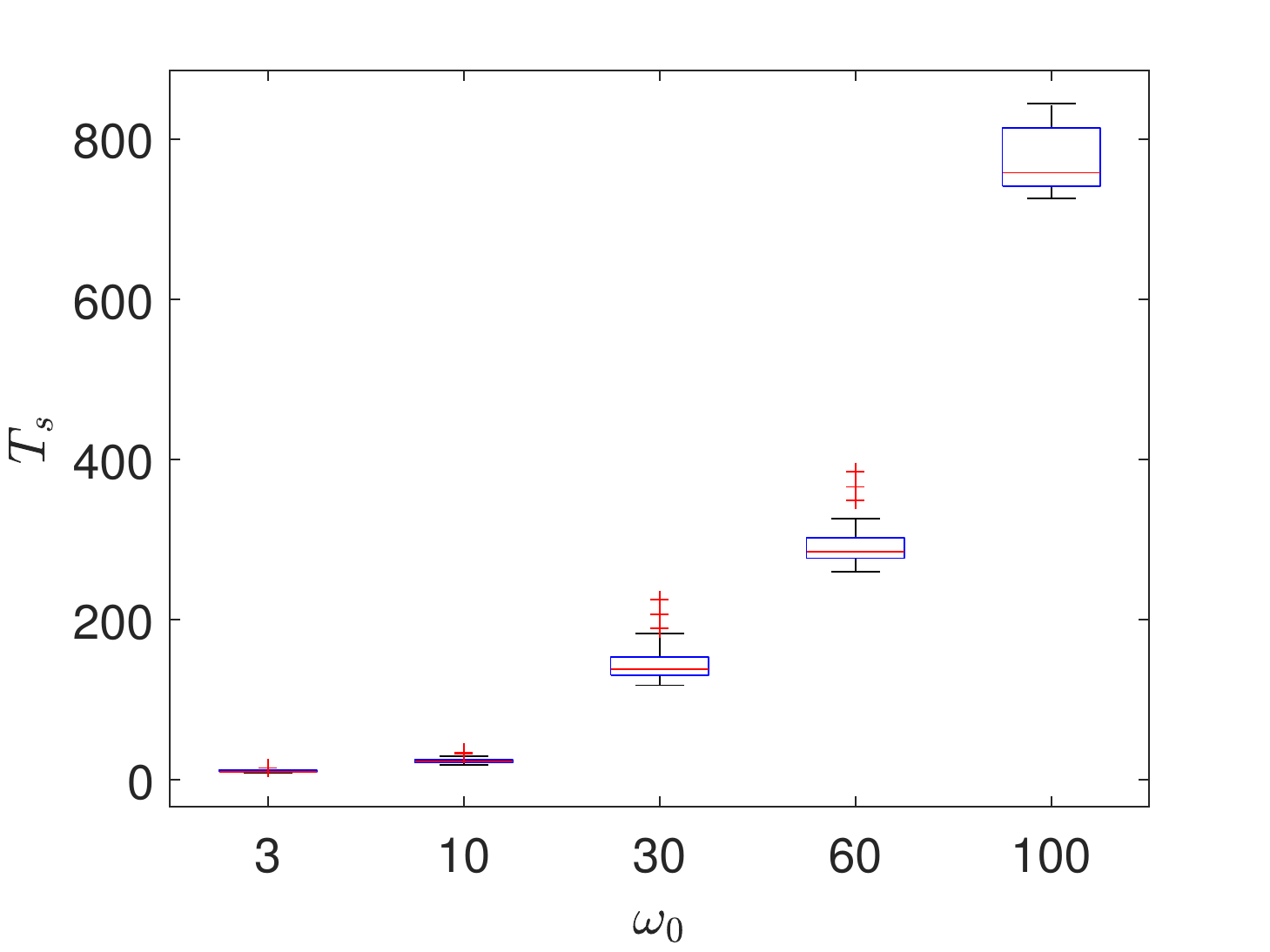}

	\centering
	\caption{The boxplot of settling time $T_s$ of the robot in a quadratic map  using the ESC-based projected gradient-ascent control law with five different $\omega_0$ and with fixed $k_1=1, k_2=20, a=2$. 
	For each $\omega_0$, the robot is initialized at 100 random initial positions.}
	\label{box_w}
\end{figure}


\begin{figure}[htbp]
	\centering
	
	\subfigure[ $\omega_0=2 rad/s$, $t=50$]{
		\begin{minipage}[t]{0.5\linewidth}
			\centering
			\includegraphics[width=1\textwidth]{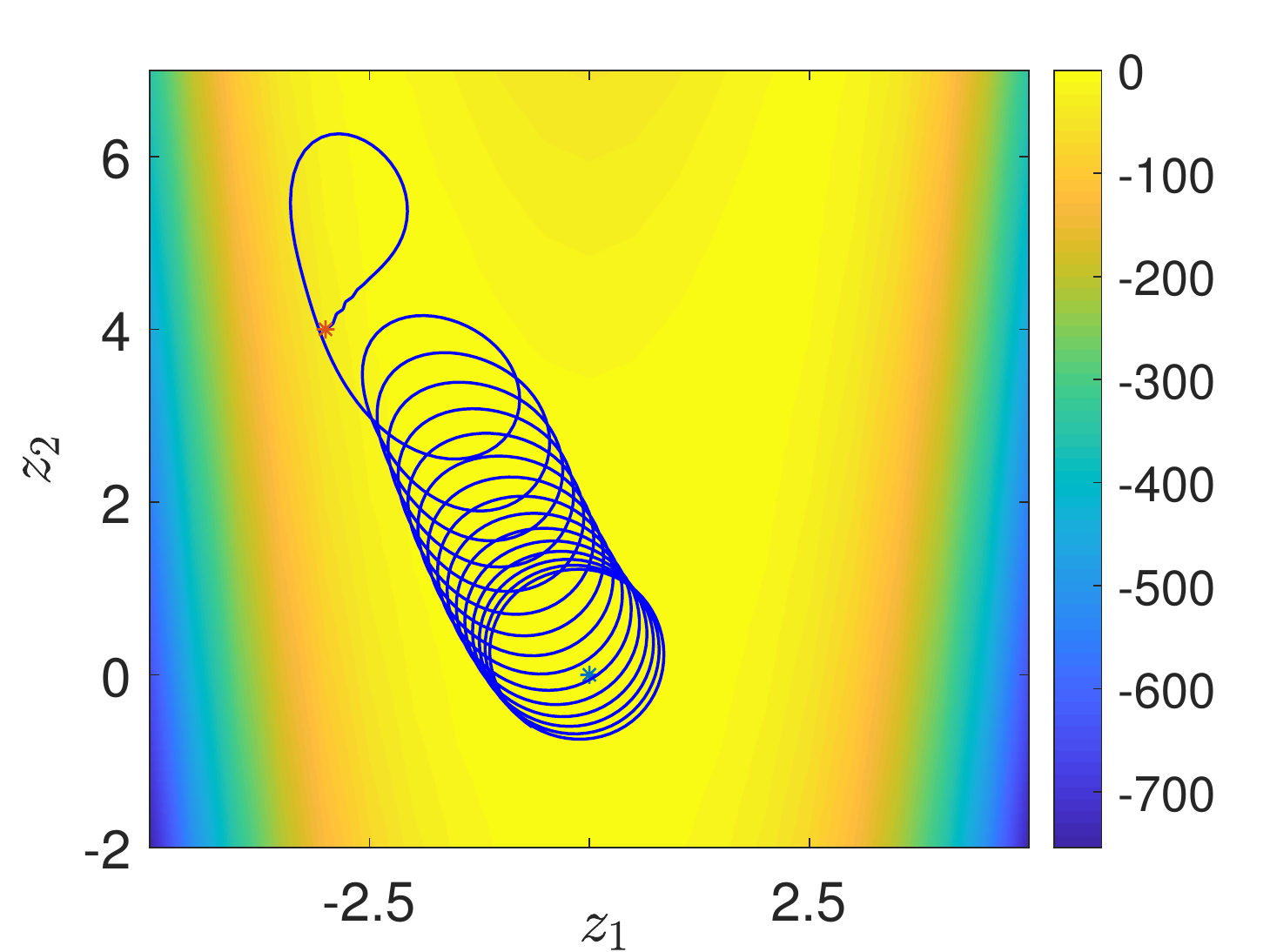}
		\end{minipage}%
	}%
	\subfigure[ $\omega_0=10 rad/s$, $t=350$]{
		\begin{minipage}[t]{0.5\linewidth}
			\centering
			\includegraphics[width=1\textwidth]{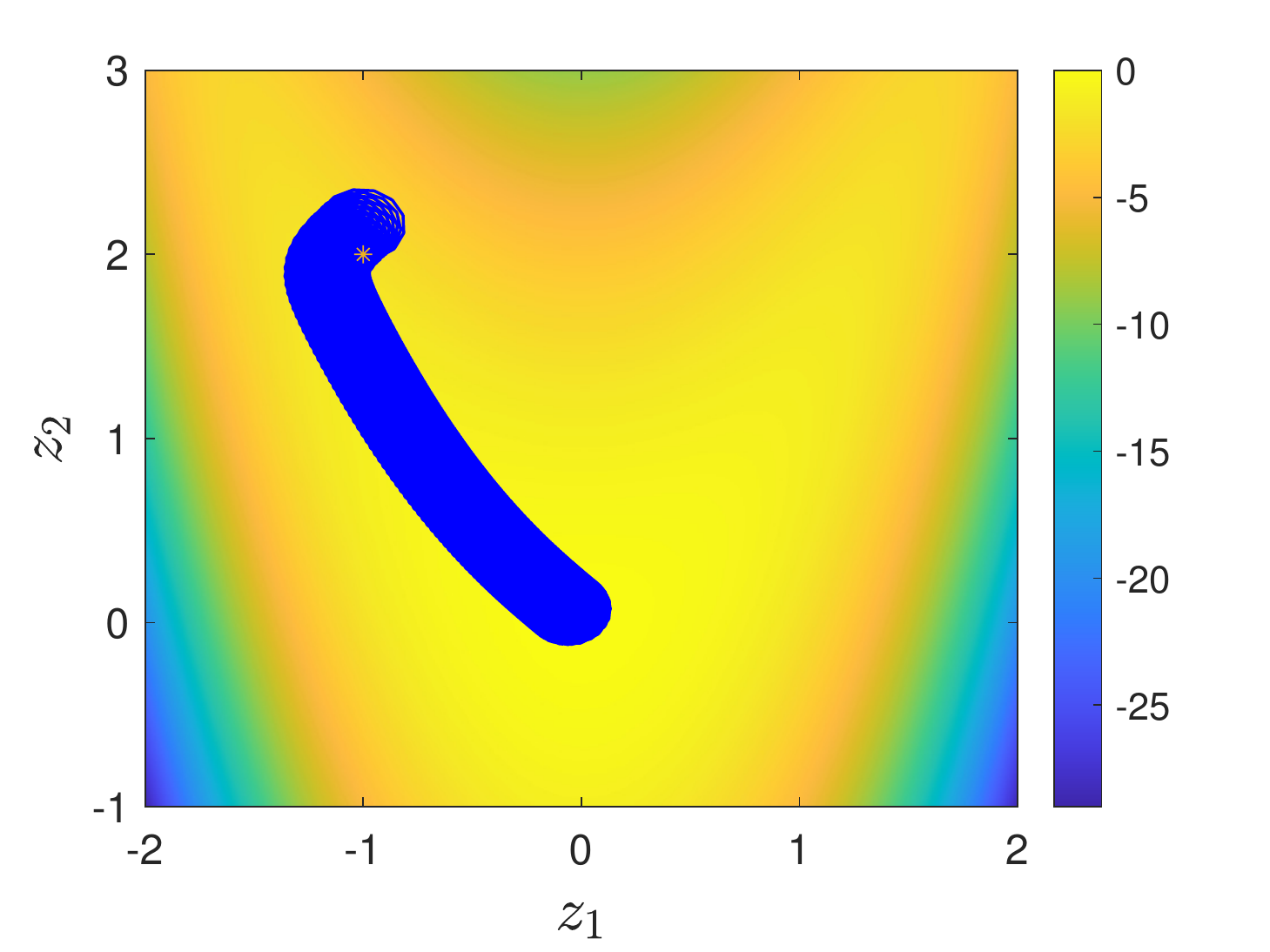}
		\end{minipage}%
	}%
	
		\subfigure[ $\omega_0=50 rad/s$, $t=1500$ ]{
		\begin{minipage}[t]{0.5\linewidth}
			\centering
			\includegraphics[width=1\textwidth]{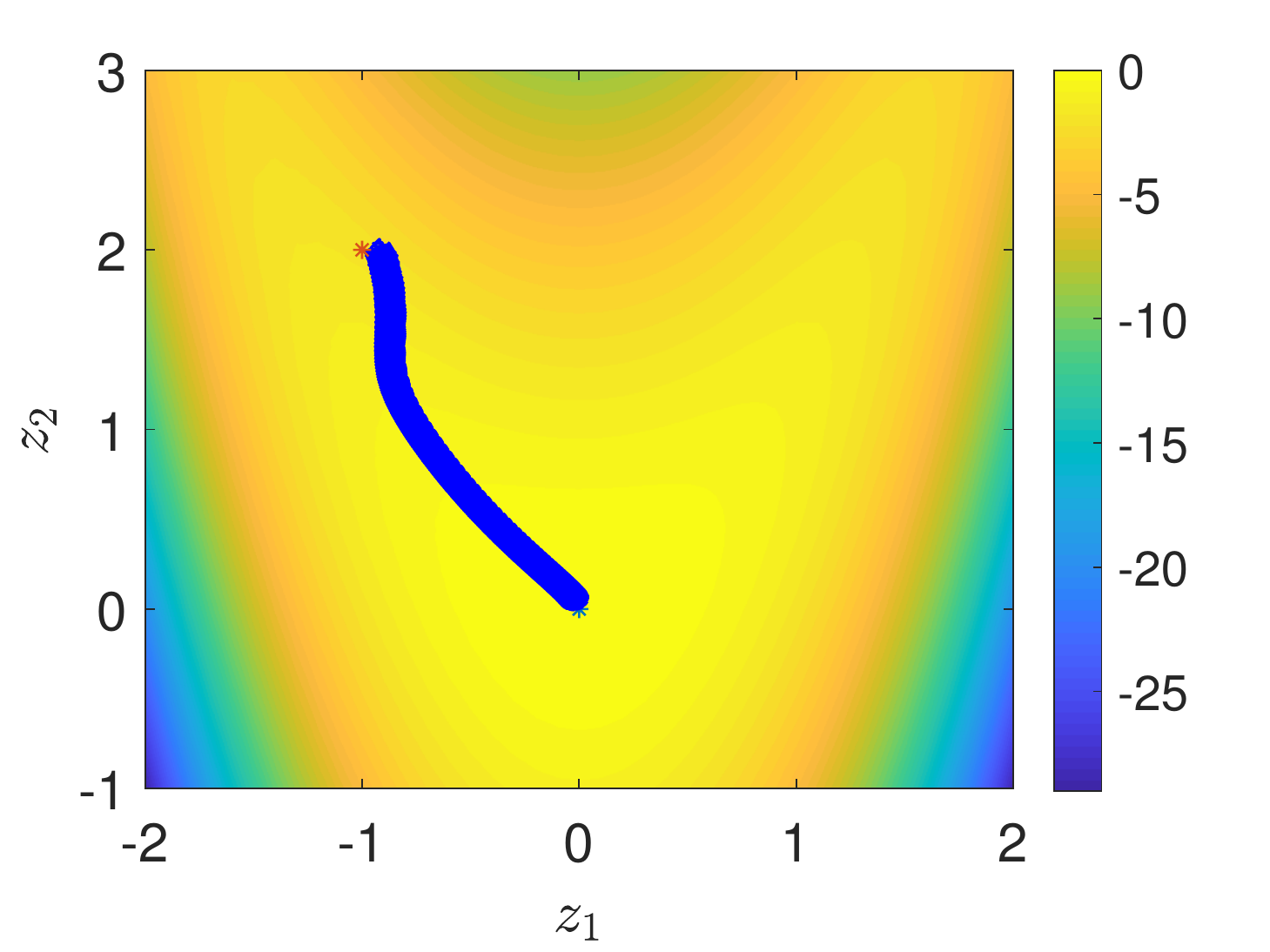}
		\end{minipage}%
	}%
	\subfigure[ $\omega_0=80 rad/s$, $t=3000$ ]{
		\begin{minipage}[t]{0.5\linewidth}
			\centering
			\includegraphics[width=1\textwidth]{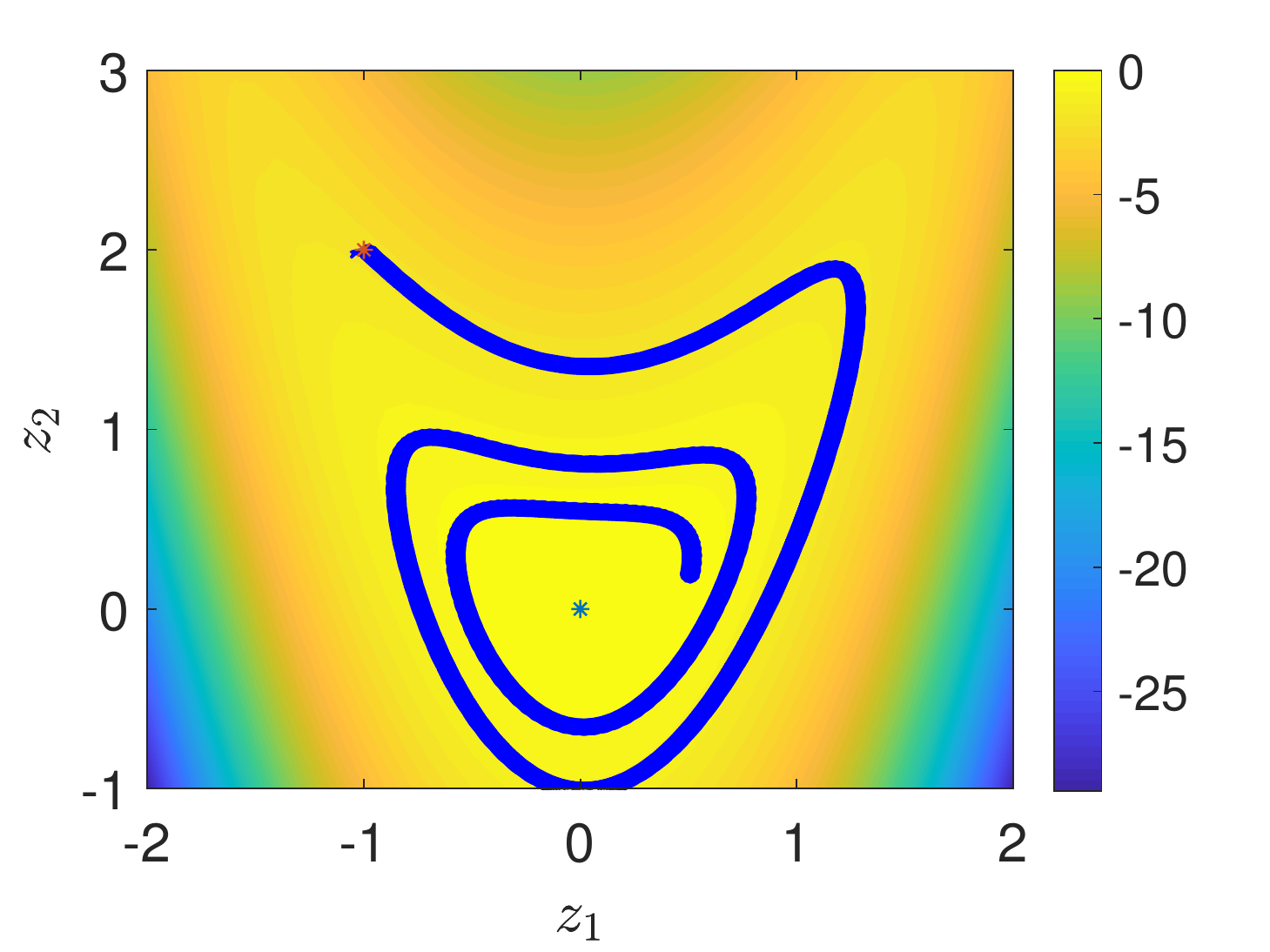}
		\end{minipage}%
	}%
	
	\centering
	\caption{ Simulation results of the closed-loop system based on various dither signals in the non-quadratic map ($k_1=1$, $k_2=20$, $J(z_1,z_2)=-z^2_1-(z_2-z^2_1)^2$), using the ESC-based projected gradient-ascent control law. All the initial state of robot are set to $[-1,2,-60^o]$.  }
	\label{four-w-non}
\end{figure}

\begin{figure}[htbp]
	\centering
	\subfigure[ robot state $z_1$]{
		\begin{minipage}[t]{0.5\linewidth}
			\centering			\includegraphics[width=1\textwidth]{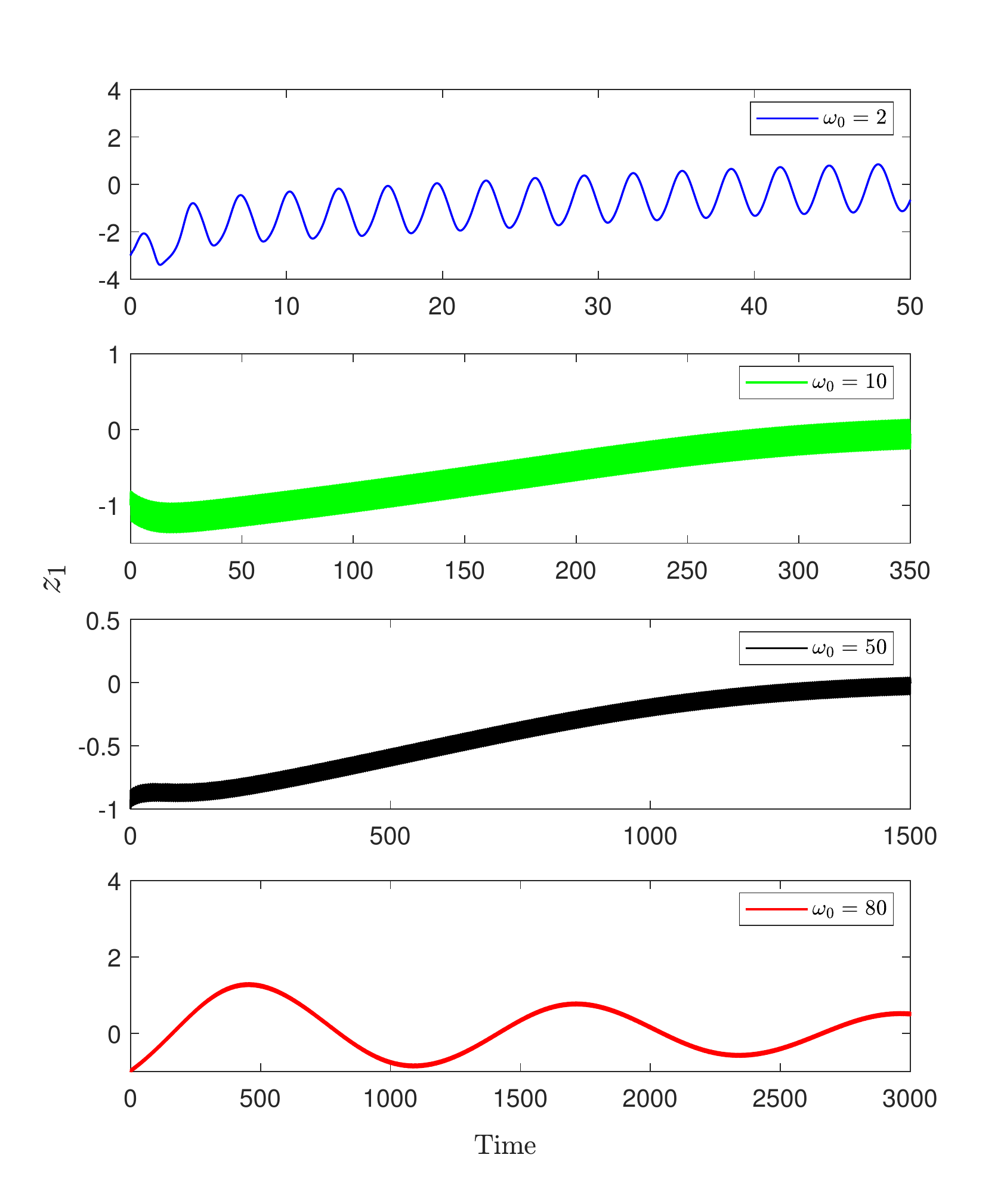}
		\end{minipage}%
	}%
	\subfigure[ robot state $z_2$ ]{
		\begin{minipage}[t]{0.5\linewidth}
			\centering
		\includegraphics[width=1\textwidth]{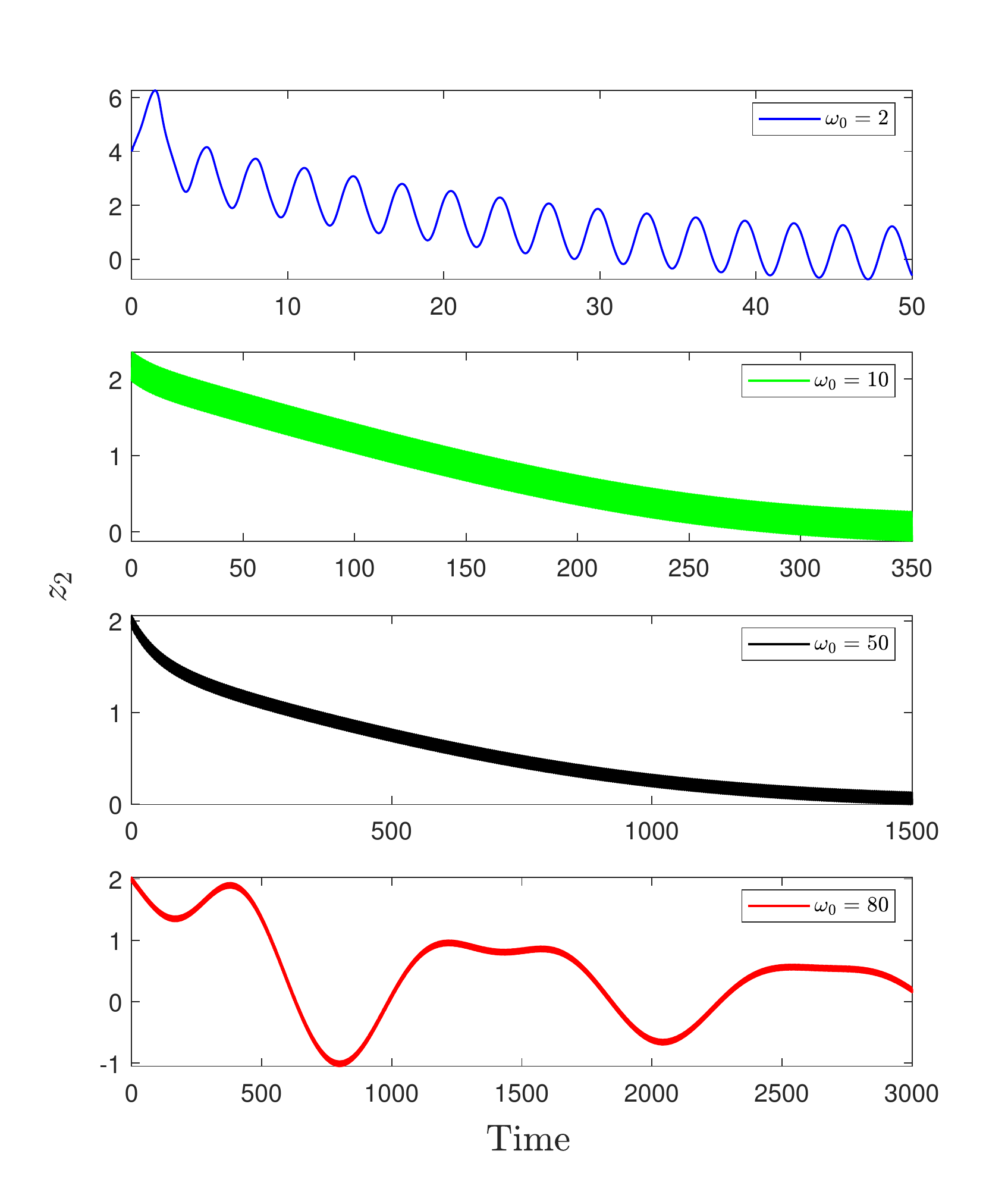}
		\end{minipage}%
	}%
	\centering
	\caption{The plot of robot position as a function of time where ESC-based projected gradient-ascent control law is used with different values of $\omega_0$ in non-quadratic map. 
	The initial state of robot is set to be $[-1,2,-60^o]$. Different time scale is used for different $\omega_0$ which is due to the different settling time.  }
	\label{p_t_ESC_w_z_n}
\end{figure}

\section{Experimental setup and results}
\subsection{Experimental Setup}

For all experiments, we used a Nexus$^\text{\textregistered}$ mobile robot which is equipped with four 100mm Mecanum$^\text{\textregistered}$ wheels that are driven by  Faulhaber$^\text{\textregistered}$ $12V$ motors with optical encoders and controlled by Arduino$^\text{\textregistered}$ 328 Controller and Arduino$^\text{\textregistered}$ IO Expansion.
In this configuration, the mobile robot is omnidirectional.
The parameters of the robot are shown in Table \ref{table}. As shown in Figure \ref{Four Wheels Nexus Robot}, we realized the unicycle robot dynamics by transforming both the longitudinal and angular velocity of unicycle robot into the individual velocity of each wheel
Correspondingly, for realizing the unicycle dynamics as in \eqref{eq:unicycle_model}, we set
the longitudinal velocity $u_{z_{1,r}}=u$, lateral velocity $u_{z_{2,r}}=0$ and  angular velocity $\omega_z=\omega$, which can be computed in real-time based on the encoder for each wheel.

The four sensors that were designed and built as presented in Section \ref{sec:3D-sensor} were all mounted on the four sides of the mobile robot. A quad-channel analog differential input shield for Arduino (see also Figure \ref{Arduino Board}) was used as a front-end for the 10-bit Arduino ADC inputs. By connecting the Wheatstone bridge with the analog differential inputs, the difference resistance of airflow sensors can be calculated to obtain the airflow velocity as described in subsection \ref{sec:3D-sensor}B.
\begin{figure}[htbp]
	\centering
	
	\subfigure{
		\begin{minipage}[t]{0.5\linewidth}
			\centering
			\includegraphics[width=120pt,height=100pt]{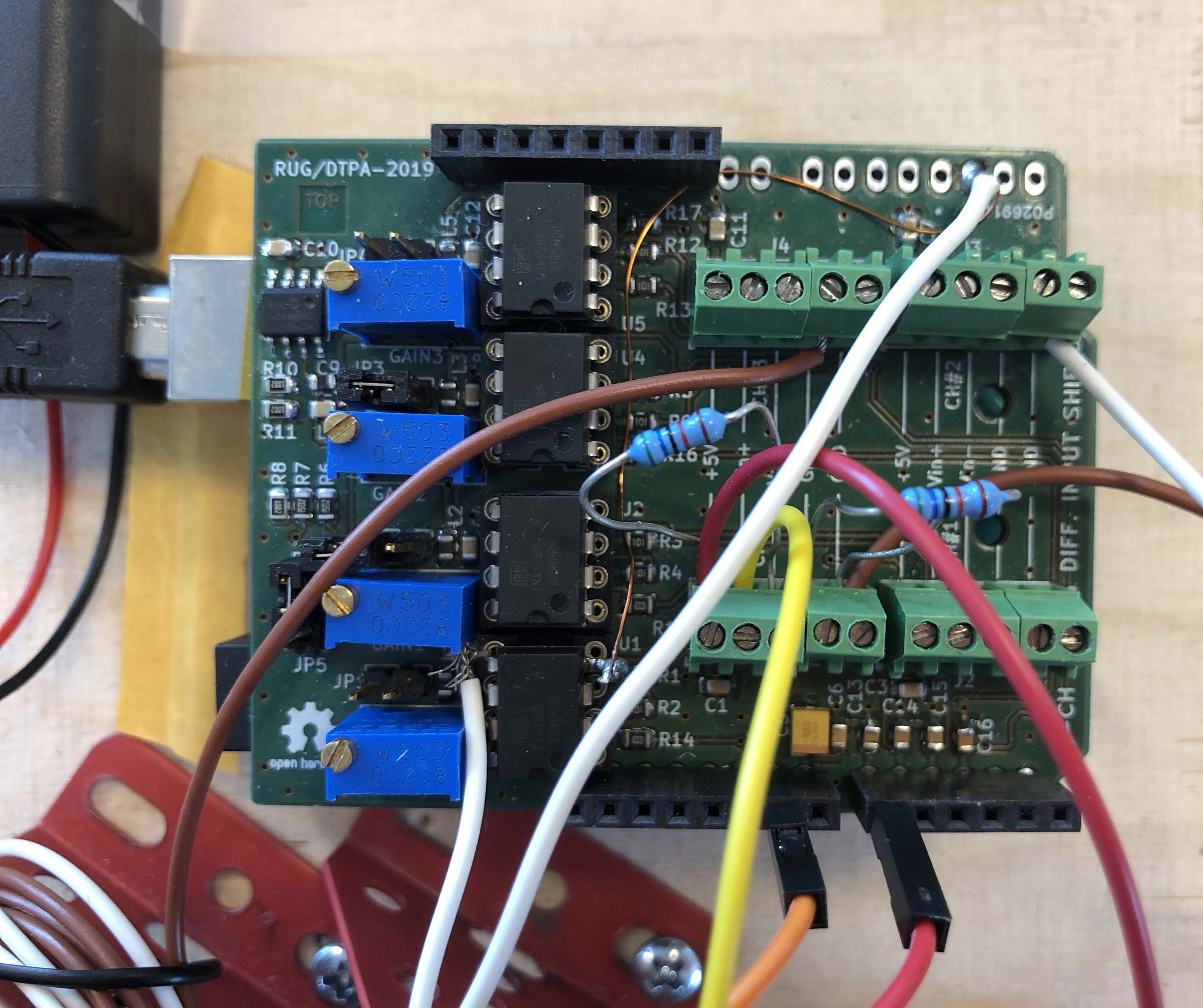}
		\end{minipage}%
	}%
	\subfigure{
		\begin{minipage}[t]{0.5\linewidth}
			\centering
			\includegraphics[width=120pt,height=100pt]{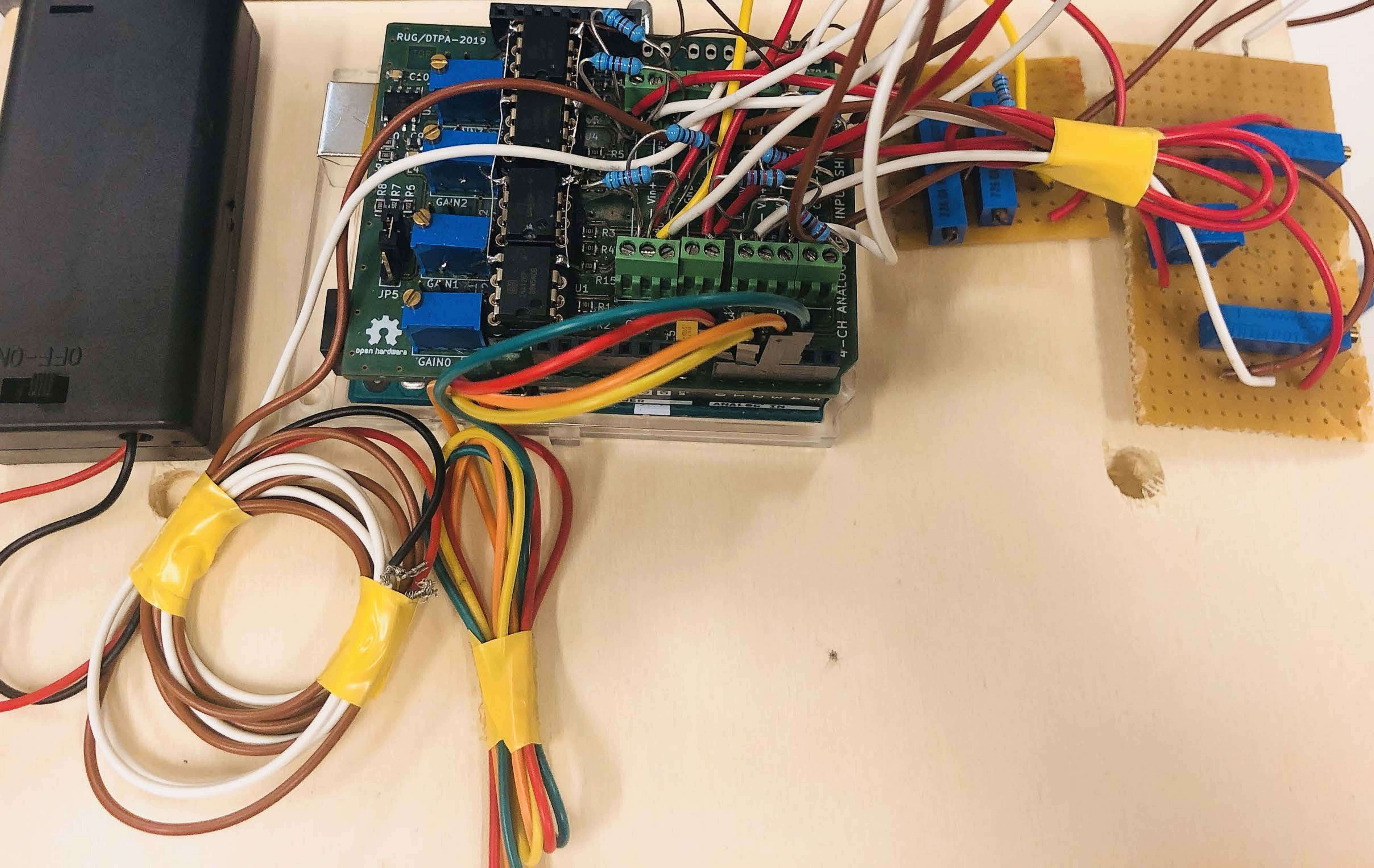}
		\end{minipage}%
	}%
	\centering
	\caption{The quad analog differential input shield for Arduino board that is used to process the flow sensor signals.}
	\label{Arduino Board}
\end{figure}

\begin{table}[!t]
\caption{Parameters of the Mobile Robot}
 \label{table}
 \begin{tabular}{lllllll}
\toprule
 Parameters & Symbol &Value (m)\\
\midrule
Wheel Radius &$r$ & 0.05 \\
 \hline
Distance between Left and Right wheel &$d_{W}$ &0.3   \\
 \hline
Distance between Front and Rear wheel &$d_{L}$ &0.3    \\
  \bottomrule
\end{tabular}
\end{table}

\begin{figure}[htbp]
	\centering
	
	\subfigure[Front view]{
		\begin{minipage}[t]{0.5\linewidth}
			\centering
			\includegraphics[width=1\textwidth]{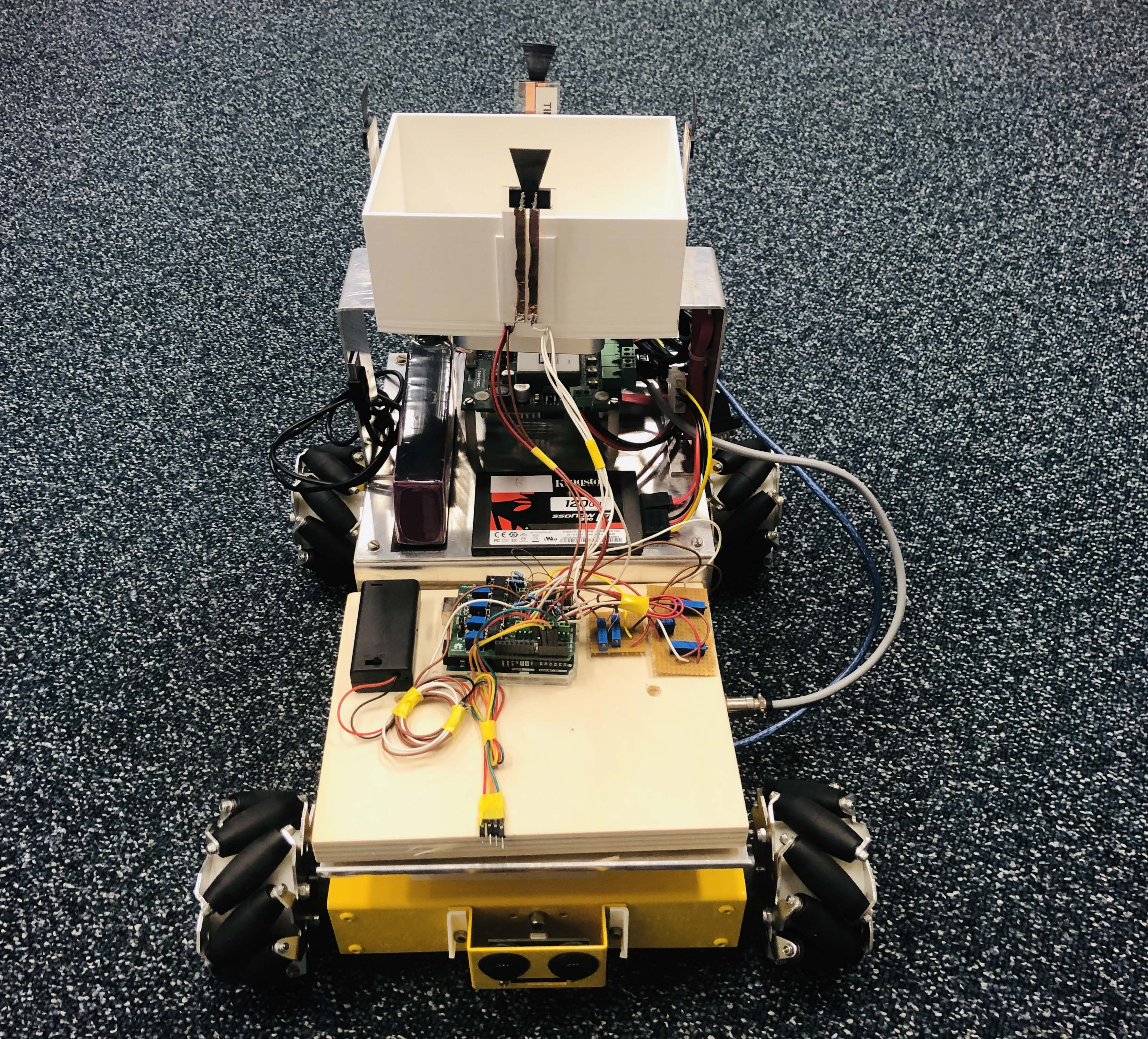}
				\end{minipage}%
	}%
	\subfigure[Side view]{
		\begin{minipage}[t]{0.5\linewidth}
			\centering
			\includegraphics[width=1\textwidth]{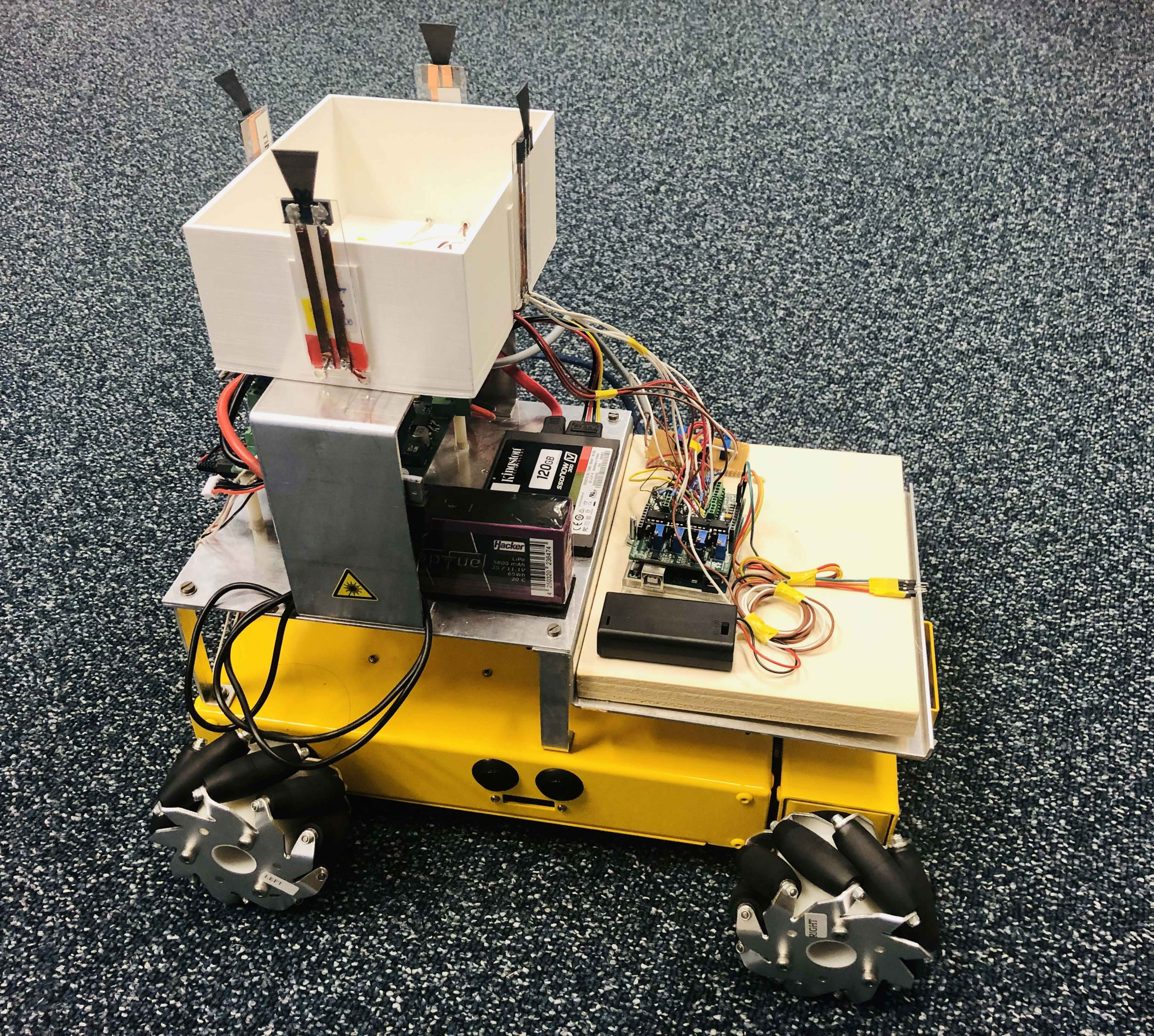}
				\end{minipage}%
	}%

	\subfigure[3D-printed flexible airflow sensors on the four sides]{
		\begin{minipage}[t]{0.5\linewidth}
			\centering
			\includegraphics[width=1\textwidth]{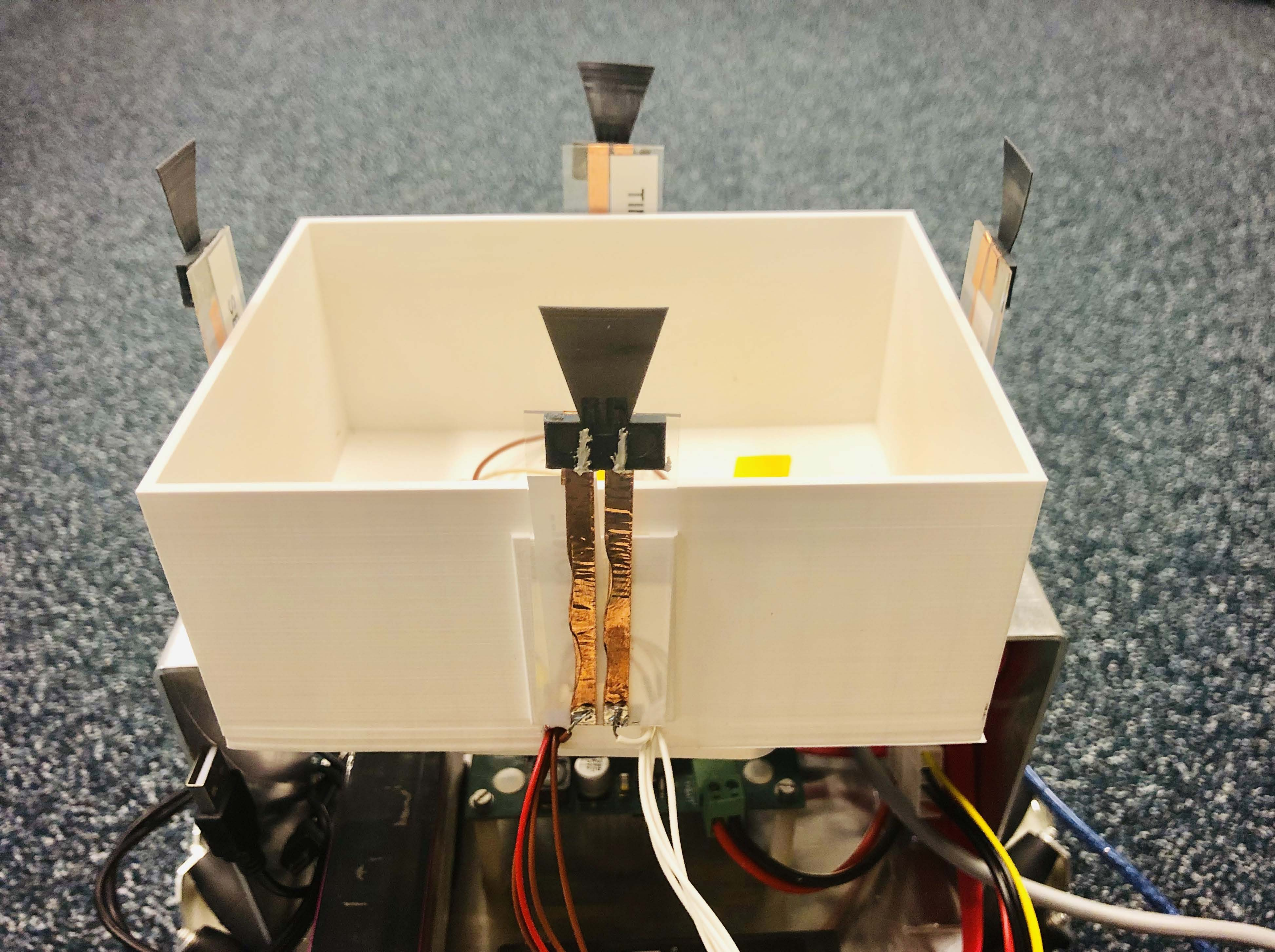}
				\end{minipage}%
	}%
	\subfigure[Robot kinematic model]{
		\begin{minipage}[t]{0.5\linewidth}
			\centering
			\includegraphics[width=1\textwidth]{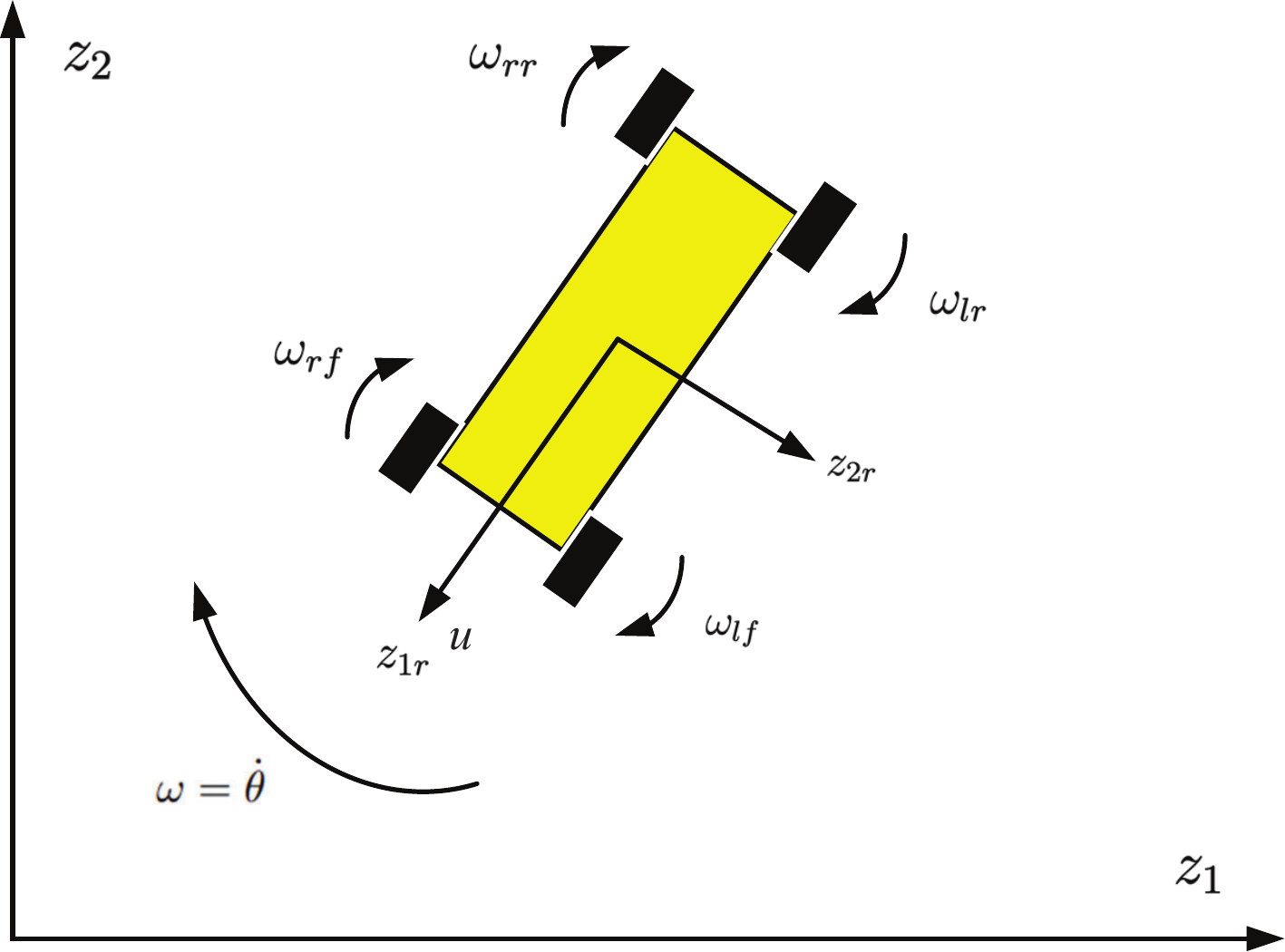}
					\end{minipage}%
	}%
	\centering
	\caption{The experimental unicycle robot setup based on Nexus$^\text{\textregistered}$ mobile robot equipped with 3D-printed flexible flow sensors, Arduino$^\text{\textregistered}$-based microcontroller and 4WD Mecanum$^\text{\textregistered}$ wheels.}
	\label{Four Wheels Nexus Robot}
\end{figure}

The closed-loop system was implemented using Robot Operating System (ROS) middleware run in Linux (Ubuntu 16.04).

\begin{figure}[htbp]
	\centering
	\subfigure{
		\begin{minipage}[t]{0.5\linewidth}
			\centering
			\includegraphics[width=1\textwidth]{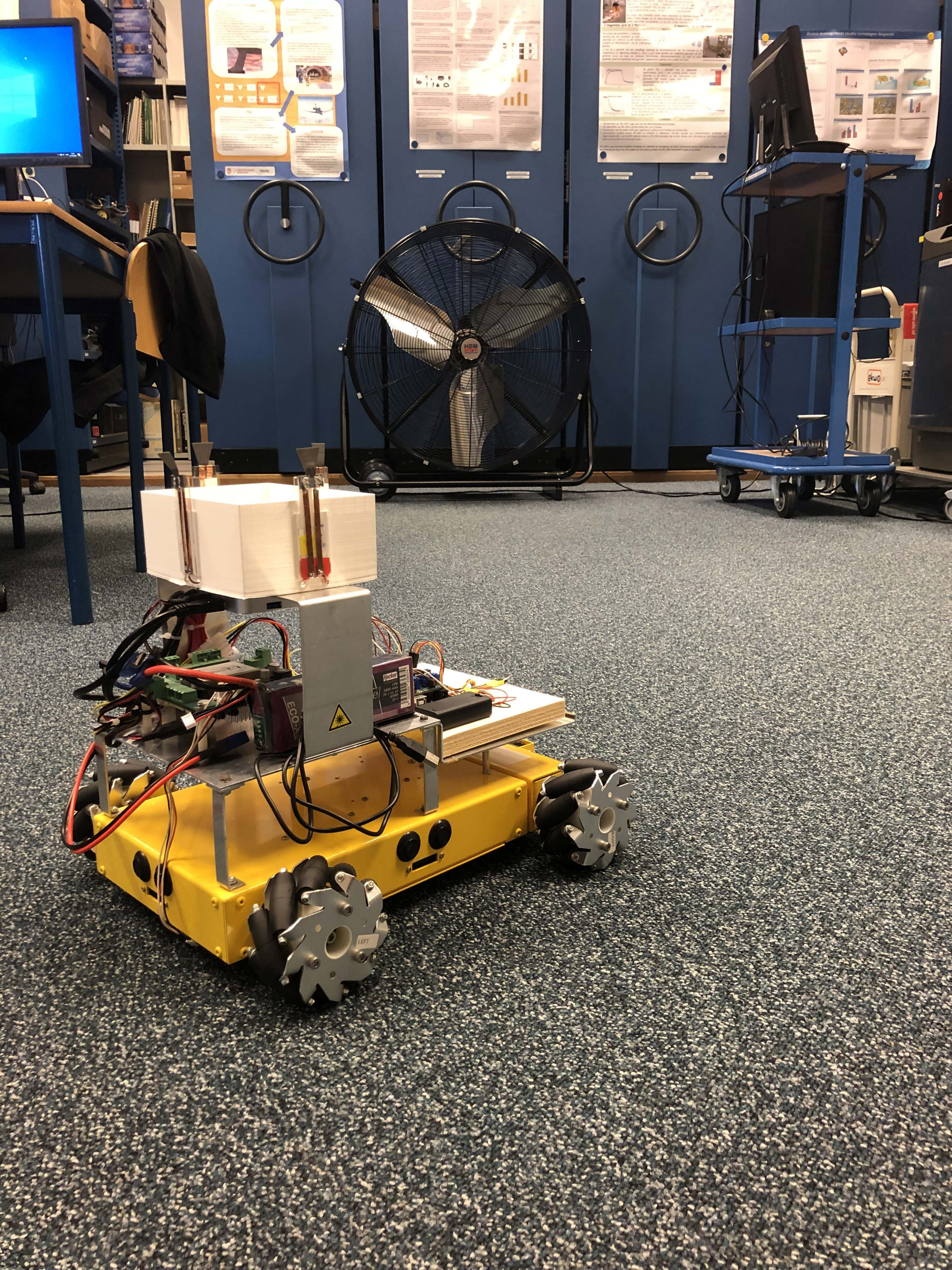}
					\end{minipage}%
	}%
	\subfigure{
		\begin{minipage}[t]{0.5\linewidth}
			\centering
			\includegraphics[width=1\textwidth]{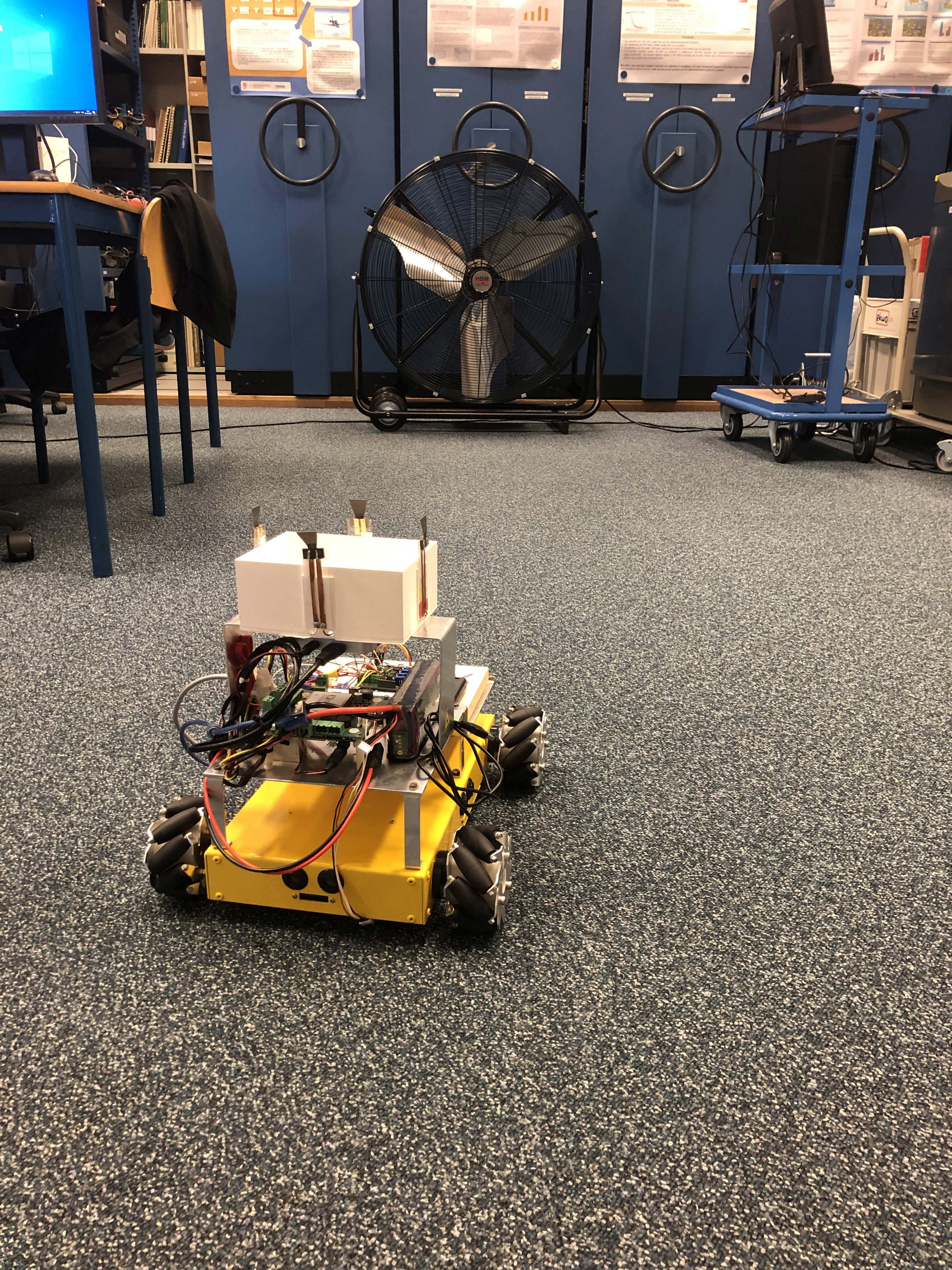}
				\end{minipage}%
	}%
	
\centering
    \caption{Photos of the lab environment where the experiments were conducted. Both the mobile robot and the industrial fan are visible in the pictures.}
	\label{fig:lab}
\end{figure}

Prior to conducting live experiments, we firstly gathered the information of the potential field in our lab environment when a commercial industrial fan (HBM 36INCH drum fan) was used to generate the wind (see also Figure \ref{fig:lab} on the lab environment where the robot and the fan are visible). We used a commercial anemometer (Airflow Meter PCE-423) to measure the wind speed field at a fixed grid of points and we fitted the data with polynomial functions by least square fitting method. The fitted field can be used to validate numerically the proposed control laws before they were evaluated experimentally. The following fourth-order polynomial function was obtained
\begin{equation}\label{eq:fitted_poly}
    v(z_1,z_2)=68.54R^4-102.80R^3+36.13R^2+6.41R-0.34,
\end{equation}
where $v$ is the wind strength, $R=\frac{r_f}{d}$ with $d$ being the distance between the fan and the position $(z_1,z_2)$, and $r_f=0.45m$ is the radius of the fan blade. In our setup, we considered only an active area where $d \geq 0.5m$. The polynomial function above is based on the modeling of flow field of a ceiling fan as reported in \cite{Azim}. Figure \ref{Airflow Model} shows the fitted polynomial vs measured data points and Figure \ref{Airflow Map of the Environment} shows the resulting airflow field speed map and the corresponding gradient distribution.

\begin{figure}[htbp]
	\centering
	\subfigure{
		\begin{minipage}[t]{0.5\linewidth}
			\centering
			\includegraphics[width=1\textwidth]{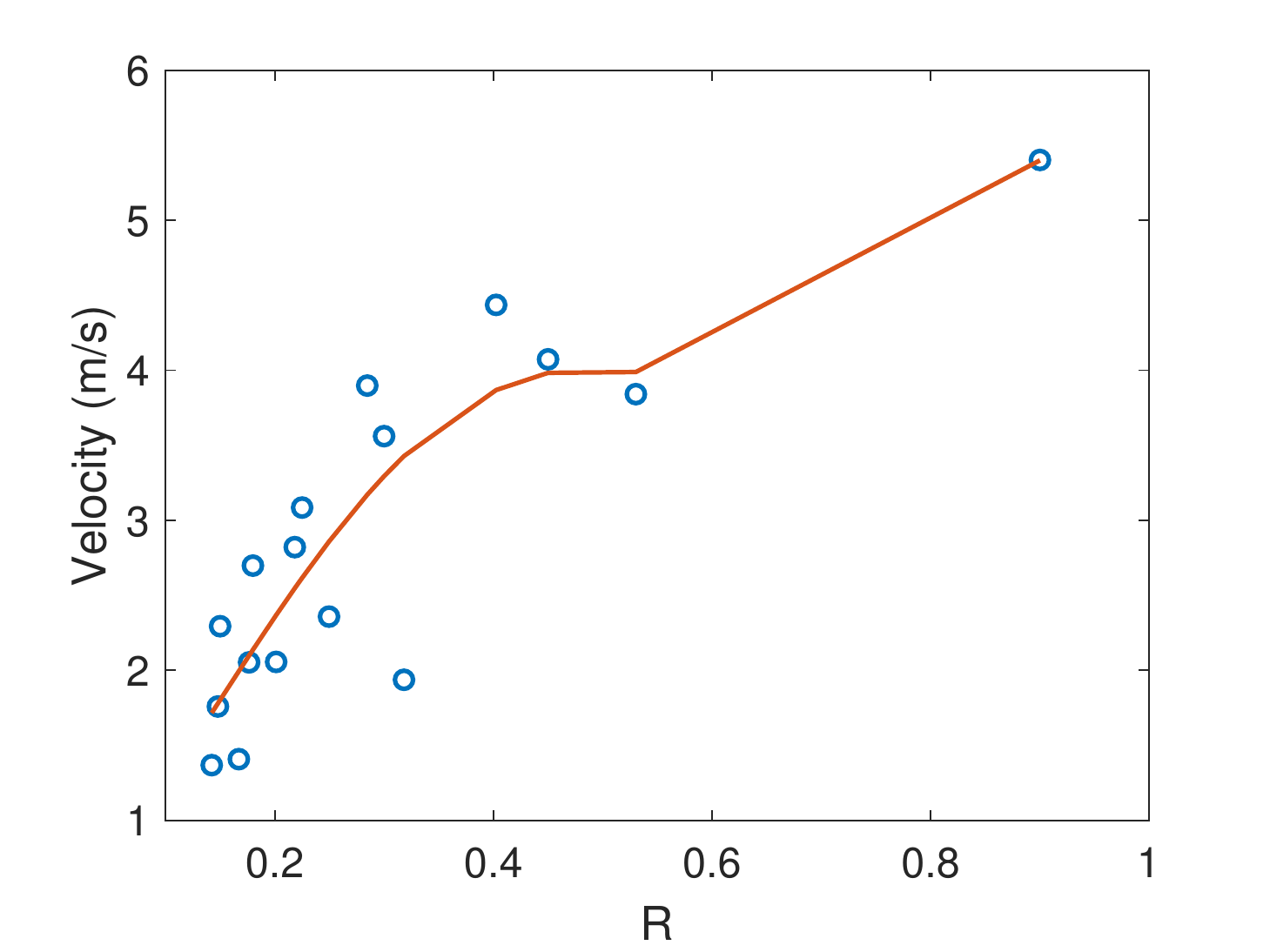}
		\end{minipage}%
	}%
	\subfigure{
		\begin{minipage}[t]{0.5\linewidth}
			\centering
			\includegraphics[width=1\textwidth]{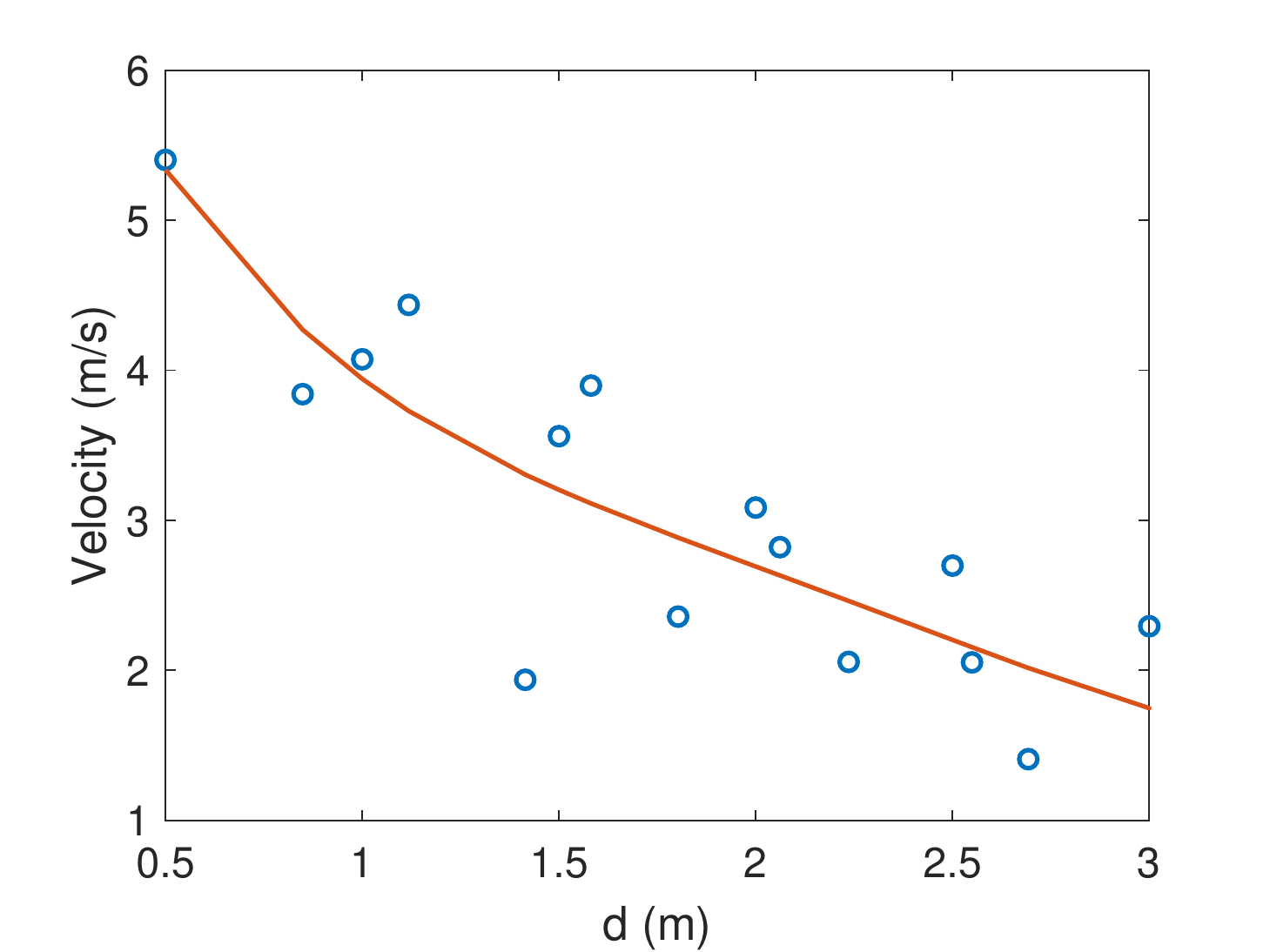}
		\end{minipage}%
	}%

	\caption{The plot of measured wind speed at various distances $d$ (in blue circle) and the fitted fourth-order polynomial as in \eqref{eq:fitted_poly} as a function of $R=\frac{r_f}{d}$ with $r_f$ be the fan blade radius (in solid red line).}
	\label{Airflow Model}
\end{figure}

\begin{figure}[htbp]
	\centering
	
\begin{tabular}{cc}
	\subfigure{
		\begin{minipage}[t]{0.5\linewidth}
			\centering
			\includegraphics[width=1\textwidth]{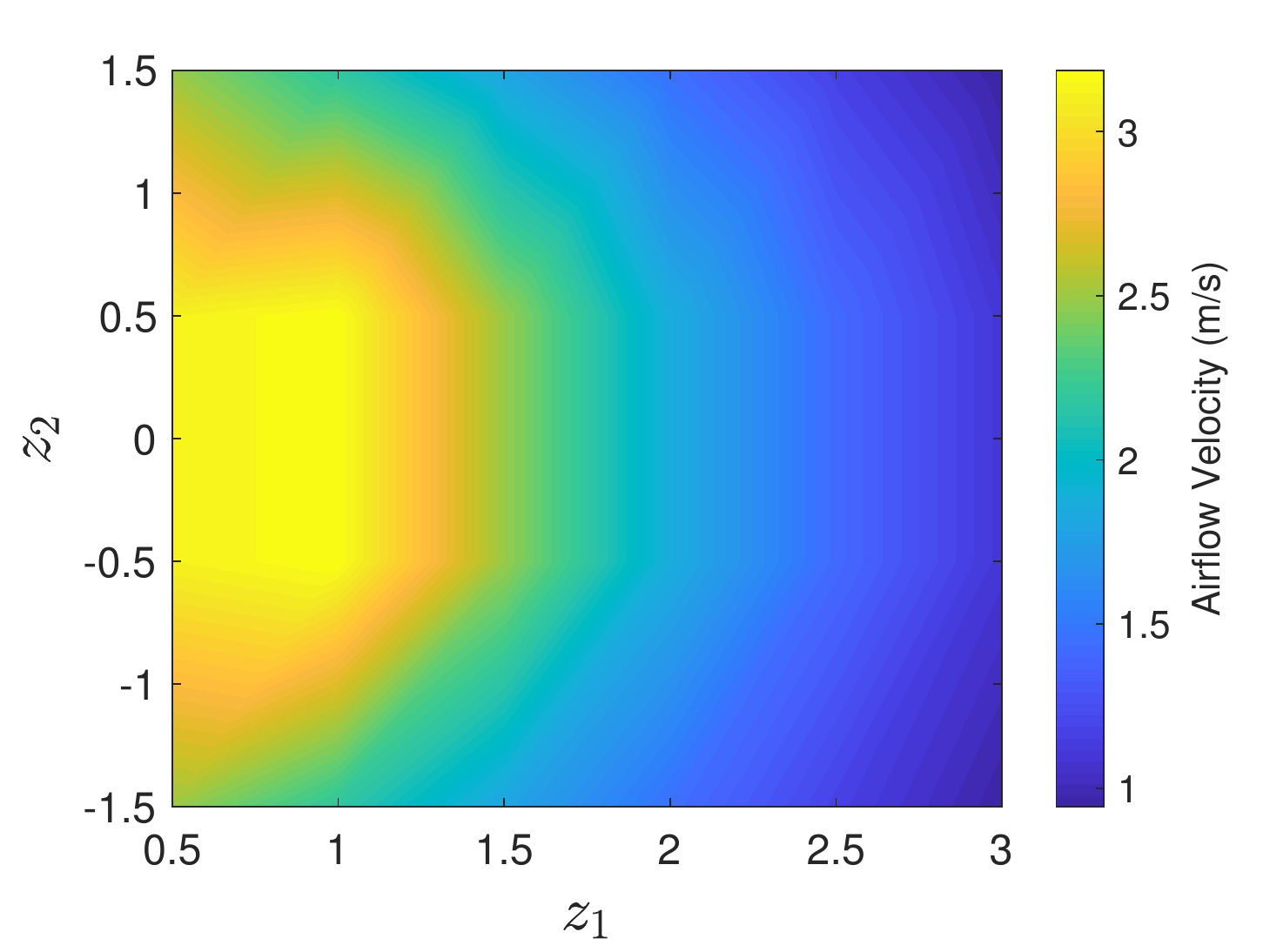}
		\end{minipage}}
    \subfigure{
        \begin{minipage}[t]{0.5\linewidth}
			\centering
			\includegraphics[width=1\textwidth]{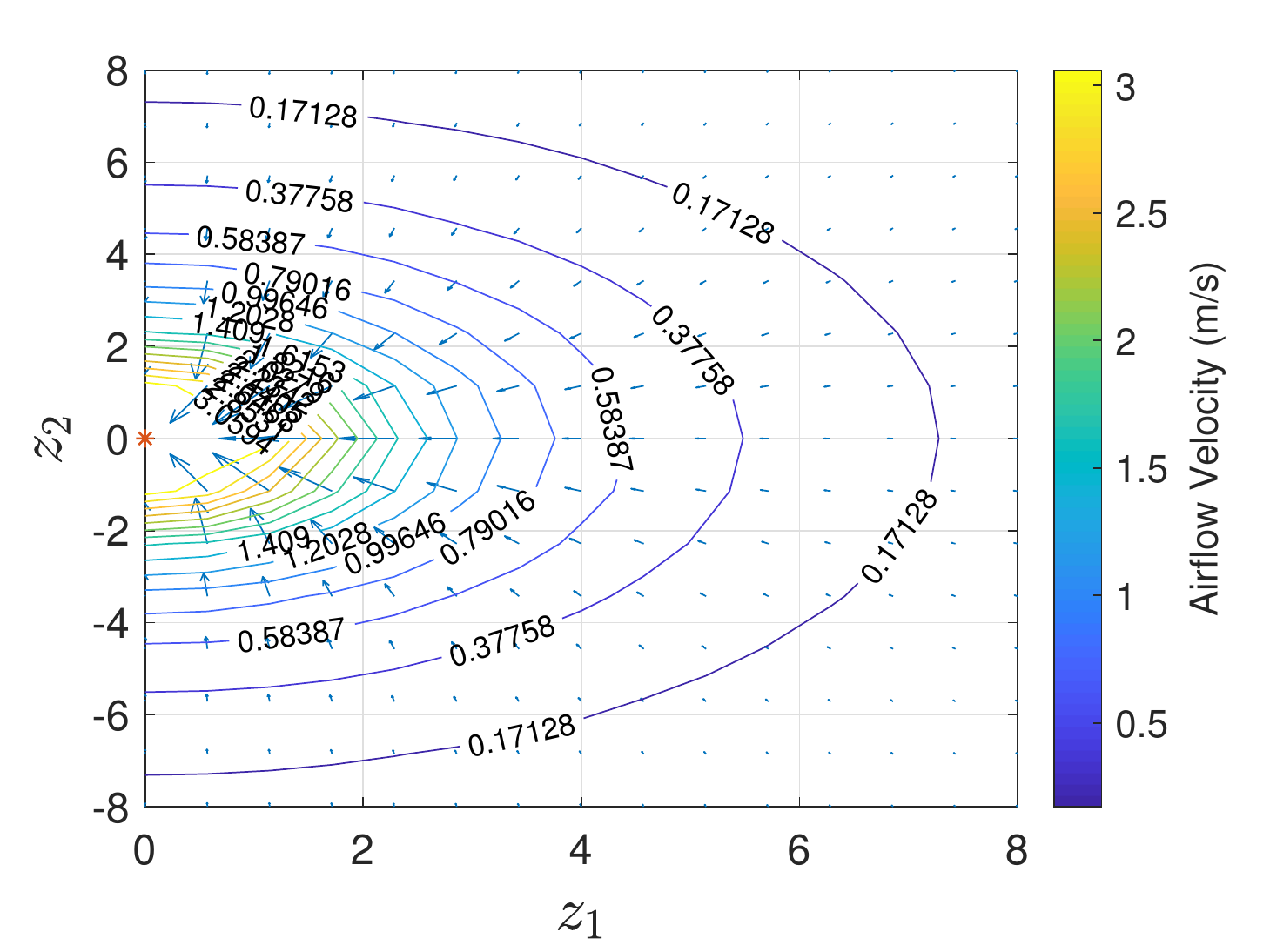}
		\end{minipage}}
	\end{tabular}
\centering
    \caption{The plot of air flow field based on the fitted fourth-order polynomial as in \eqref{eq:fitted_poly}. The location of the fan is at $(0,0)$.}
	\label{Airflow Map of the Environment}
\end{figure}

For processing the ADC input signal from each of the four flexible sensors, we used linear fitting to the calibration curves that have been obtained before in Section \ref{sec:3D-sensor} in order to get the estimated wind flow strength from the measured resistance change in the Wheatstone bridge. The linear fit simplified the signal processing and it would not affect the performance of source-seeking algorithm as it did not change the extremum location and did not alter the concave property of the underlying potential function. As an example, for one of the sensors, the nominal resistance was $47.6k\Omega$ and the linear estimation of the air flow strength $S_1$ used in \eqref{eq:sensor1}-\eqref{eq:sensor2} is given by
\begin{equation}
    S_1= 23.80 \frac{\Delta R_{1}}{R_{1,0}},
\end{equation}
with the nominal resistance $R_{1,0}=47600 \Omega$. The resistance changes $\Delta R_i$, $i=1,2,3,4$, were obtained by converting the measured voltage across the piezoresistive sensors using the 10-bit ADC of the Arduino and the resistance changes were subsequently computed by solving the relation of resistances in each of the four different Wheatstone bridges.

\subsection{Experimental Results}

As described before in Subsection \ref{sec:sensor_reading}, at any given time, only two of the sensors reading will be used to estimate the gradient. For each experiment, the outputs of four channels 10-bit ADC were always calibrated to the mid-voltage of the ADC so that the resistance changes due to the tension or to the compression of the cantilever can properly be measured.
For both proposed control laws, we set the gains of the projected gradient-ascent control part the same, and they were given by
$k_1=1$ and $k_2=20$. The accompanying video of the experimental results is publicly accessible at  \href{https://youtu.be/y3OoRu5GX3M}{https://youtu.be/y3OoRu5GX3M}.

\subsubsection{Projected gradient-ascent control law}
Figure \ref{fig:Experiment_1} shows the trajectories of the unicycle robot from three different initial conditions. The trajectories were overlaid on top of the air flow field map as identified before in Figure \ref{Airflow Map of the Environment} where the fan was located at the origin. As the mobile robot uses only local sensor systems and was not equipped with a local/global positioning system, the overlaying was done by using the recorded on-board odometry measurements.
As shown in the figure, the resulting trajectories were affected significantly by the noisy measurement in the sensor systems that was attributed mainly to the turbulent flow on the cantilever as well as the body vibrations coming from the motion. Despite this, the proposed control law is able to seek the source robustly against such disturbances.

\begin{figure}[htbp]
	\centering
	
	\subfigure[Experiment 1]{
		\begin{minipage}[t]{0.5\linewidth}
			\centering				\includegraphics[width=1\textwidth]{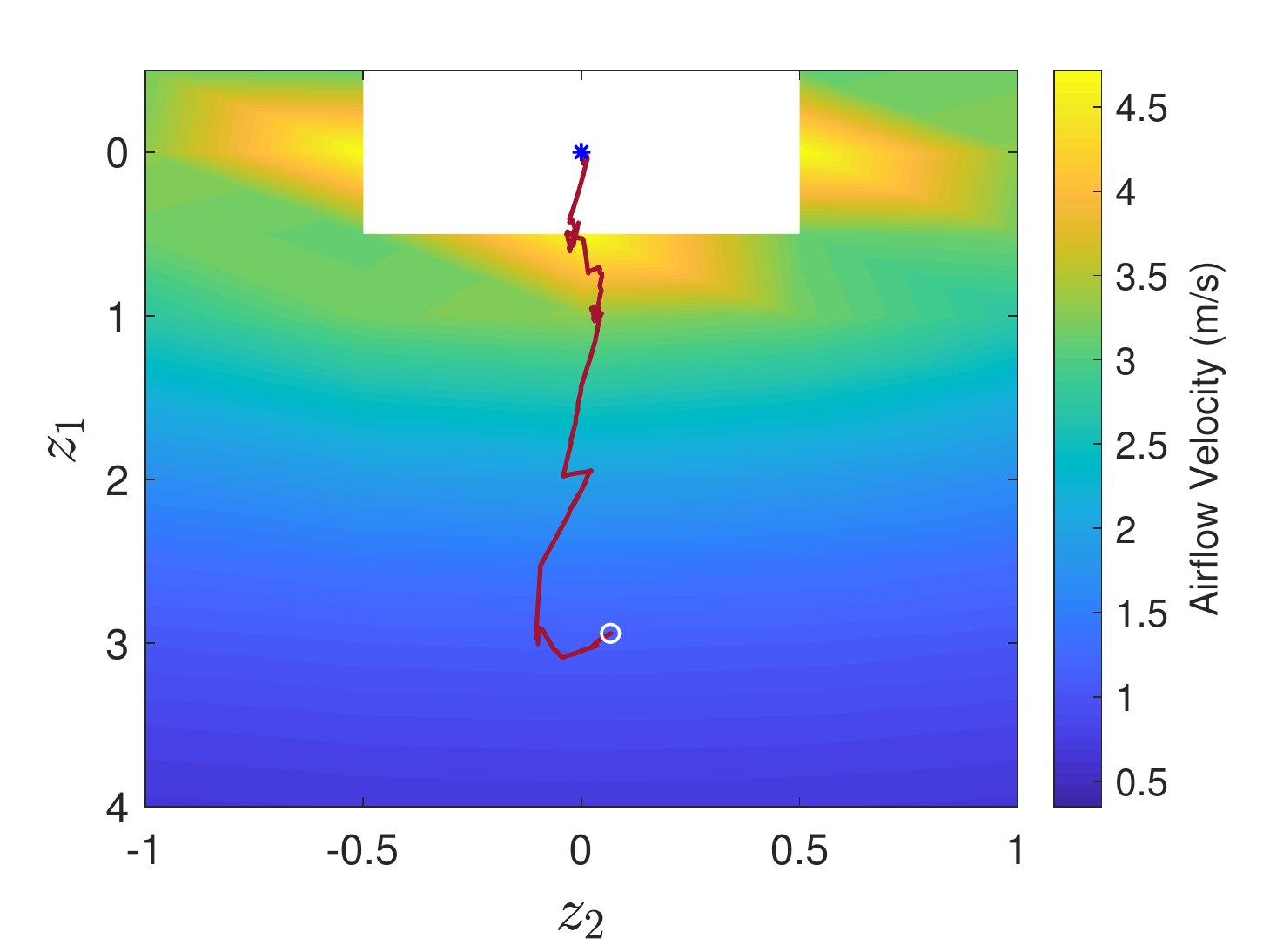}
		\end{minipage}%
	}%
	\subfigure[Experiment 2]{
		\begin{minipage}[t]{0.5\linewidth}
			\centering
		\includegraphics[width=1\textwidth]{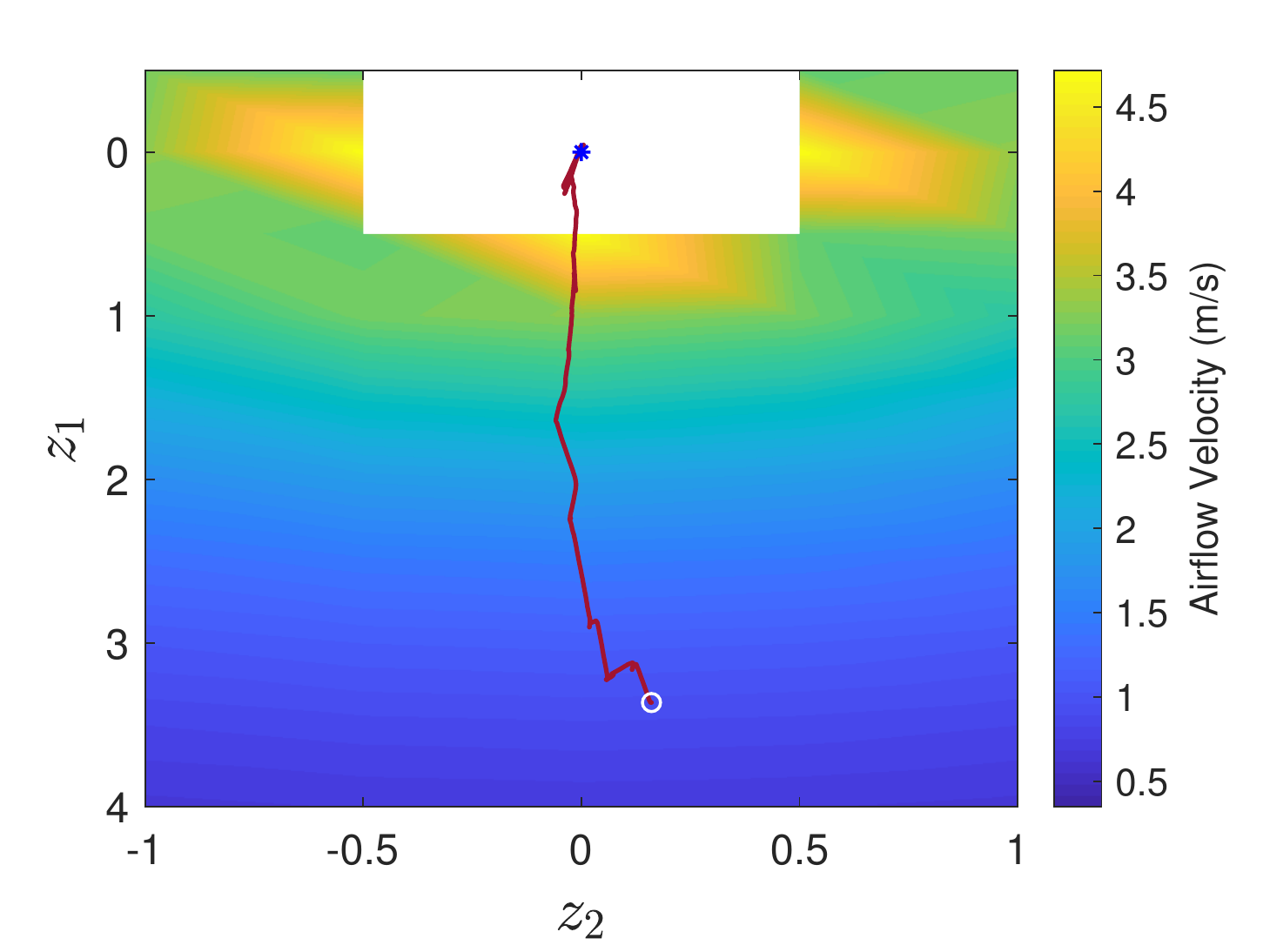}
		\end{minipage}%
	}%
	
	\subfigure[Experiment 3]{
		\begin{minipage}[t]{0.5\linewidth}
			\centering
			\includegraphics[width=1\textwidth]{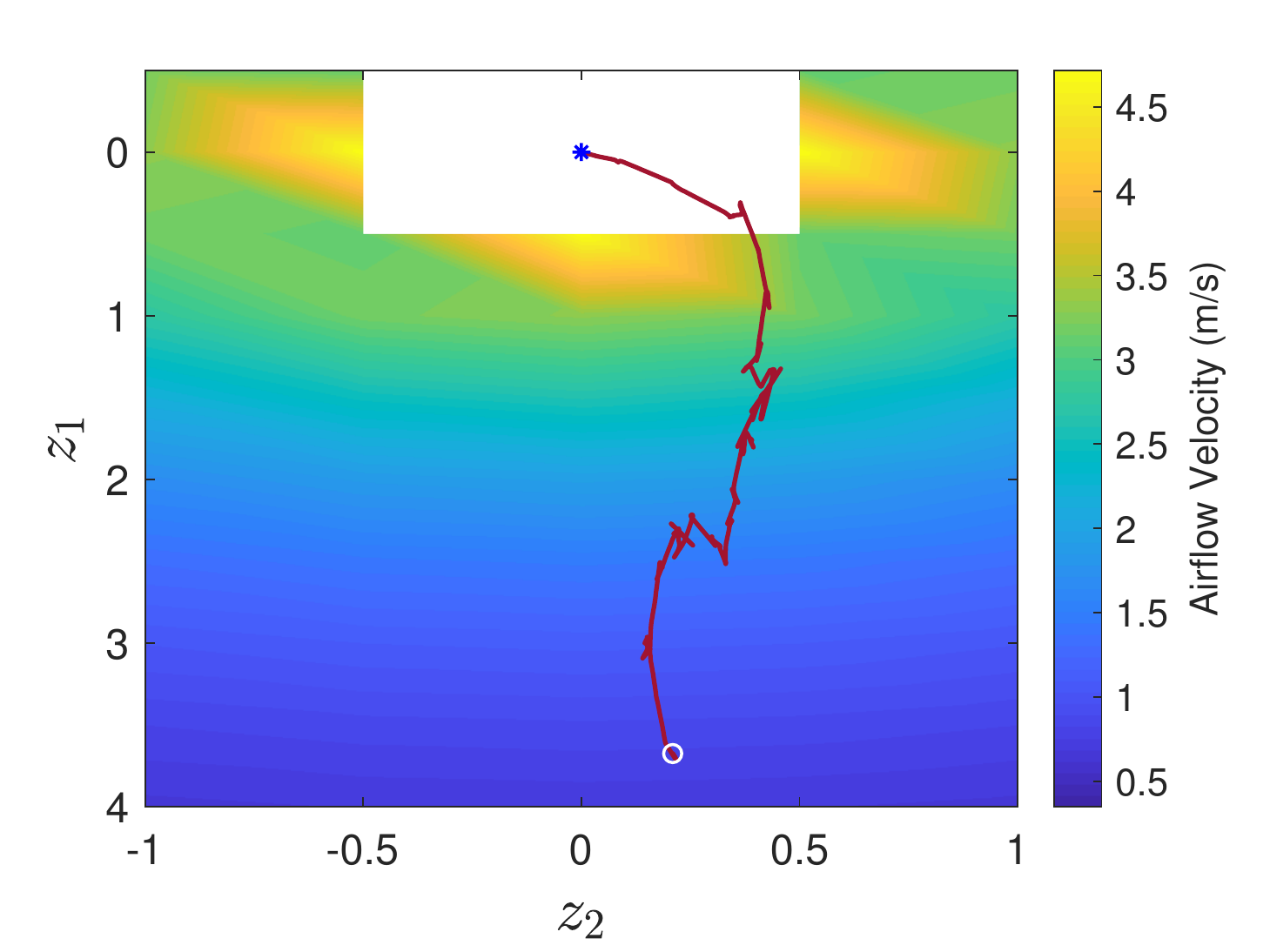}
		\end{minipage}%
	}%
	
	\centering
	\caption{Experimental results of the closed-loop system with the projected gradient-ascent control law. The plots are the resulting robot trajectories started from three different initial positions.}
	\label{fig:Experiment_1}
\end{figure}

\subsubsection{ESC-based projected gradient-ascent control law}

As described before, the ESC-based control law is designed as a redundant controller in dealing with partial sensors' failure. 
Therefore, in this experiment, we mounted only two sensors instead of all four sensors. 
The parameters for high-pass filter were set as $a=0.8$ and $h=0.5$ while the dither signal constants are given by $C_{z1}=0.2$ and $C_{z2}=0.2$.

Figure \ref{fig:Experiment_2} shows the resulting trajectories of the robot with the ESC-based control law from different initial conditions. In contrast to the simulation results in the previous section, the robot performed larger rotational motion than the ones from the simulation which can be attributed to the use of low frequency dither signal $\omega_0$. Figure \ref{box_exp} depicts the corresponding time for robot approaching the final $20\%$ distance of the fan. As it shows, the higher $\omega_0$ leads to a slower overall source seeking motion as well as low amplitude oscillations.

In the experiments using Nexus robot, implementing the dither frequency higher than $2 rad/s$ was not possible due to the hardware and physical limitation of the wheels. Nevertheless, the experiments that were conducted have shown that the proposed ESC-based control law is still able to steer the robot towards the source. This demonstrates the robustness of the algorithm to the sensor noise.

\begin{figure}[htbp]
	\centering
	
	\subfigure[$\omega_0=0.2 rad/s$]{
		\begin{minipage}[t]{0.5\linewidth}
			\centering				\includegraphics[width=1\textwidth]{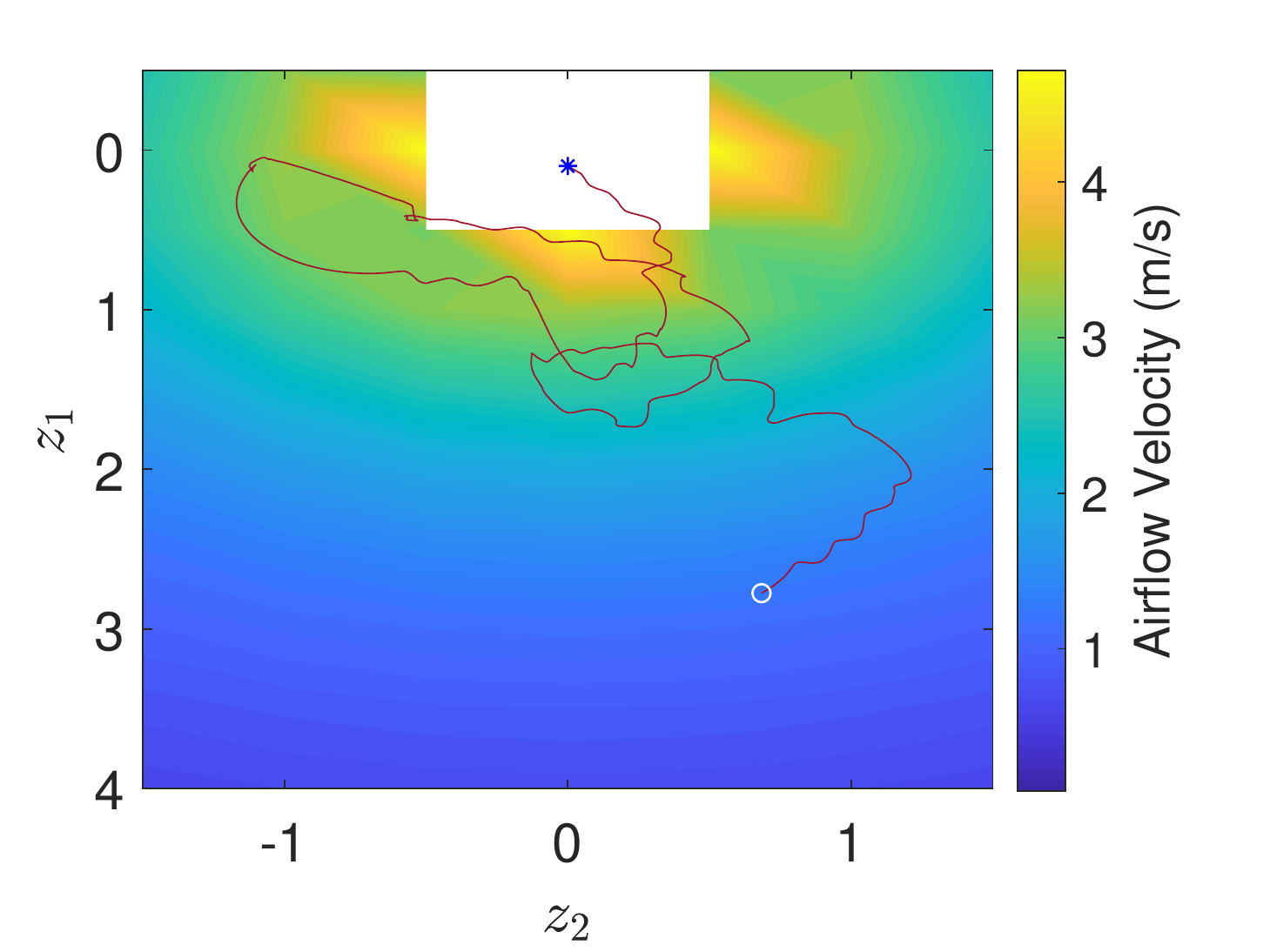}
				\end{minipage}%
	}%
	\subfigure[$\omega_0=0.2 rad/s$]{
		\begin{minipage}[t]{0.5\linewidth}
			\centering
		\includegraphics[width=1\textwidth]{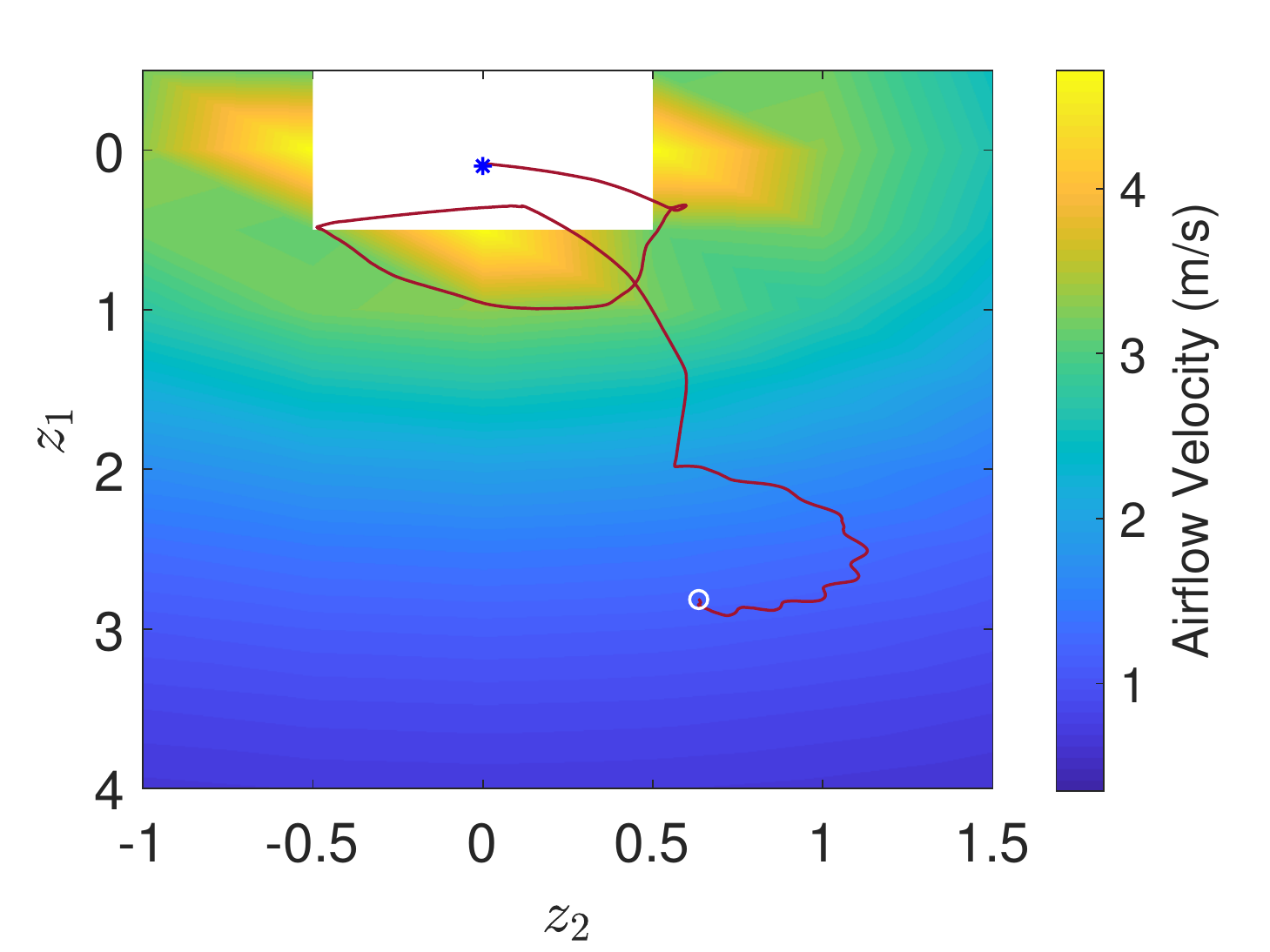}
				\end{minipage}%
	}%
	
		\subfigure[$\omega_0=0.4 rad/s$]{
		\begin{minipage}[t]{0.5\linewidth}
			\centering
		\includegraphics[width=1\textwidth]{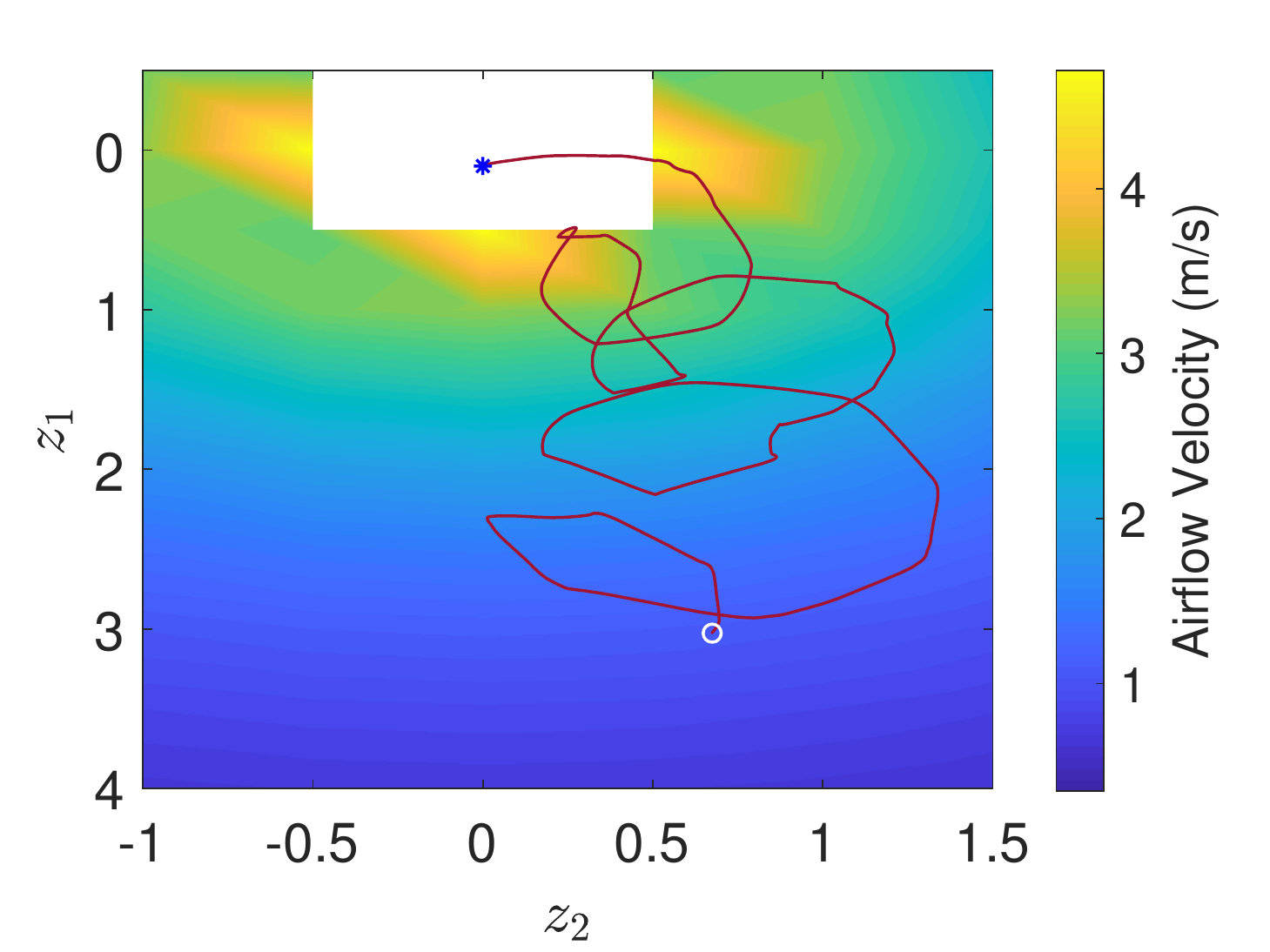}
				\end{minipage}%
	}%
	\subfigure[$\omega_0=0.6 rad/s$ (Experiment 1)]{
		\begin{minipage}[t]{0.5\linewidth}
			\centering				\includegraphics[width=1\textwidth]{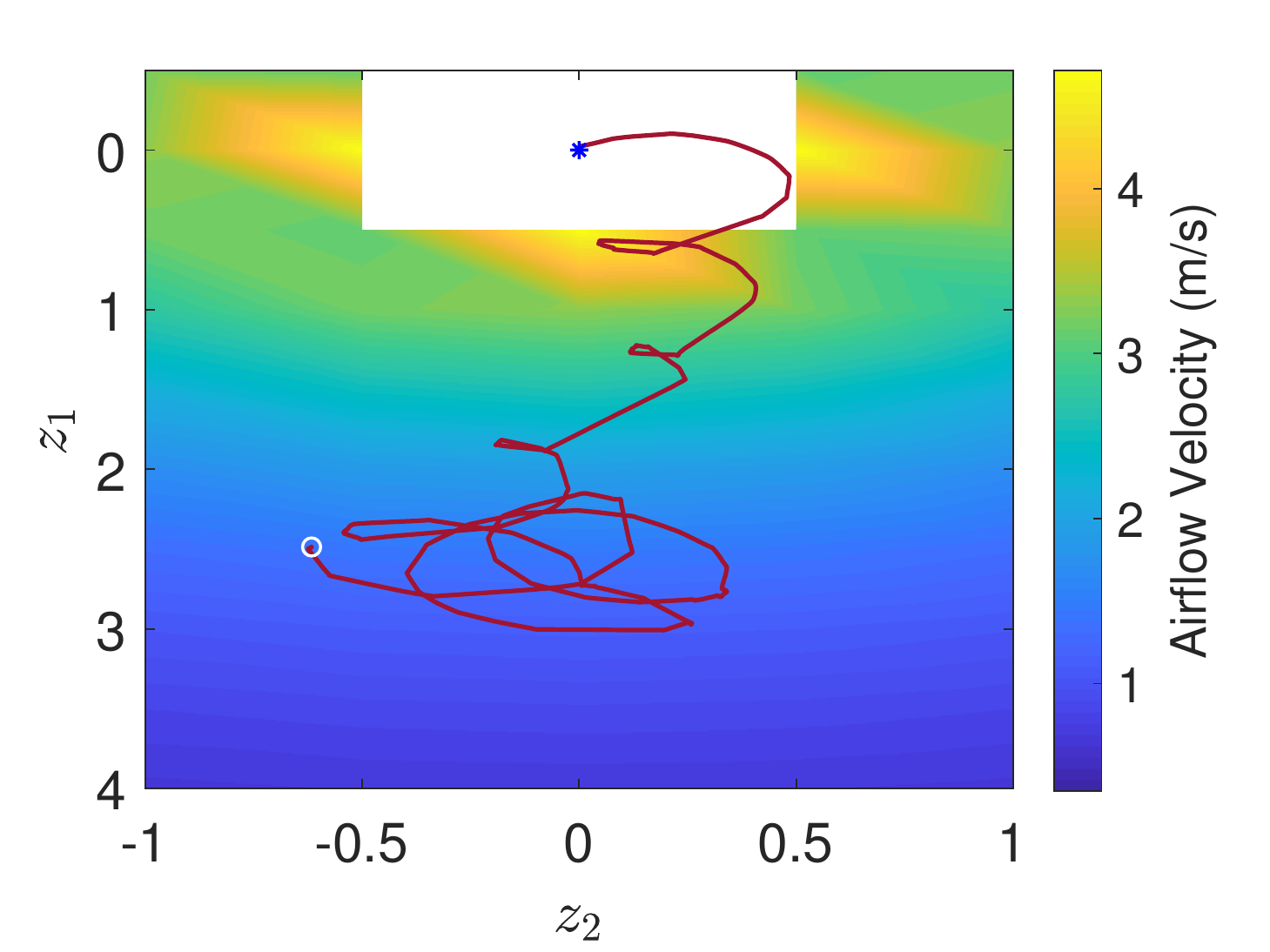}
				\end{minipage}%
	}%
	
	\subfigure[$\omega_0=0.6 rad/s$ (Experiment 2)]{
		\begin{minipage}[t]{0.5\linewidth}
			\centering
		\includegraphics[width=1\textwidth]{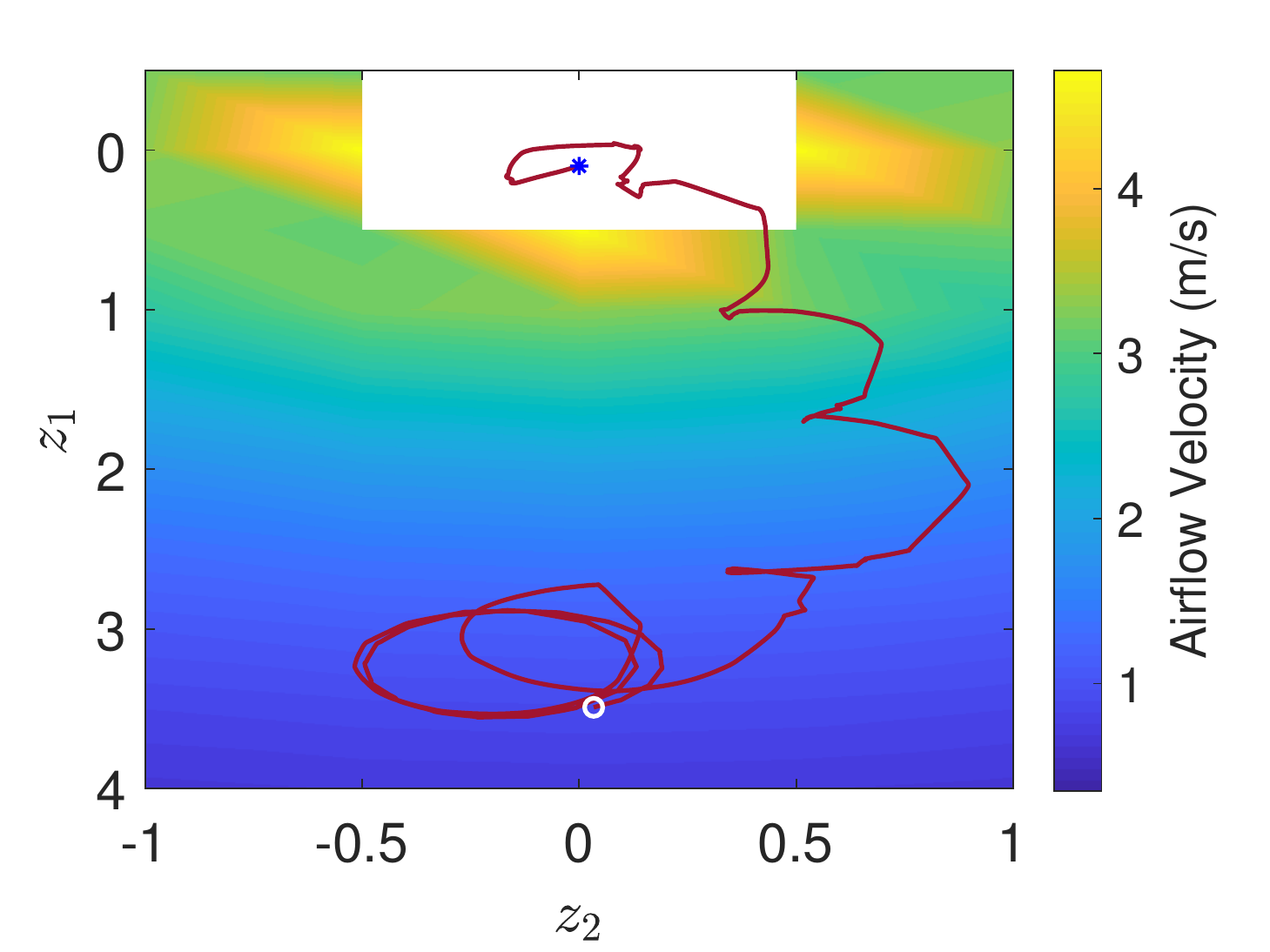}
				\end{minipage}%
	}%
	\subfigure[$\omega_0=1.5 rad/s$]{
		\begin{minipage}[t]{0.5\linewidth}
			\centering
		\includegraphics[width=1\textwidth]{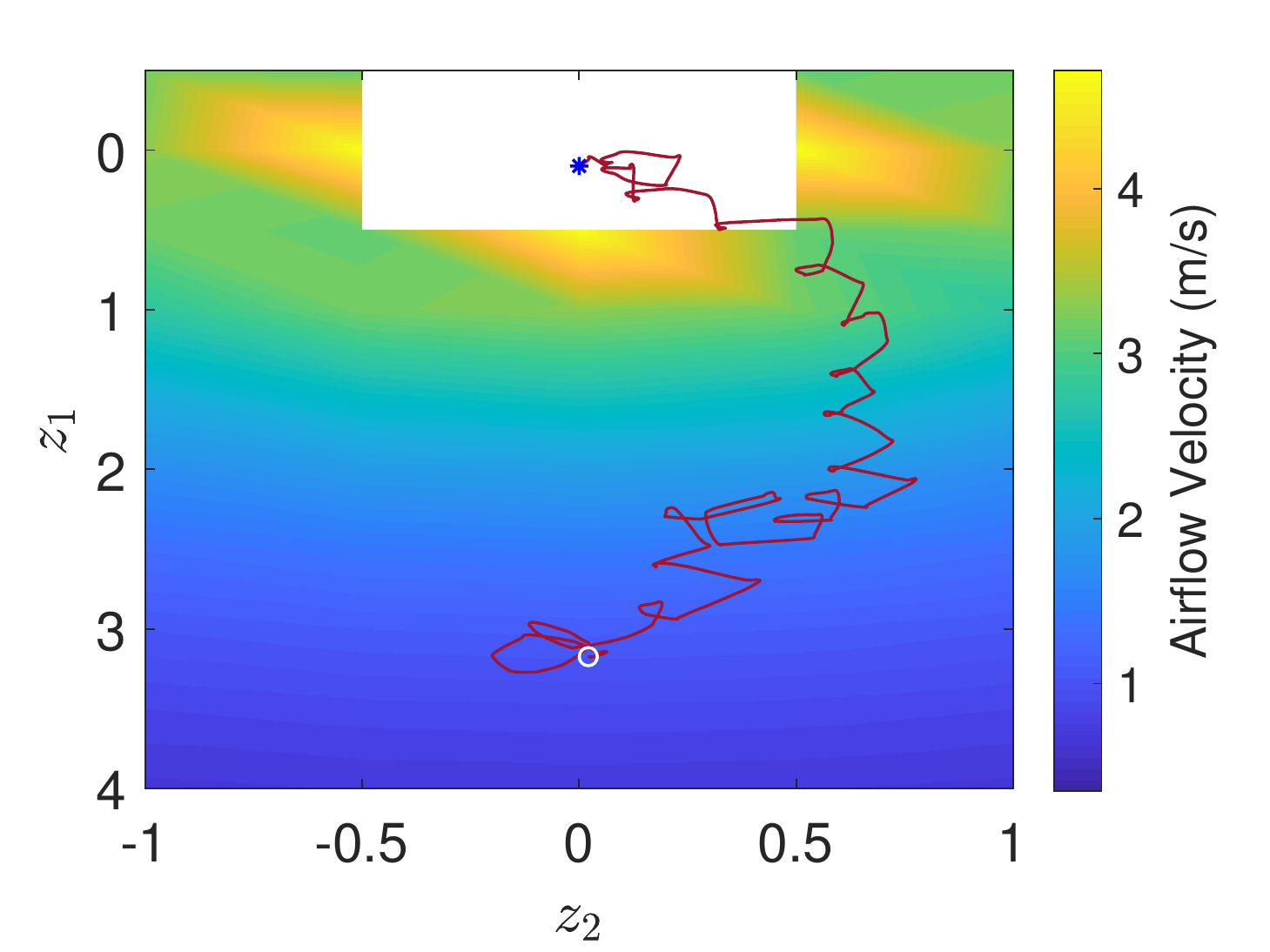}
				\end{minipage}%
	}%
	
	\centering
	\caption{Experimental results of the closed-loop system with the ESC-based projected gradient-ascent control law. The plots are the resulting robot trajectories started from different initial positions, while the dither signal was set to be $\omega_0=0.2, 0.4, 0.6, 1.5 rad/s$, respectively.}
	\label{fig:Experiment_2}
\end{figure}

\begin{figure}[htbp]
	\centering
	\includegraphics[width=0.4\textwidth]{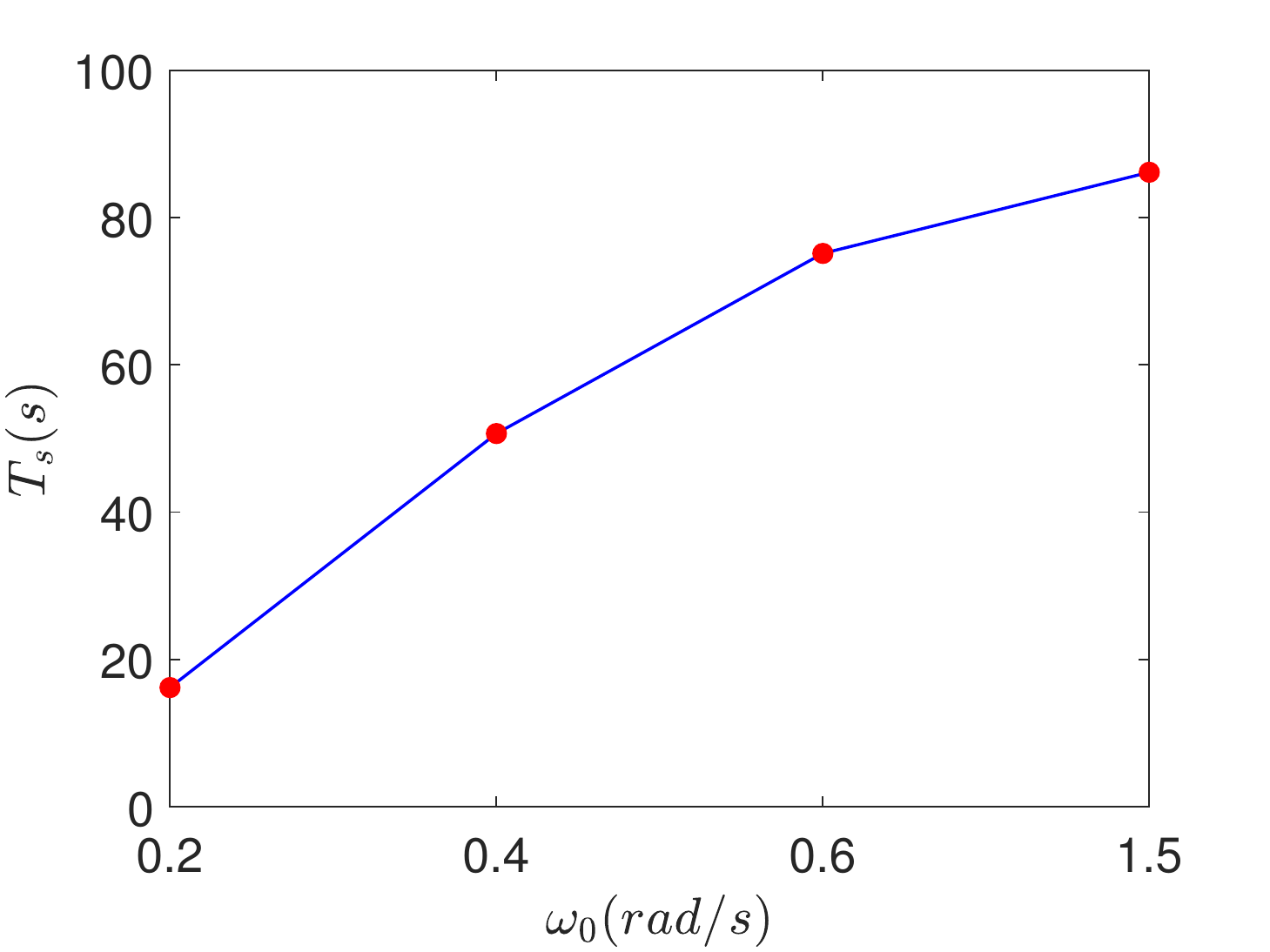}
	\caption{The result of experimental $T_s$ w.r.t the increasing $\omega_0=0.2,0.4,0.6,1.5 rad/s$, respectively. $T_s$ represents the time that robot takes to approach the final $20\%$ distance of the fan.}
	\label{box_exp}
\end{figure}

\section{Conclusions}
In this paper, we have presented two source seeking control laws for unicycle mobile robot that are based on a projected gradient-ascent method and on a combination thereof with the Extremum Seeking Control approach in order to deal with incomplete information due to sensor's failure. We evaluated the efficacy of the proposed control laws for seeking the airflow source where we designed and deployed 3D-printed graphene-based piezoresistive airflow sensors. Both simulation and experimental results showed that the proposed controllers and the novel sensor systems were able to locate the airflow source robustly in the presence of the hysteresis nonlinearity in the piezoresistive sensors and with respect to sensor noise that comes from the turbulent flow and body vibrations.

\section*{Acknowledgment}
The authors would like to acknowledge Bartje Alewijnse for the preliminary work on the source-seeking problem for unicycle, Simon Busman for setting up the Robot Operating Systems and embedded systems, Martin Stokroos for the  development of the differential analog input shield and the software interface for the Arduino-based omniwheel control electronics, and David Veninga for the the sensor calibration using the bench-top wind tunnel in Figure 4. We further thank Simon Busman and Martin Stokroos for their comments on the manuscript.

\section*{Appendix}
 \label{Appendix}
\begin{proof}
For the proof of Proposition \ref{prop_2}, firstly, by denoting $g(t) = \delta(t) - g_\ell(t)$, $t\geq 0$, as the convolution kernel of the high-pass filter $\frac{s}{s+h}$ with $\delta$ be the Dirac function and $g_\ell(t)=he^{-ht}$, we define the following error variable
\begin{equation}
e(t)=(g_\ell*J)(t)-J^*,
\end{equation}
where $*$ denotes the convolution operator. If $\Delta(t)$ denotes the output of the ``washout'' filter (as shown in Figure \ref{Model of Extremum Seeking Control}) then it follows that
\begin{align}
\nonumber \Delta(t) & =(g*J)(t)=J(t)-(g_\ell * J)(t)=J(t)-J^*-e(t)\\
\label{eq:delta}&=-c_{1}(z_1(t)-z_1^*)^2-c_{2}(z_2(t)-z_2^*)^2-e(t).
\end{align}

Let us define the following shifted variables
\begin{equation}
\begin{bmatrix}
\tilde{z}_1(t) \\
\tilde{z}_2(t)  \\
\tilde{\theta}(t)
\end{bmatrix}=\begin{bmatrix}
z_1(t)-z^*_1-a\sin(\omega_0 t)    \\
z_2(t)-z^*_2+a\cos(\omega_0 t)     \\
\theta-\theta^*
\end{bmatrix}
\end{equation}
where $\theta^* \in [0,2\pi)$ will be any stationary heading angle
when the robot is at the maximum position. Following a similar proof to that of Proposition \ref{prop_1}, we can define

\begin{equation*}
\begin{bmatrix}
z_3 \\z_4
\end{bmatrix}=\begin{bmatrix}
\cos(\theta)\\
\sin(\theta)
\end{bmatrix}=\begin{bmatrix}
\cos(\tilde{\theta}+\theta^*)  \\
\sin(\tilde{\theta}+\theta^*),
\end{bmatrix}
\end{equation*}
so that the closed-loop system when we use the projected gradient-ascent control law is as in \eqref{system_eq_new}. In the ESC-based approach, we simply replace $\nabla J$ and $\nabla J^\perp$ with their approximated ones as in \eqref{eq:widehat_J}.

Let us now consider a new time scale given by 
$\tau=\omega_0 t$. It follows then that 
\begin{equation}
\frac{\dd e}{\dd\tau }=-\frac{h}{\omega_0}\big(c_{1}(\tilde{z}_1+a\sin(\tau))^2+c_{2}(\tilde{z}_2-a\cos(\tau))^2+e\big),
\end{equation}
where we have used the relation that $\frac{\dd}{\dd t}(g_\ell * J)(t) = h(g*J)(t)$. Similarly, for $\tilde z_1$, $\tilde z_2$ and $\tilde \theta$, we have
\begin{align}
\nonumber \frac{\dd\tilde{z}_1}{\dd\tau }&=\frac{1}{\omega_0 }\frac{\dd\tilde{z}_1}{\dd t} =\frac{1}{\omega_0 }\frac{\dd(z_1-z_1^*-a\sin(\omega_0 t))}{\dd t}
\\
\label{eq:tilde_z1}&=\frac{1}{\omega_0 }(k_1 z_3(J_{z_1}z_3+J_{z_2}z_4)-a\omega_0\cos(\omega_0 t)),
\end{align}
\begin{align}
\nonumber \frac{\dd\tilde{z}_2}{\dd\tau }&=\frac{1}{\omega_0 }\frac{\dd\tilde{z}_2}{\dd t}=\frac{1}{\omega_0 }\frac{\dd(z_2-z_2^*+a\cos(\omega_0 t))}{\dd t}
\\
\label{eq:tilde_z2}&=\frac{1}{\omega_0 }(k_1 z_4(J_{z_1}z_3+J_{z_2}z_4)-a\omega_0\sin(\omega_0 t))
\end{align}
and
\begin{align}
\nonumber \frac{\dd\tilde{\theta}}{\dd\tau }&=\frac{1}{\omega_0 }\frac{\dd\tilde{\theta}}{\dd t}=\frac{1}{\omega_0 }\frac{\dd(\theta-\theta^*)}{\dd t}
\\
\label{eq:tilde_theta}&=-\frac{1}{\omega_0 }k_2(-J_{z_2}z_3+J_{z_1}z_4)
\end{align}

As we are using the dither signal, which has the period of $2k\pi$ for any positive integer $k$ (in the time scale $\tau$), we will analyze the closed-loop systems behavior by analyzing the associated averaged systems.
Using the time-scale $\tau$, we can analyze the averaged systems on the time interval $\lambda \in \left [ \tau ,\tau+2k\pi \right)$. Since the dither signal has a sufficiently high frequency such that the trajectories are approximately constant in this time interval, we can assume that $z_3\left(\frac{\lambda }{\omega _0}\right)\approx z_3\left(\frac{\tau }{\omega _0}\right)$, $z_4\left(\frac{\lambda }{\omega _0}\right)\approx z_4\left(\frac{\tau }{\omega _0}\right)$ and $\Delta \left(\frac{\lambda }{\omega _0}\right)\approx \Delta \left(\frac{\tau }{\omega _0}\right)$ for all  $\lambda \in \left [ \tau ,\tau+2k\pi \right ]$.
Computing the averaged equation to the $\tilde z_1$-system, we obtain
\begin{align*}
& \frac{\dd\tilde{z}_{1,\text{avg}}}{\dd\tau}(\tau) =\frac{1}{2k\pi}\int_{\tau}^{\tau+2k\pi}\frac{\dd\tilde{z}_1}{\dd \tau}(\lambda)\dd\lambda
\\
&=\frac{1}{\omega_0}\frac{1}{2k\pi}\int_{\tau}^{\tau+2k\pi}\left(\bluff k_1 z_3\left(\frac{\lambda}{\omega_0}\right) \left[\left(\bluff C_{z_1}\Delta\left(\frac{\lambda}{\omega_0}\right)\sin(\lambda)\right.\right. \right.
\\&\qquad+\left. \bluff a\omega_0 \cos(\lambda)\right) z_3\left(\frac{\lambda}{\omega_0}\right) + \left(\bluff a\omega_0 \sin(\lambda) \right.
\\&
\qquad
\left. \left.\bluff -C_{z_2}\Delta\left(\frac{\lambda}{\omega_0}\right) \cos(\lambda)\right)z_4\left(\frac{\lambda}{\omega_0}\right) \right]
\\&\qquad \left. \bluff -a\omega_0 \cos(\lambda)\right)\dd\lambda
\end{align*}
\begin{align}
\nonumber &\approx\frac{1}{\omega_0}\frac{1}{2k\pi}\int_{\tau}^{\tau+2k\pi}\left(\bluff k_1 z_3\left(\frac{\tau}{\omega_0}\right) \left[\left(\bluff C_{z_1}\Delta\left(\frac{\tau}{\omega_0}\right)\sin(\lambda)\right.\right. \right.
\\
\nonumber &\qquad+\left. \bluff a\omega_0 \cos(\lambda)\right) z_3\left(\frac{\tau}{\omega_0}\right) + \left(\bluff a\omega_0 \sin(\lambda) \right.
\\
\nonumber &
\qquad
\left. \left.\bluff -C_{z_2}\Delta\left(\frac{\tau}{\omega_0}\right) \cos(\lambda)\right)z_4\left(\frac{\tau}{\omega_0}\right) \right]
\\
\label{eq:z1_avg_dot}&\qquad \left. \bluff -a\omega_0 \cos(\lambda)\right)\dd\lambda
\end{align}

By expanding the first term above that involves $\Delta$ using \eqref{eq:delta}, we have \begin{align*}
&\int_\tau^{\tau+2k\pi} \Delta\left(\frac{\tau}{\omega_0}\right)\sin(\lambda) \dd\lambda \\
& = - \int_\tau^{\tau+2k\pi} c_1\left(\tilde z_1^2\left(\frac{\tau}{\omega_0}\right)\sin(\lambda) + 2a\tilde z_1\left(\frac{\tau}{\omega_0}\right)\sin^2(\lambda) \bluff \right.
\end{align*}
\begin{align}
\nonumber& \qquad  \left. \bluff + a^2\sin^3(\lambda)\right) \dd \lambda   - \int_\tau^{\tau+2k\pi} c_2\left(\tilde z_2^2\left(\frac{\tau}{\omega_0}\right)\sin(\lambda)\bluff \right. \\
\nonumber & \qquad \left. \bluff - 2a\tilde z_1\left(\frac{\tau}{\omega_0}\right)\sin(\lambda)\cos(\lambda) + a^2\cos^2(\lambda)\sin(\lambda)\right) \dd \lambda \\
\nonumber & \qquad - \int_\tau^{\tau+2k\pi} e\left(\frac{\tau}{\omega_0}\right)\sin(\lambda) \dd\lambda \\
\label{eq:first_delta}& = -2k\pi a c_1 \tilde{z}_{1,\text{avg}}.
\end{align}
Similar computation can be performed for second term in \eqref{eq:z1_avg_dot} that involves $\Delta$ which gives us
\begin{align}\label{eq:second_delta}
&\int_\tau^{\tau+2k\pi} \Delta\left(\frac{\tau}{\omega_0}\right)\cos(\lambda) \dd\lambda =  2k\pi a c_2  \tilde{z}_{2,\text{avg}}.
\end{align}
Substituting \eqref{eq:first_delta} and \eqref{eq:second_delta} into \eqref{eq:z1_avg_dot}, along with the fact that $z_3 = \cos(\theta)$ and $z_4 = \sin(\theta)$, yields
\begin{align*}
\frac{\dd\tilde{z}_{1,\text{avg}}}{\dd\tau} = & -\frac{ak_1\cos(\tilde{\theta}_{\text{avg}})}{\omega_0}
\left(\bluff C_{z1}c_{1}\tilde{z}_{1,\text{avg}}\cos(\tilde{\theta}_{\text{avg}})\right.
\\& \left. \bluff+C_{z2}c_{2}\tilde{z}_{2,\text{avg}}\sin(\tilde{\theta}_{\text{avg}})\right)
\end{align*}

Similarly, we can compute the averaged system to the $\tilde z_2$ system as follows.

\begin{align*}
& \frac{\dd\tilde{z}_{2,\text{avg}}}{\dd\tau}(\tau)=\frac{1}{2k\pi}\int_{\tau}^{\tau+2k\pi}\frac{\dd\tilde{z}_2}{\dd \tau}\dd\lambda\\
&=\frac{1}{\omega_0}\frac{1}{2k\pi}\int_{\tau}^{\tau+2k\pi}\left(\bluff k_1 z_4\left(\frac{\lambda}{\omega_0}\right) \left[\left(\bluff C_{z_1}\Delta\left(\frac{\lambda}{\omega_0}\right)\sin(\lambda)\right.\right. \right.
\\&\qquad+\left. \bluff a\omega_0 \cos(\lambda)\right) z_3\left(\frac{\lambda}{\omega_0}\right) + \left(\bluff a\omega_0 \sin(\lambda) \right.
\\&
\qquad
\left. \left.\bluff -C_{z_2}\Delta\left(\frac{\lambda}{\omega_0}\right) \cos(\lambda)\right)z_4\left(\frac{\lambda}{\omega_0}\right) \right]
\\&\qquad \left. \bluff -a\omega_0 \sin(\lambda)\right)\dd\lambda
\end{align*}
\begin{align}
\nonumber &\approx\frac{1}{\omega_0}\frac{1}{2k\pi}\int_{\tau}^{\tau+2k\pi}\left(\bluff k_1 z_4\left(\frac{\tau}{\omega_0}\right) \left[\left(\bluff C_{z_1}\Delta\left(\frac{\tau}{\omega_0}\right)\sin(\lambda)\right.\right. \right.
\\
\nonumber &\qquad+\left. \bluff a\omega_0 \cos(\lambda)\right) z_3\left(\frac{\tau}{\omega_0}\right) + \left(\bluff a\omega_0 \sin(\lambda) \right.
\\
\nonumber &
\qquad
\left. \left.\bluff -C_{z_2}\Delta\left(\frac{\tau}{\omega_0}\right) \cos(\lambda)\right)z_4\left(\frac{\tau}{\omega_0}\right) \right]
\\
\label{eq:z2_avg_dot}&\qquad \left. \bluff -a\omega_0 \sin(\lambda)\right)\dd\lambda
\end{align}
Following the same computation of the terms in \eqref{eq:z2_avg_dot} involving $\Delta$ as done before in \eqref{eq:first_delta} and \eqref{eq:second_delta}, it follows that
\begin{align*}
\frac{\dd\tilde{z}_{2,\text{avg}}}{\dd\tau}&=-\frac{ak_1\sin(\tilde{\theta}_{\text{avg}})}{\omega_0}\left(\bluff C_{z1}c_1\tilde{z}_{1,\text{avg}}\cos(\tilde{\theta}_{\text{avg}})\right.
\\& \left. +C_{z2}c_{2}\tilde{z}_{2,\text{avg}}\sin(\tilde{\theta}_{\text{avg}}) \bluff\right)
\end{align*}
Finally, we can compute the averaged system of $\tilde \theta$ as follows.

\begin{align*}
& \frac{\dd\tilde{\theta}_{\text{avg}}}{\dd\tau}(\tau)=\frac{1}{2k\pi}\int_{\tau}^{\tau+2k\pi}\frac{\dd\tilde{\theta}}{\dd\tau}\dd\lambda
\\
&=\frac{1}{\omega_0}\frac{1}{2k\pi}\int_{\tau}^{\tau+2k\pi} k_2 \left[-\left(\bluff -C_{z_2}\Delta\left(\frac{\lambda}{\omega_0}\right)\cos(\lambda)\right.\right.
\\&\qquad+\left. \bluff a\omega_0 \sin(\lambda)\right) z_3\left(\frac{\lambda}{\omega_0}\right) + \left(\bluff a\omega_0 \cos(\lambda) \right.
\\&
\qquad
\left. \left.\bluff +C_{z_1}\Delta\left(\frac{\lambda}{\omega_0}\right) \sin(\lambda)\right)z_4\left(\frac{\lambda}{\omega_0}\right) \right] \dd\lambda
\end{align*}
\begin{align}
\nonumber &\approx\frac{1}{\omega_0}\frac{1}{2k\pi}\int_{\tau}^{\tau+2k\pi} k_2 \left[-\left(\bluff -C_{z_2}\Delta\left(\frac{\tau}{\omega_0}\right)\cos(\lambda)\right.\right.
\\
\nonumber &\qquad+\left. \bluff a\omega_0 \sin(\lambda)\right) z_3\left(\frac{\tau}{\omega_0}\right) + \left(\bluff a\omega_0 \cos(\lambda) \right.
\\
\label{eq:theta_avg_dot} &
\qquad
\left. \left.\bluff +C_{z_1}\Delta\left(\frac{\tau}{\omega_0}\right) \sin(\lambda)\right)z_4\left(\frac{\tau}{\omega_0}\right) \right]
\dd\lambda
\end{align}
As before, substituting \eqref{eq:first_delta} and \eqref{eq:second_delta} into the above equation yields
\begin{align*}
\frac{\dd\tilde{\theta}_{\text{avg}}}{\dd\tau}&=-\frac{ak_2}{\omega_0}\left(\bluff C_{z1}c_{1}\tilde{z}_{1,\text{avg}}\sin(\tilde{\theta}_{\text{avg}}) \right.
\\&\qquad\left.\bluff - C_{z2}c_{2}\tilde{z}_{2,\text{avg}}\cos(\tilde{\theta}_{\text{avg}})\right)
\end{align*}
By defining the following auxiliary state variables
\begin{equation}
\begin{bmatrix}
z_5 \\ z_6
\end{bmatrix}=\begin{bmatrix}
\cos(\tilde{\theta}_{\text{avg}}) \\ \sin(\tilde{\theta}_{\text{avg}})
\end{bmatrix}
\end{equation}
the averaged closed-loop system in the time-scale $\tau$ is then given by
\begin{equation}
\begin{bmatrix}
\dot{\tilde{z}}_{1,\text{avg}}  \\
\dot{\tilde{z}}_{2,\text{avg}}  \\
\dot{z}_5   \\
\dot{z}_6
\end{bmatrix}=\begin{bmatrix}
-\frac{ak_1\cos(\tilde{\theta}_{\text{avg}})}{\omega_0}\big( C_{z1}c_{1}\tilde{z}_{1,\text{avg}}\cos(\tilde{\theta}_{\text{avg}})\\+C_{z2}c_{2}\tilde{z}_{2,\text{avg}}\sin(\tilde{\theta}_{\text{avg}})\big)  \\
-\frac{ak_1\sin(\tilde{\theta}_{\text{avg}})}{\omega_0}\big(C_{z1}c_{1}\tilde{z}_{1,\text{avg}}\cos(\tilde{\theta}_{\text{avg}})\\+C_{z2}c_{2}\tilde{z}_{2,\text{avg}}\sin(\tilde{\theta}_{\text{avg}})\big)  \\
\frac{ak_2\sin(\tilde{\theta_{\text{avg}}})}{\omega_0}\big(C_{z1}c_{1}\tilde{z}_{1,\text{avg}}\sin(\tilde{\theta}_{\text{avg}})\\-C_{z2}c_{2}\tilde{z}_{2,\text{avg}}\cos(\tilde{\theta}_{\text{avg}})\big) \\
-\frac{ak_2\cos(\tilde{\theta_{\text{avg}}})}{\omega_0}\big(C_{z1}c_{1}\tilde{z}_{1,\text{avg}}\sin(\tilde{\theta}_{\text{avg}})\\-C_{z2}c_{2}\tilde{z}_{2,\text{avg}}\cos(\tilde{\theta}_{\text{avg}})\big)
\end{bmatrix}.
\end{equation}

We will now analyze the stability of the averaged closed-loop system above with state variables $\tilde z_{\text{avg}}=\bbm{\tilde{z}_{1,\text{avg}} & \tilde{z}_{2,\text{avg}} & z_5 & z_6}^T$. For this purpose, we consider the following function
\begin{equation}
V(\tilde z_\text{avg})=\frac{1}{2}\left(C_{z1}c_{1}\tilde{z}^2_{1,\text{avg}}+C_{z2}c_{2}\tilde{z}^2_{2,\text{avg}}+z^2_{5}+z^2_{6}\right)
\end{equation}
which is positive definite and radially unbounded. A routine computation on its time derivative gives us
\begin{align}
\nonumber \dot{V}&=C_{z1}c_{1}\tilde{z}_{1,\text{avg}}\dot{\tilde{z}}_{1,\text{avg}}+C_{z2}c_{2}\tilde{z}_{2,\text{avg}}\dot{\tilde{z}}_{2,\text{avg}}+\tilde{z}_{5}\dot{\tilde{z}}_{5}+\tilde{z}_{6}\dot{\tilde{z}}_{6}
\\
\nonumber &=-\frac{ak_1}{\omega_0}\left(\tilde{z}_{1,\text{avg}}z_5C_{z1}c_{1}+\tilde{z}_{2,\text{avg}}z_6C_{z2}c_{2}\right)^2 \\
\label{eq:V_dot_ztilde}& = - \frac{ak_1}{\omega_0}\left(\bbm{C_{z_1}c_1\tilde z_{1,\text{avg}} & C_{z_2}c_2\tilde z_{2,\text{avg}}}\bbm{z_5 \\ z_6}\right)^2 \leq 0.
\end{align}

By the radial unboundedness of $V$, the above inequality implies that $\tilde z_{\text{avg}}$ is bounded. Following the similar arguments as in the proof of Proposition \ref{prop_1}, we will prove now the convergence of \eqref{eq:ESC_convergence}. By the La-Salle's invariance principle, the compactness of the solution of $\tilde z_{\text{avg}}$ means that it will converge to the largest invariance set $\Omega $ where $\dot{V}(\tau)=0$ for all $\tau\geq 0$. In other words, in $\Omega$, $\tilde z_{\text{avg}}$  satisfy
\begin{equation}\label{eq:M_in_omega}
\bbm{C_{z_1}c_1\tilde z_{1,\text{avg}}(\tau) & C_{z_2}c_2\tilde z_{2,\text{avg}}(\tau)}\bbm{z_5(\tau) \\ z_6(\tau)}=0
\end{equation}
for all $\tau \geq 0$.

In the following, for showing that in $\Omega$, $\tilde z_{1,\text{avg}}=0$ and $\tilde z_{2,\text{avg}}=0$, we will use a contradiction. Suppose now that $\tilde z_{1,\text{avg}}\neq 0$ or $\tilde z_{2,\text{avg}}\neq 0$ in $\Omega$. Firstly, we define
\[
M(\tau) := \bbm{C_{z_1}c_1\tilde z_{1,\text{avg}}(\tau) & C_{z_2}c_2\tilde z_{2,\text{avg}}(\tau)}\bbm{z_5(\tau) \\ z_6(\tau)}.
\]
Computing its time derivative w.r.t. $\tau$, we have
\begin{align*}
\dot{M}&=\begin{bmatrix}
C_{z1}c_{1}\dot{\tilde{z}}_{1,\text{avg}} & C_{z2}c_{2}\dot{\tilde{z}}_{2,\text{avg}}
\end{bmatrix}\begin{bmatrix}
z_5  \\z_6
\end{bmatrix}
\\&\qquad+\begin{bmatrix}
C_{z1}c_{1}\tilde{z}_{1,\text{avg}} & C_{z2}c_{2}\tilde{z}_{2,\text{avg}}
\end{bmatrix}\begin{bmatrix}
\dot{z}_5 \\ \dot{z}_6
\end{bmatrix}
\\&=-\frac{ak_1}{\omega_0}\left(\bluff z^2_5C_{z1}c_{1}\left(\tilde{z}_{1,\text{avg}}z_5C_{z1}c_{1}+\tilde{z}_{2,\text{avg}}z_6C_{z2}c_{2}\right) \right.
\\&\qquad \left. \bluff +z^2_6C_{z2}c_{2}\left(\tilde{z}_{1,\text{avg}}z_5C_{z1}c_{1}+\tilde{z}_{2,\text{avg}}z_6C_{z2}c_{2}\right)\right)
\\&\qquad+\frac{ak_2}{\omega_0}\left(\tilde{z}_{1,\text{avg}}z_6C_{z1}c_{1}-\tilde{z}_{2,\text{avg}}z_5C_{z2}c_{2}\right)^2.
\end{align*}
Since in $\Omega$ we have that \eqref{eq:M_in_omega} holds for all $\tau\geq 0$, then the first two terms in the above equality are equal to zero so that
\begin{align*}
\dot M & = \frac{ak_2}{\omega_0}\left(\tilde{z}_{1,\text{avg}}z_6C_{z1}c_{1}-\tilde{z}_{2,\text{avg}}z_5C_{z2}c_{2}\right)^2 \\
& = \frac{ak_2}{\omega_0}\left(\bbm{-C_{z_2}c_2\tilde z_{2,\text{avg}} & C_{z_1}c_1\tilde z_{1,\text{avg}}}\bbm{z_5 \\ z_6}\right)^2.
\end{align*}
It is immediate to check that the vector $\sbm{C_{z_1}c_1\tilde z_{1,\text{avg}} & C_{z_2}c_2\tilde z_{2,\text{avg}}}^T$ in \eqref{eq:M_in_omega} is orthogonal to the vector $\sbm{-C_{z_2}c_2\tilde z_{2,\text{avg}} & C_{z_1}c_1\tilde z_{1,\text{avg}}}^T$ above. Consequently, as the $\bbm{z_5 & z_6}^T$ is orthogonal to $\sbm{C_{z_1}c_1\tilde z_{1,\text{avg}} & C_{z_2}c_2\tilde z_{2,\text{avg}}}^T$ in $\Omega$ (c.f. \eqref{eq:M_in_omega}), it implies that it is co-linear with $\sbm{-C_{z_2}c_2\tilde z_{2,\text{avg}} & C_{z_1}c_1\tilde z_{1,\text{avg}}}^T$. Thus, if $\tilde z_{1,\text{avg}}\neq 0$ or $\tilde z_{2,\text{avg}}\neq 0$ in $\Omega$ then
\begin{align*}
\dot M
& = \frac{ak_2}{\omega_0}\left(\bbm{-C_{z_2}c_2\tilde z_{2,\text{avg}} & C_{z_1}c_1\tilde z_{1,\text{avg}}}\bbm{z_5 \\ z_6}\right)^2 \neq 0,
\end{align*}
which contradicts \eqref{eq:M_in_omega}. Therefore, we establish that $\tilde z_{1,\text{avg}}=0$ and $\tilde z_{2,\text{avg}}=0$ in $\Omega$. By La-Salle invariance principle, it follows that $\tilde z_{1,\text{avg}}(\tau)\to 0$ and $\tilde z_{2,\text{avg}}(\tau) \to 0$ as $\tau \to \infty$. This concludes the proof.
\end{proof}

\ifCLASSOPTIONcaptionsoff
  \newpage
\fi

\begin{IEEEbiography}[{\includegraphics[width=1in,height=1.25in,clip,keepaspectratio]{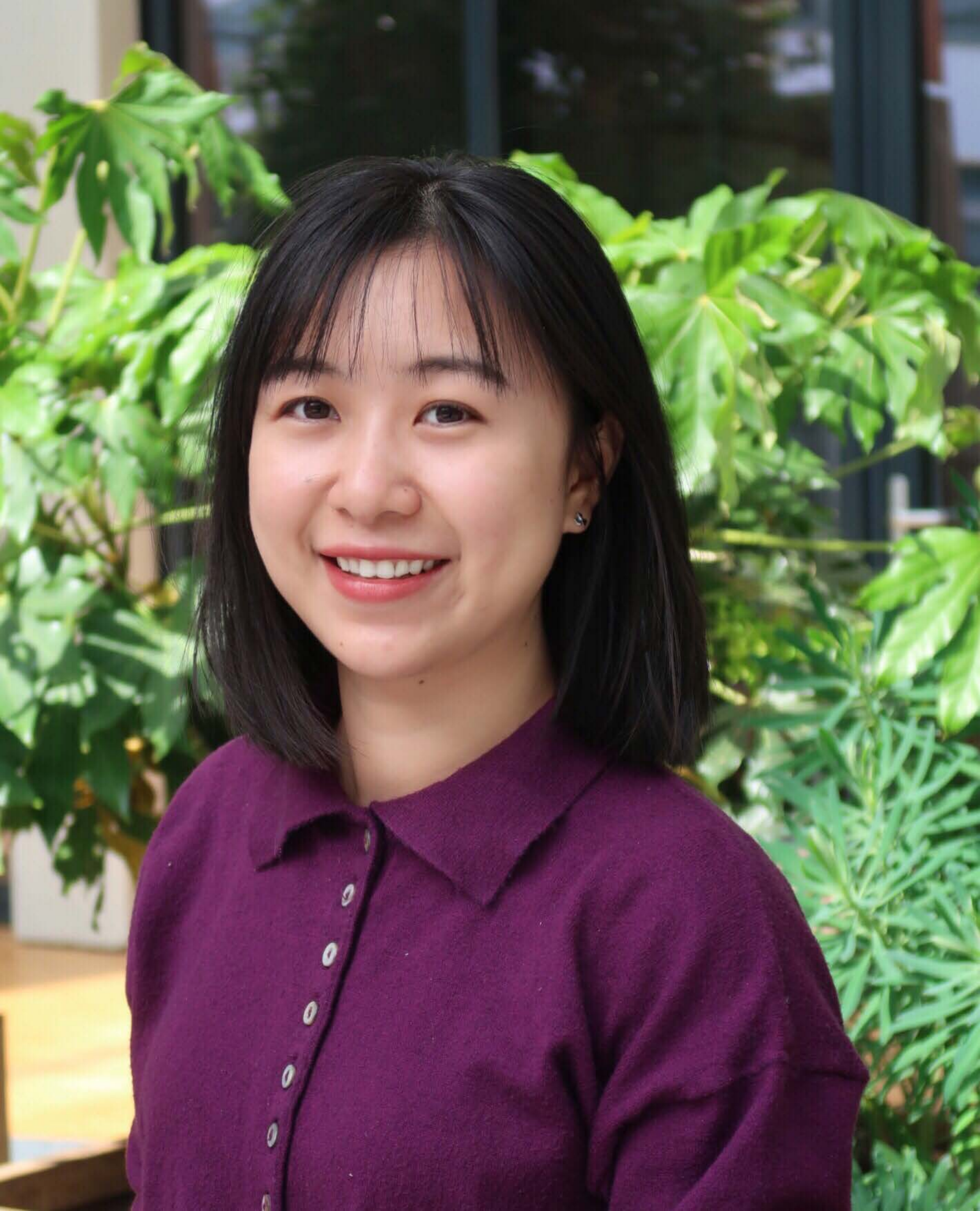}}]{Tinghua Li}
received the M.Sc. degree in control engineering from Shanghai University, Shanghai, China, in 2019. She is currently working towards the PhD. degree with the Faculty of Science and Engineering, University of Groningen, Groningen, The Netherlands. Her research interests include sensor application, motion control and navigation for autonomous mobile robots.
\end{IEEEbiography}
\vfill

\begin{IEEEbiography}[{\includegraphics[width=1in,height=1.25in,clip,keepaspectratio]{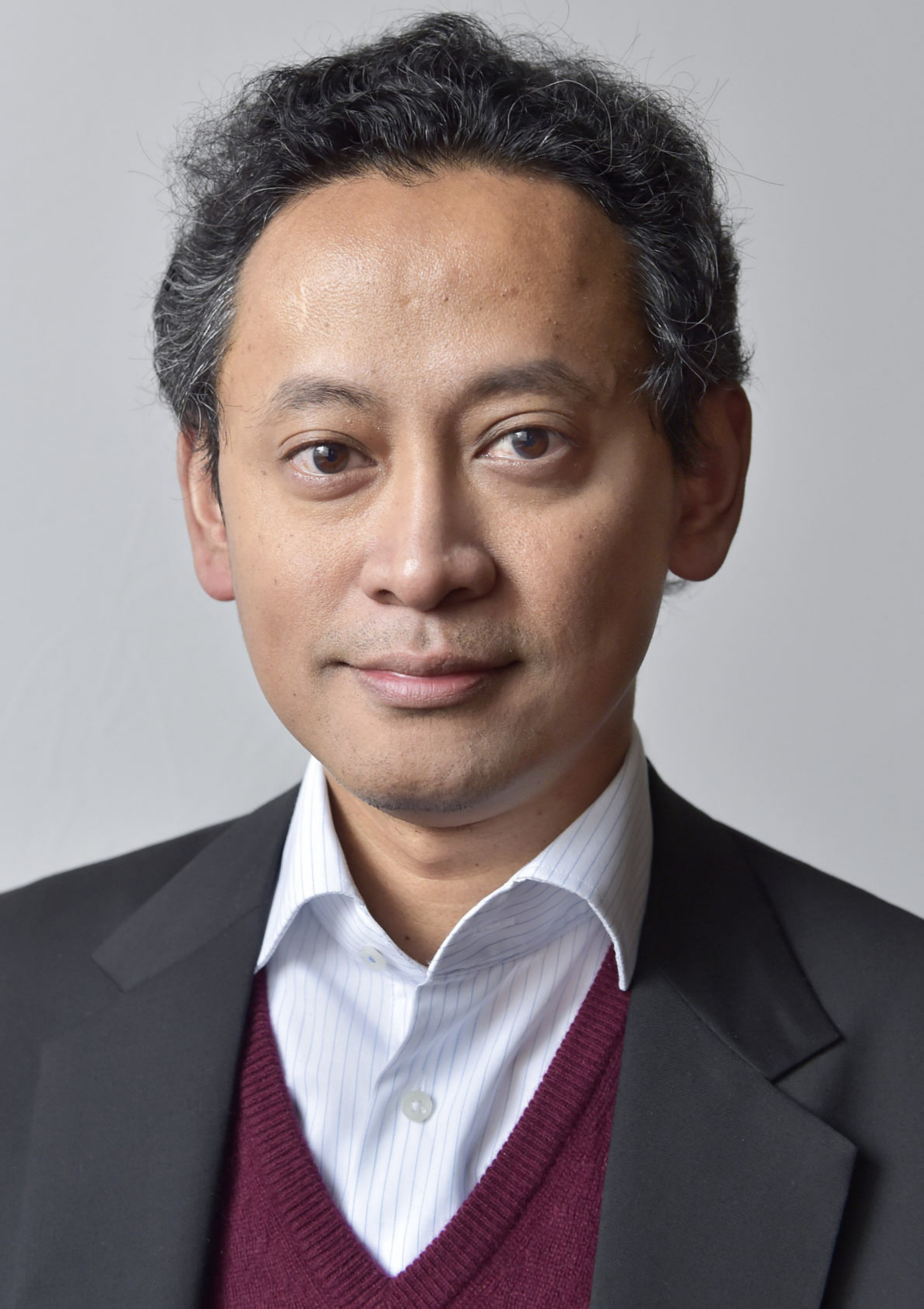}}]{Bayu Jayawardhana (SM’13)} received the B.Sc. degree in electrical and electronics engineering from the Institut Teknologi Bandung, Bandung, Indonesia, in 2000, the M.Eng. degree in electrical and electronics engineering from the Nanyang Technological University, Singapore, in 2003, and the Ph.D. degree in electrical and electronics engineering from Imperial College London, London, U.K., in 2006.

He is currently a professor of mechatronics and control of nonlinear systems in the Faculty of Science and Engineering, University of Groningen, Groningen, The Netherlands. He was with Bath University, Bath, U.K., and with University of Manchester, Manchester, U.K. His research interests include the analysis of nonlinear systems, systems with hysteresis, mechatronics, robotics and systems biology. Prof. Jayawardhana is a Subject Editor of the International Journal of Robust and Nonlinear Control, an Associate Editor of the European Journal of Control and a member of the Conference Editorial Board of the IEEE Control Systems Society.
\end{IEEEbiography}
\vfill
\vspace{-100 pt}

\begin{IEEEbiography}[{\includegraphics[width=1in,height=1.25in,clip,keepaspectratio]{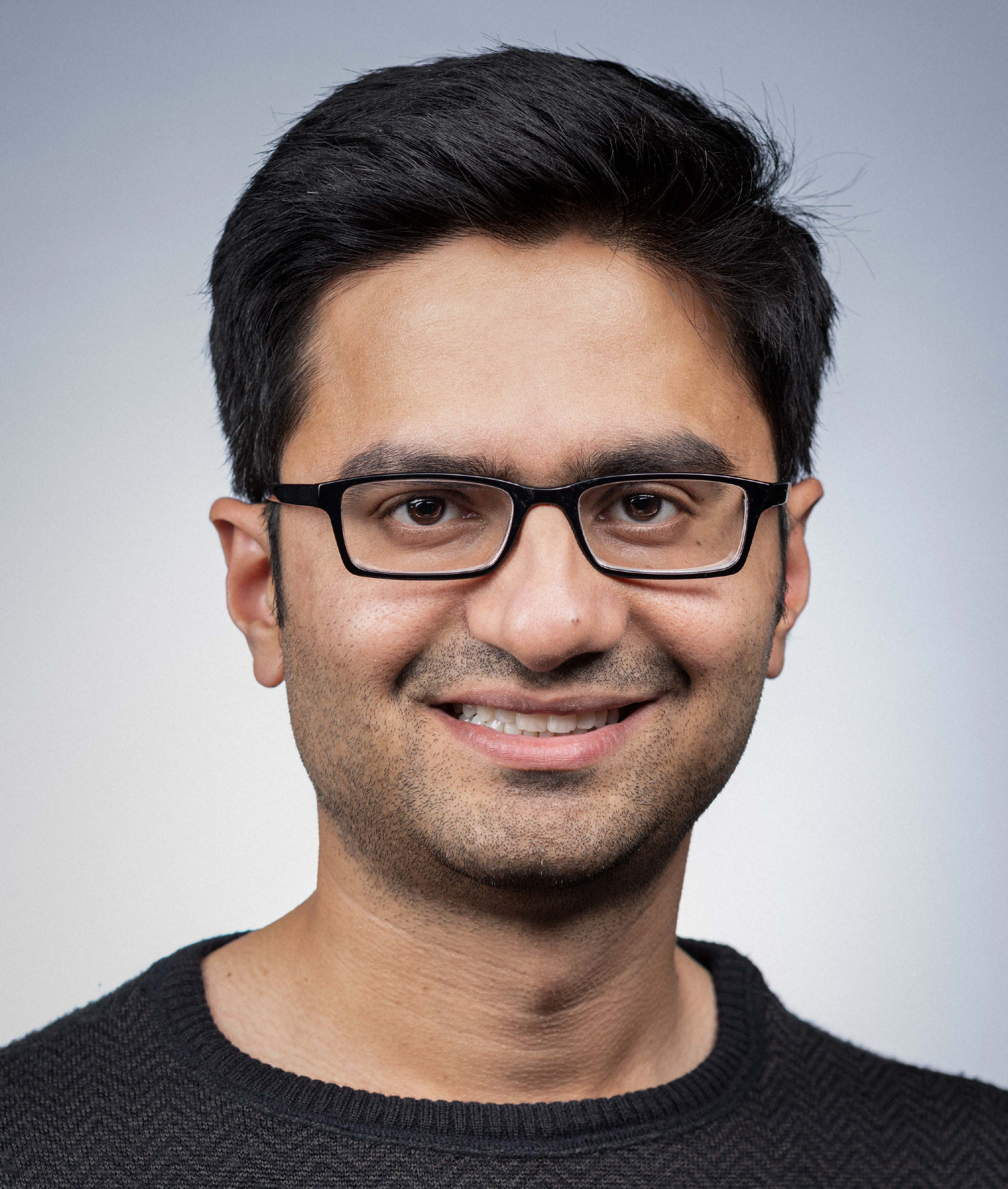}}]{Amar M. Kamat}
received his MS degree in Mechanical Engineering (2011) and his PhD degree in Engineering Science and Mechanics (2016) from the Pennsylvania State University (USA). He previously performed research internships at Intel Corporation (USA) and BASF SE (Germany), and is currently a postdoctoral researcher at the University of Groningen in the Netherlands. His research interests include additive manufacturing (3D printing), bioinspiration and biomimetics, and sensor development for biomedical applications. Dr. Kamat currently serves as a Review Editor (Micro- and Nanoelectromechanical Systems section) for \emph{Frontiers in Mechanical Engineering} and as an Editorial Board member (Manufacturing Processes and Systems section) for \emph{Materials}.
\end{IEEEbiography}
\vfill
\vspace{-100 pt}

\begin{IEEEbiography}[{\includegraphics[width=1in,height=1.25in,clip,keepaspectratio]{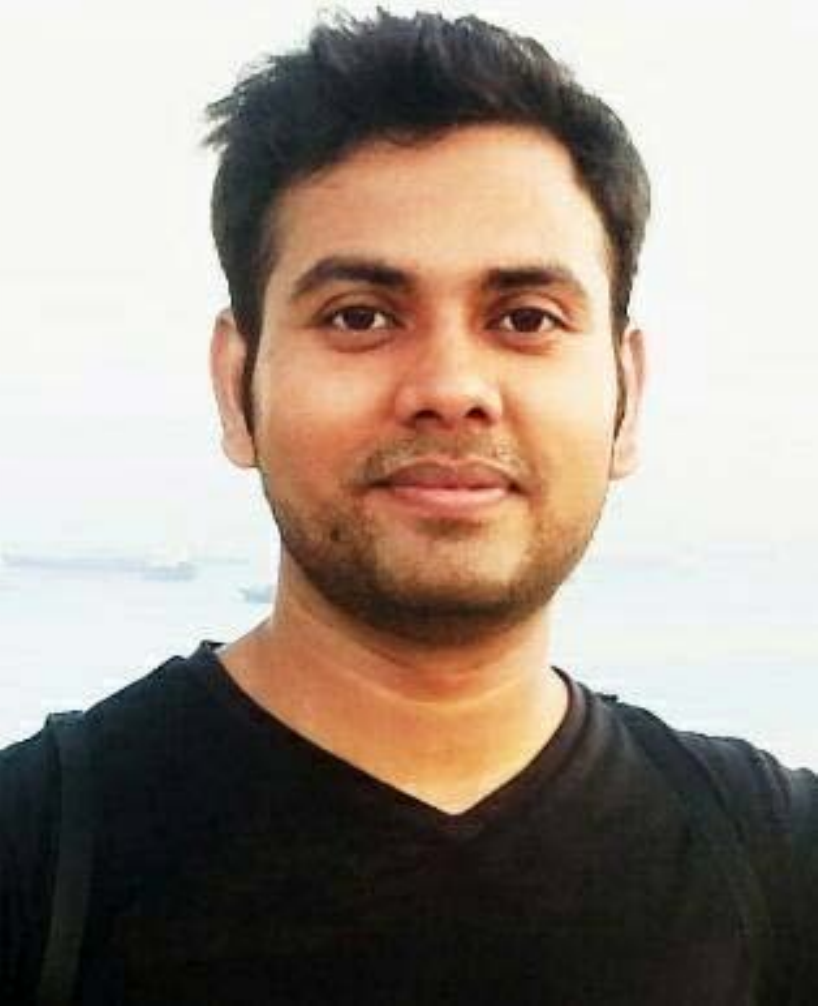}}]{Ajay Giri Prakash Kottapalli}
is an assistant professor in the Department of Advanced Production Engineering at University of Groningen. In 2013, he received his Ph.D. degree from Nanyang Technological University. During 2014-2015, he was a Postdoctoral Associate at Singapore-MIT Alliance for Research and Technology (SMART) and in 2016 he became a Principal Research Scientist at SMART. He is also currently a Research Affiliate with the MIT Sea Grant at MIT. In 2018, he was awarded the top-10 innovators under 35 in Asia-Pacific by MIT Technology Review. His research interests mainly include biomimetic/bio-inspired MEMS/NEMS, nanoelectronics, TENG and PENGs, and biomedical sensors.
\end{IEEEbiography}
\vfill





\end{document}